\documentclass[twoside,11pt]{article}

\usepackage{blindtext}

\usepackage[abbrvbib, preprint]{jmlr2e}

\let\proof\relax

\usepackage{amsthm}

\newtheorem{theorem}{Theorem}
\newtheorem{definition}[theorem]{Definition}
\newtheorem{remark}[theorem]{Remark}
\newtheorem{lemma}[theorem]{Lemma}

\usepackage{mleftright}

\usepackage[utf8]{inputenc}
\usepackage[T1]{fontenc}
\usepackage{amstext}
\usepackage{amsmath}
\usepackage{amssymb}
\usepackage{bm}
\usepackage{amsfonts}
\usepackage[utf8]{inputenc}
\usepackage{graphicx}
\usepackage{bbm, comment}
\usepackage{color}
\usepackage{lipsum}
\usepackage{enumitem}
\usepackage{url}
\usepackage{makecell}
\usepackage{booktabs}
\usepackage{multirow}
\usepackage{thmtools,thm-restate}

\usepackage{etoolbox}

\DeclareMathOperator*{\argmax}{arg\,max}

\usepackage{pgfplots}
\usetikzlibrary{pgfplots.groupplots}
\usetikzlibrary{shapes.geometric}
\usepgfplotslibrary{fillbetween}
\selectcolormodel{cmyk}

\usepackage{svg}

\usepackage[ruled,linesnumbered]{algorithm2e} 
\usepackage{setspace}
\DontPrintSemicolon

\newenvironment{proofsketch}{%
  \noindent\textit{Proof sketch.}\ }%
  {\hfill$\square$\par}

\newcommand\yk[1]{\textcolor{blue}{[Yuqing: #1]}}
\newcommand\pr[1]{\textcolor{green}{[Paul: #1]}}
\newcommand\gs[1]{\textcolor{blue}{[Grant: #1]}}

\newcommand\La[0]{\mathcal{L}}
\newcommand\St[0]{\mathcal{S}}
\newcommand\GT[0]{\mathcal{G}}
\newcommand\X[0]{\mathcal{X}}

\newcommand*{\eg}{{\em e.g.}}

\newcommand*{\ie}{{\em i.e.}}

\newcommand{\E}{\mathbb{E}}
\renewcommand{\vec}[1]{\mathbf{#1}}

\newcommand{\ith}{i^\textrm{th}}
\newcommand{\jth}{j^\textrm{th}}

\newcommand{\itemFeatures}{x_i}
\newcommand{\itemFeaturesRV}{X_i}
\newcommand{\itemFeaturesVec}{\mathbf{x}}
\newcommand{\itemFeaturesVecRV}{\mathbf{X}}

\newcommand{\itemResponseState}{s_i}
\newcommand{\itemResponseStateRV}{S_i}
\newcommand{\itemResponseStateVec}{\mathbf{s}}
\newcommand{\itemResponseStateVecRV}{\mathbf{S}}

\newcommand{\itemGT}{g_i}
\newcommand{\itemGTRV}{G_i}
\newcommand{\itemGTVec}{\mathbf{g}}
\newcommand{\itemGTVecRV}{\mathbf{G}}

\newcommand{\itemClassifierOutput}{c\mleft(\itemFeatures\mright)}

\newcommand{\itemClassifierOutputVec}{c\mleft(\itemFeaturesVec\mright)}
\newcommand{\itemClassifierOutputVecRV}{c\mleft(\itemFeaturesVecRV\mright)}
\newcommand{\itemGivenClassifierOutputVecRV}[1]{#1(\itemFeaturesVecRV)}

\newcommand{\evaluationsVec}{\mathbf{e}}
\newcommand{\evaluationsVecRV}{\mathbf{E}}

\newcommand{\evaluationRating}{e_{ij}}
\newcommand{\evaluationRatingRV}{E_{ij}}
\newcommand{\evaluationRatingVec}{\mathbf{e}_{[:n]j}}
\newcommand{\evaluationRatingVecRV}{\mathbf{E}_{[:n]j}}

\newcommand{\evaluationRatingPanel}{\mathbf{e}_{i[:k_e]}}
\newcommand{\evaluationRatingPanelRV}{\mathbf{E}_{i[:k_e]}}
\newcommand{\evaluationRatingPanelVec}{\mathbf{e}_{[:n][:k_e]}}
\newcommand{\evaluationRatingPanelVecRV}{\mathbf{E}_{[:n][:k_e]}}

\newcommand{\evaluationRatingSingle}{e_{i1}}
\newcommand{\evaluationRatingSingleRV}{E_{i1}}

\newcommand{\evaluationRatingSingleVecRV}{\mathbf{E}_{[:n]1}}

\newcommand{\evaluationPanelLabel}{\textrm{maj}(\evaluationRatingPanel)}
\newcommand{\evaluationPanelLabelRV}{\textrm{maj}(\evaluationRatingPanelRV)}
\newcommand{\evaluationPanelLabelVec}{\textrm{maj}(\evaluationRatingPanelVec)}
\newcommand{\evaluationPanelLabelVecRV}{\textrm{maj}(\evaluationRatingPanelVecRV)}

\newcommand{\genericLabel}{y_{ij}}

\newcommand{\scoringLabel}{e^{*}_{i}}
\newcommand{\scoringLabelRV}{E^{*}_{i}}
\newcommand{\scoringLabelVec}{\mathbf{e}^{*}}

\newcommand{\benchmarksVec}{\mathbf{b}}
\newcommand{\benchmarksVecRV}{\mathbf{B}}

\newcommand{\benchmarkRating}{b_{ij}}
\newcommand{\benchmarkRatingRV}{B_{ij}}
\newcommand{\benchmarkRatingVec}{\mathbf{b}_{[:n]j}}
\newcommand{\benchmarkRatingVecRV}{\mathbf{B}_{[:n]j}}

\newcommand{\benchmarkRatingPanel}{\mathbf{b}_{i[:k_b]}}
\newcommand{\benchmarkRatingPanelRV}{\mathbf{B}_{i[:k_b]}}

\newcommand{\benchmarkRatingPanelVecRV}{\mathbf{B}_{[:n][:k_b]}}
\newcommand{\benchmarkRatingPanelInfRV}{\mathbf{B}_{[:][:k_b]}}

\newcommand{\comb}{\textsc{comb}}

\newcommand{\benchmarkPanelLabelVecRV}{\comb(\benchmarkRatingPanelVecRV)}
\newcommand{\benchmarkPanelLabelInfRV}{\comb(\benchmarkRatingPanelInfRV)}

\newcommand{\combCaliSub}{\textsc{comb}}

\newcommand{\benchmarkPanelLabelCalibratedS}{\combCaliSub(\benchmarkRatingPanel)}

\newcommand{\powerscore}{p}
\newcommand{\powercurve}{pc}
\newcommand{\classifierescore}{\textsc{CES}}

\newcommand{\fratereqG}{\textsc{REQ}_{\comb}^{GT}}
\newcommand{\fratereqS}{\textsc{REQ}_{\comb}}

\newcommand{\fratereqcaliS}{\textsc{REQ}_{\comb'}}

\newcommand{\utilityFunction}{\textsc{Utility}}
\newcommand{\scoreFunction}{\textsc{Score}}

\newcommand{\empowerscore}{\hat{p}}
\newcommand{\empowercurve}{\hat{pc}}

\newcommand{\emfratereq}{\tilde{\textsc{REQ}}_{\comb}}

\newcommand{\allVec}{\mathbf{w}}
\newcommand{\allVecRV}{\mathbf{W}}

\newcommand{\allpartitions}{\textsc{AllPartitions}(\allVec, k_b, k_e)}

\newcommand{\sampledpartitions}{\textsc{Partitions}(\allVec, k_b, k_e)}

\newcommand{\emscore}{\hat{\scoreFunction}(c, k_e)}

\newcommand{\abc}{\textsc{ABC}}

\newcommand{\numraters}{k_w}

\usepackage{lastpage}

\ShortHeadings{Rater Equivalence}{Resnick et al}
\firstpageno{1}

\begin{document}

\title{Principled Evaluation with Human Labels: One Rater at a Time and Rater Equivalence}

\author{\name Paul Resnick \email presnick@umich.edu \\
       \addr School of Information, University of Michigan
       \AND
       \name Yuqing Kong \email yuqing.kong@pku.edu.cn \\
       \addr Center on Frontiers of Computing Studies, Peking University
       \AND
       \name Grant Schoenebeck \email schoeneb@umich.edu \\
       \addr School of Information, University of Michigan
       \AND
       \name Tim Weninger \email tweninger@nd.edu \\
       \addr Department of Computer Science and Engineering,        University of Notre Dame
       }

\maketitle

\begin{abstract}
In many classification tasks, there is no definitive ground truth, only human judgments that may disagree. 
We address two challenges that arise in such settings: (1) how to use human raters to score classifiers, and (2) how to use them for comparison benchmarks.
For the first, the common practice is to score classifiers against the majority vote of an evaluation panel of several human raters.
We argue that this is not justified when either of two properties fails: objectivity or equanimity.
Instead, under a utility model appropriate for such settings, scoring against \emph{one rater at a time} and averaging the scores across raters is a more principled approach.
For the second, we introduce the concept of \emph{rater equivalence}: the smallest number of human raters whose combined judgment matches the classifier's performance. We provide a provably optimal algorithm for combining benchmark panel labels, and demonstrate the framework through case studies.
\end{abstract}


\section{Introduction}

A news site or social media platform is considering using a large language model (LLM) to identify comments that violate their policies. Before deploying the LLM in their moderation and ranking systems, the company needs to evaluate its performance. To do this, they recruit a panel of human raters who are trained on the company's definitions and policies regarding inappropriate comments. The raters independently assess a dataset of comments. The LLM's performance is then evaluated based on how often its output matches the majority vote of the human raters. The evaluation reveals that the LLM achieves an accuracy of only 80\% on this dataset. The natural conclusion is that it is not ready for deployment.

However, the data scientists on the team question this conclusion. They note that the human raters did not always agree with each other, so it may not make sense to treat the majority vote as a proxy for ground truth correct labels. 
Moreover, not all misclassifications are equally bad: failing to flag a comment that 90\% of raters would agree violates the policy seems worse than missing one where only 60\% would agree. The 80\% accuracy number obscures both of these issues.

This is just one of a large class that we will call \emph{human judgment settings}, where classifiers have to be evaluated against human labels, even though not all humans provide the same label for each item. Consider some other common classification tasks:
\begin{itemize}
\item Is there a person in this image?
\item Does this set of radiology images show a potential malignancy that warrants further investigation?
\item What grade should be assigned to this student's homework submission?
\end{itemize}

The practice of evaluating classifiers using human labels is well established, but it is not yet on a firm theoretical footing, which we try to remedy in this paper. We consider two aspects: how to use human labels for evaluation, and how to use them for comparison benchmarks. There is also a growing literature on using noisy human labels for training~\citep{frenay2013classification}, but that is not our focus.


The key question that we address about using human labels for evaluation is whether they should be combined into a single pseudo-ground-truth label for each item, or whether the classifier should be scored against each rater's labels separately. There is a widely shared intuition that combining labels will reduce any idiosyncratic noise from individual raters. Thus, it has become a common practice to use the combined labels of a panel as gold-standard labels (e.g., ten raters for evaluating the Perspective API \citep{Wulczyn2017machina}, five for the MNLI/SNLI tasks used in the GLUE benchmark \citep{williams2018broad}).

Contrary to that intuition, we find that it is usually preferable \emph{not} to combine them by taking a majority vote. Instead, the classifier should be scored against the labels from each individual rater separately. Then, the scores should be averaged to yield an estimate of the expected utility of the classifier. In short, average the scores, don't score against the average label.

\begin{figure}[t]
\centering
\includegraphics[width=.8\textwidth]{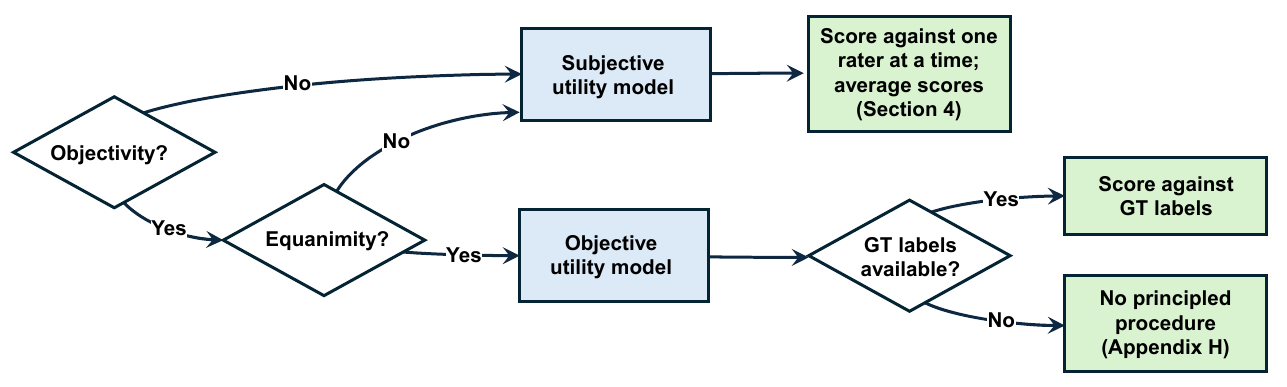}
\caption{Decision tree for choosing an evaluation approach. Two properties of the evaluation setting---\emph{objectivity} (whether an objective ground truth exists) and \emph{equanimity} (whether all misclassifications are equally problematic regardless of how obvious the item is)---determine which utility model applies. In the subjective utility model, the best practice is to score against each rater's labels separately and average the scores. In the objective utility model with no ground truth labels, there is no theoretically justified best practice for evaluation against human labels.}
\label{fig:decision-tree}
\end{figure}

We arrive at this advice by considering two properties of evaluation settings. First is \emph{objectivity}: is there an objective ground truth for each item? If we think there is an objective ground truth, we think of any disagreement among raters as noise. Alternatively, with a subjective ground truth, we think of disagreement as reflecting valid differences in perspectives and judgments.

For some tasks, the answer is clear. In radiology, a tumor is either present or not. In rating jokes, humor is in the eye of the beholder. For other tasks, a reasonable analyst could conceptualize it either way. For example, many content moderation platforms treat the task as having objective answers defined by their detailed policies. When raters disagree about whether a particular post is "dehumanizing", it's an indicator that the post is difficult to judge but has no effect on what the correct determination is with respect to the policy. An alternative approach conceives of policies as providing high-level guidance but does not expect them to yield a binary outcome for every item. An item is somewhat dehumanizing if 60\% of raters judge it to be so, and very dehumanizing if 90\% of raters do.

Second is \emph{equanimity}: Is the harm of a misclassification the same regardless of how obvious or borderline the item is? Or is it worse for a classifier to get wrong an item that most people would agree on? From the patient's perspective, a missed tumor causes the same harm whether it was easy or hard for radiologists to detect---equanimity holds. But from the hospital's perspective, missing a tumor that 90\% of radiologists would have caught may carry greater legal or reputational liability than missing one that only 60\% would detect---equanimity does not hold. Similarly, even if a content moderation platform is convinced that its policies define an objective ground truth for every post, it may still matter more to correctly classify the clear-cut cases.

When there is no objective ground truth, we apply what we call the \emph{subjective utility model}, in which utility is defined as the expected score against a randomly selected rater for each item. When there is an objective ground truth but equanimity does not hold, the subjective utility model is also a pragmatic choice---other utility functions could model 
that it is more problematic to misclassify items with high rater agreement,
but our subjective utility model has the benefit that it is directly measurable from individual rater labels. This is shown in Figure~\ref{fig:decision-tree}, where the subjective utility model applies in both cases.

Under the subjective utility model, we prove in Section~\ref{sec:subjective-gt} that scoring against a single rater yields an unbiased estimate of the expected utility. To reduce variance, score the classifier against each available rater's labels separately and average the scores.

When there is an objective ground truth and equanimity holds---the harm of a misclassification depends only on matching it---we call this the \emph{objective utility model}. If ground truth labels are available, the classifier should be evaluated against them directly---that is what defines utility. Ground truth labels may be expensive or delayed, but they can often be obtained for at least a validation subset. In radiology, the ground truth is eventually confirmed by biopsy, surgery, or long-term follow-up. In forecasting, predictions can be evaluated against realized outcomes. In fraud detection, investigations confirm or disconfirm the initial classification. In all such cases, the best practice is to evaluate against those ground truth labels rather than using human labels as a proxy.

What about settings where the objective utility model applies but no ground truth evaluation labels are available? Unfortunately, in those settings there is no theoretically justified best practice. We prove negative results in Appendix~\ref{appdx:objective-utility}. There is no finite size of evaluation panel that will always yield the same classifier score as scoring against the ground truth. Moreover, it is possible for evaluation against a single rater to match the ground truth ordering of two classifiers, while evaluation against a three-person panel reverses the ordering.

Fortunately, such evaluation settings are uncommon in practice. The problematic case requires that all three conditions hold simultaneously: no ground truth labels, objectivity, and equanimity. 
Most settings fail at least one: when ground truth is unobservable it is usually because the concept is inherently subjective, and even when it is objective, obvious errors rarely cause the same harm as borderline ones.

So far, we have discussed how to score a classifier against human labels. But a score in isolation may not be enough for a deployment decision---a manager also needs to know how the classifier compares to the alternative, which is often a human process. This leads to the second use of human labels: as comparison benchmarks. Figure~\ref{fig:framework} illustrates both uses together.
The middle row summarizes the evaluation procedure just described: individual raters provide labels, which are used to score the classifier.
For generality, it depicts the common practice of combining an \textbf{evaluation panel}'s labels using majority vote or the mean for numeric labels, with single rater panels as a special case. 

\begin{figure}[t]
\centering
\includegraphics[width=.75\textwidth]{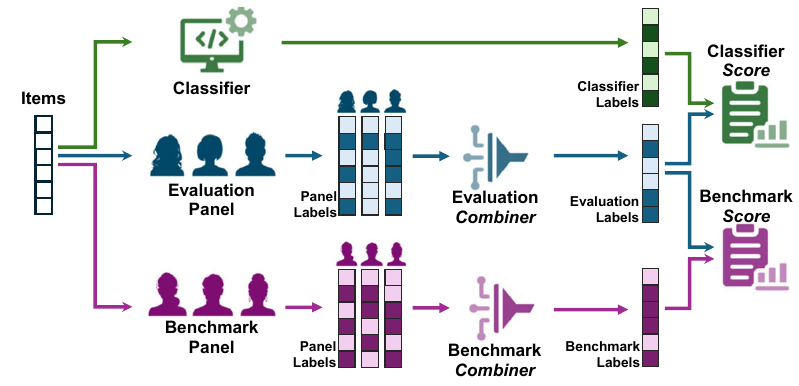}
\caption{Using human labels both for evaluation and for comparison benchmarks. Evaluation labels are used to score the classifier (row 2); benchmark labels are combined into a human panel classifier to compare against (row 3).}
\label{fig:framework}
\end{figure}

The bottom row of Figure~\ref{fig:framework} illustrates how benchmark panels work. Raters are assembled into a \textbf{benchmark panel}. Each rater generates an independent label for each item. A \textit{benchmark combiner} then processes the labels from the benchmark panel to produce a single benchmark label for each item. The classifier is scored against these benchmark labels using the same scoring function as for the evaluation panel. The resulting score can be compared to the score of the classifier against the same evaluation panel to determine which is better.

A benchmark panel can be drawn from the same population as the evaluation raters or from a different group. For example, misinformation researchers compared the performance of benchmark panels consisting of either Mechanical Turkers or expert journalists and fact-checkers, evaluating both against expert panels \citep{doi:10.1126/sciadv.abf4393, doi:10.1177/26339137231173407}. In a course grading setting, the benchmark panel might consist of peer graders and the evaluation panel consist of a single professor.

In Section~\ref{sec:power-curve}, we provide an approach that simulates benchmark panels of various sizes, by drawing samples from the available labels for each item. This method allows us to compare the classifier to a range of human benchmark classifiers of increasing quality as the benchmark panel size grows. By doing so, we can plot a \emph{power curve}, with the benchmark panel size on the x-axis and the expected score of a classifier derived from a benchmark panel of that size on the y-axis, as shown in Figure~\ref{fig:ce_power_curve}. We define the \emph{rater equivalence} of a classifier as the smallest human benchmark panel that has a higher expected score than the classifier.\footnote{In practice, there may be additional fairness criteria that would also come into play in deciding whether a classifier is better than an alternative human process. For example, even if a classifier gets a better expected score overall than three-person human benchmark panels, if it scores worse for some subgroups of items, such as assignments from male vs. female students, that may be problematic. If additional fairness criteria can be quantified, then a power curve could also be plotted for other metrics besides overall expected score.}
Thus, a higher rater equivalence indicates a better-performing classifier.

\begin{figure}[t]
\centering
\begin{tikzpicture}
\sffamily
\begin{axis}[
legend style={font=\small,
	nodes={scale=0.7, transform shape},
	at={(0.0,1)},
	anchor=north west,
	draw=none,
	fill=none},
legend cell align={left},
width = 4.0in, height = 2.0in,
ylabel near ticks,
ylabel = {\small score},
xlabel near ticks,
every tick label/.append style={font=\scriptsize},
xmin=-1,xmax=10,ymin=0,ymax=0.4,
xtick={0, 1.959349891708352, 4, 6, 8},
xlabel={\small Number of benchmark raters},
xlabel style = {yshift=0.05in},
yticklabel style={
		/pgf/number format/fixed,
		/pgf/number format/precision=5
},
scaled y ticks=false
]

\addplot[solid, mark=o, mark options={scale=.8}, black]
table{
0	0.0
1	0.08024457130058849
2	0.13522196667357766
3	0.16583024313435013
4	0.18584839594613023
5	0.1986896732953577
6	0.2082526492172606
7	0.21407454595233355
8	0.2189467764130909
9	0.2227354184241982
};

\addplot[mark=none, red, dashed, thick, samples=2, domain=-1:31] {0.1329871295980729};

\addlegendentry{$k$ raters}

\addlegendentry{classifier}

\addplot[mark=*, mark options={scale=.8}, red, thick]
    table[]{
        1.959349891708352	0.1329871295980729
    };

\addplot[mark=none, red, dashed, thick, samples=2, domain=-1:31] coordinates {(1.959349891708352,0.1329871295980729) (1.959349891708352,0)};

\end{axis}
\end{tikzpicture}
\caption{An example power curve depicting a classifier's score of 0.13. The classifier's rater equivalence is 1.96. A single benchmark rater yields a lower expected score, whereas a benchmark panel comprising two raters shows a slightly higher expected score.}
\label{fig:ce_power_curve}
\end{figure}

Note that using a poor benchmark panel combiner function can artificially inflate the rater equivalence, leading to an overestimation of the classifier's performance. In Section~\ref{sec:abc}, we introduce an algorithm called the \textit{Anonymous Bayesian Combiner} for combining labels from simulated benchmark panels. It produces the Bayesian posterior probability distribution for the next label, based on the labels observed so far. We prove that it is the optimal combiner if cross-entropy is used as the scoring function, and that it is computationally tractable.\footnote{We provide a software library that implements all of the procedures for generating power curves and computing rater equivalence values. The software can be downloaded from https://pypi.org/project/surveyequivalence/}

Finally, in Section~\ref{sec:case-studies}, we examine three case studies from prior literature. For each case study, we plot power curves and calculate rater equivalence values. These analyses offer additional context for interpreting the results of these prior studies.

In summary, our contributions are:
\begin{itemize}
    \item A decision framework based on two properties of the evaluation setting---\emph{objectivity} and \emph{equanimity}---that determines whether classifiers should be evaluated according to the subjective or objective utility model (Section~\ref{sec:utility-models});

    \item Proof that, under the subjective utility model, scoring against a single evaluator at a time yields an unbiased estimate of expected utility, while scoring against the combined votes of larger panels can introduce bias (Section~\ref{sec:subjective-gt});

    \item Counter-examples showing that, under the objective utility model, there is no guarantee that evaluating against panels of any finite size will yield the same managerial decision as evaluating against the unknown ground truth, and that bigger panel sizes are not always better (Appendix~\ref{appdx:objective-utility});

    \item The concept of a \emph{power curve} and the definition of \emph{rater equivalence}---a formal measure of classifier performance relative to human benchmark panels of varying sizes (Section~\ref{sec:power-curve});

    \item A computationally tractable algorithm for combining labels from simulated benchmark panels that is provably optimal if cross-entropy is used as the scoring function (Section~\ref{sec:abc}).
\end{itemize}

\section{Related Work}

This section begins by surveying scoring functions that have been used for scoring classifier outputs against human labels. We then turn to past work that explicitly considers the sources of rater disagreement and their impacts on evaluating and training classifiers.

\subsection{Scoring Functions}
Classifier performance is typically measured using accuracy, precision, recall, and both frequentist and Bayesian comparative analyses \citep{rainio2024evaluation, demvsar2006statistical, ferri2009experimental}. In cost-sensitive contexts, scoring functions account for asymmetric error costs using cost matrices or F$_\beta$-Scores \citep{elkan2001foundations, chinchor-1992-muc}.

Multi-class tasks often rely on macro- and micro-averaging to aggregate performance across classes \citep{sokolova2006beyond}. Although ROC AUC is widely used, it may mask class imbalance, prompting interest in alternatives like the Matthews Correlation Coefficient (MCC) \citep{chicco2021matthews}.

Domain-specific metrics offer further refinement. In text classification, ranked relevance is better captured by measures such as Mean Reciprocal Rank (MRR) and Normalized Discounted Cumulative Gain (NDCG) \citep{voorhees1999trec, jarvelin2002cumulated}.

\subsection{Models of Rater Disagreement}

Rater disagreement is a central challenge in using human-generated labels, especially for social or value-laden tasks. Prior work suggests that disagreement can arise from three distinct but interrelated sources: (1) the type of task, (2) variation in how raters interpret labels, and (3) noise in the labeling process. Each has distinct implications for how classifier performance should be evaluated and interpreted.

\subsubsection{Task Type: Objective vs. Subjective Ground Truth}

Plato introduces the juxtaposition between knowledge (episteme) and opinion (doxa) in \textit{The Republic}, where opinion is only an imperfect (noisy) reflection of the truth whereas knowledge is a perfect reflection of truth~\citep{plato2007republic}.  In contrast, ~\cite{whitehead1978process} gives human opinions their own ontological existence as ``prehensions''.  
Similarly, in content labeling, it is natural to distinguish between objective tasks where each item has a single correct label (\eg, digit classification, spelling correction) and subjective tasks where multiple judgments can all be valid (\eg, sentiment analysis~\citep{10.5555/2002736.2002760, haralabopoulos2020objective}, respondent's personally preferred moderation action~\citep{10.1145/3610082}).
We distinguish between tasks where the actions induced by the labels lead to the same utility for everyone, which we call the \textit{objective utility model} and tasks where an action on an item has a different effect on different people, which we call the \textit{subjective utility model} (see Section~\ref{sec:utility-models}).

\subsubsection{Interpreting Variation Across Raters}

Disagreement among annotators often arises from systematic differences in how labeling criteria are interpreted, shaped by cultural, ideological, or demographic backgrounds. For example, liberal and conservative raters may judge the tone of political content differently, each applying consistent but divergent standards. This interpretive variation challenges the assumption that subjective tasks have a single, universal ground truth. Disaggregated evaluation approaches recommend explicitly modeling and reporting subgroup differences rather than averaging across them \citep{barocas2021designing}. Similarly, ground truth built from one annotator population (\eg, Mechanical Turk workers) may fail to generalize, underscoring the importance of viewing human labels not as neutral truths but as community-dependent judgments \citep{sen2015turkers}.

\subsubsection{Label Noise}
Label noise, \ie, a mismatch between the rater's reported label and their true internal judgment, arises from many sources: lack of attention, expertise, fatigue, poor task design, or ambiguous instructions. In many models, this is treated as i.i.d. noise, supporting results like the Condorcet Jury Theorem \citep{10481/57153}, but this assumption often breaks down in practice \citep{paun2018comparing, pmlr-v216-burrell23a}.

Noise can be item-dependent, varying with difficulty \citep{ghosh2011moderates}, ambiguity \citep{dumitrache2018capturing}, or even with the label itself \citep{https://doi.org/10.2307/2346806}. More refined models incorporate both item and rater parameters \citep{lakkaraju2015bayesian, zarkoob2023better}. For example, item-response theory models the interaction between item difficulty and rater expertise \citep{hambleton1985item}, while matrix factorization captures latent rater-item effects \citep{koren2009matrix}. 
\citet{gordon2021disagreement} factor the rating matrix into latent rater and item factors, imputing per-rater item ratings that they reinterpret as subjective truths of which observed labels are noisy proxies.

In our framework, this distinction is reflected in the assumptions underlying the utility models. The objective model assumes labels are samples from a latent distribution that reflects noisy judgments. The subjective model treats each label as a valid judgment, free from noise, but drawn from a distribution that reflects differences in subjective judgments across the population.

\subsection{Evaluation with Rater Disagreement}

When labels are noisy or subjective, evaluating classifiers becomes difficult. 
Prior work has approached this by reconstructing absolute scores or comparing to imperfect human benchmarks.

\subsubsection{Reconstructing Absolute Scores}

\paragraph{Objective Ground Truth.}
A classifier's score against any imperfect proxy may not match the classifier's score against ground truth. 
Agreement thresholds such as Kappa or Krippendorff’s alpha are often used to justify proxies \citep{doi:10.1177/001316446002000104, krippendorff04} but even datasets with agreement levels above commonly used thresholds may yield misleading evaluation scores (see Appendix~\ref{appdx:obj-absolute}).
Taking the majority vote of several raters can yield a better proxy for ground truth and probabilistic models like Dawid-Skene can improve on this further by inferring rater expertise \citep{https://doi.org/10.2307/2346806} and then taking that into account when combining rater labels. Still, they provide only a proxy.

It would be nice to reconstruct the absolute score against ground truth labels by applying an adjustment to the observed score against proxy labels. \cite{lam2003evaluating} make a parametric assumption about the error generating process for proxy labels. They also assume independence between proxy label errors and classifier errors. They then produce an adjusted classifier score that reconstructs what the score would have been against ground truth labels. 

To see why some assumption is needed about the joint distribution of ground truth, proxy labels, and classifier labels, consider a simple scenario. Suppose half the items are ground truth positive, 90\% of positive items have positive proxy labels and 90\% of negative items have negative labels. A classifier that correctly outputs the ground truth on every item would have a score against the proxy labels of 90\%. But a classifier with only 80\% true accuracy could also have the same score against the proxy labels, if it matched all the incorrect proxy labels and had mismatches on 10\% of the items where the proxy labels were correct.
The issue here is not that we cannot learn the error of the raters, or that the rater model is too simple or too complicated. Instead, the issue is that, without an assumption about how the classifiers errors are correlated with proxy label errors, we do not know whether the correct score against the ground truth is 80\%, 100\%, or something in between.

Unfortunately, generally it will not be safe to assume classifier errors are uncorrelated with human proxy label errors. Automated classifiers may key on different features, or process them in different ways than people do. Any strong assumption about the three-way joint distribution of ground truth, human-generated proxy labels, and classifier labels, is unwarranted and likely to lead to misleading results.

\paragraph{Subjective Labels.}

If each rater's label reflects a subjective judgment, evaluation is ideally based on the distribution of views in the population rather than just the central tendency \citep{pavlick-kwiatkowski-2019-inherent,10.1145/3136755.3136792}. Unlike the objective setting, labels are not imperfect proxies for a hidden ground truth but valid draws from the population distribution of views.

\subsubsection{Relative Scores Against Imperfect Labels}

Some evaluations focus on relative comparison to human benchmarks. In polling and economics, approaches like forecast accuracy and ``Equivalent Number of Observations'' quantify how model performance compares to human groups \citep{rothschild2011forecasting, erevrothENO2007}. 
We generalize this idea using power curves and rater equivalence.

\subsubsection{Rater Groups}

Rater groups may vary systematically and it is critical that the evaluation labels come from the right pool. \citet{sen2015turkers} showed that algorithms that perform well against evaluation labels gathered from Mechanical Turk can perform worse when evaluated against labels from other rater pools. \citet{barocas2021designing} argue for disaggregated evaluation, with scores calculated separately for different subgroups that may be affected differently. In evaluating classifiers, for example, that would involve acquiring labels from multiple subgroups where those labels differ systematically. 

Our subjective utility model can be thought of as an extreme version of disaggregated evaluation, where every individual's judgments generates a different evaluation score. Whether we think of evaluators as unique individuals or as representative of subgroups, it is crucial that the rater pool be representative of the population of interest. Any estimate of average utility for the whole group or for subgroups will be biased if the rater pool systematically differs from the population of interest.

\subsection{Learning with Rater Disagreement}

While related to evaluation, learning under disagreement involves different goals and metrics. Learning seeks to find parameter values that produce the best classifier, while evaluation seeks to quantify how good the resulting classifier is. 

If the labeled data used in the learning process doesn't always match the objective ground truth, the standard supervised learning techniques learn a sub-optimal classifier. 
%
%
\citet{frenay2013classification} provide a comprehensive survey of this area.
Existing methods for handling noisy labels typically either leverage a small amount of clean data to correct loss functions \citep{DBLP:conf/cvpr/PatriniRMNQ17,10.5555/3327546.3327707,sukhbaatar2015training,10.5555/2999611.2999745,10.1109/TPAMI.2015.2456899} or design loss functions that are invariant to noise~\citep{6342929,GHOSH201593,10.5555/3298483.3298518,pmlr-v97-charoenphakdee19a}. Most work relies on strong assumptions---such as the assumption that label noise is independent and identically distributed (i.i.d.)---to provide theoretical guarantees. While these loss functions help optimize learning, they are not well-suited for evaluating the utility of learned models.  

Additionally, techniques like dropout \citep{JMLR:v15:srivastava14a}, data augmentation (e.g., mixup \citep{zhang2018mixup}), pruning, and weighting strategies  \citep{10.1016/j.knosys.2012.01.015,Muhlenbach2004} have been proposed to address issues like overfitting and noisy labels. Again, these techniques are training strategies and are not directly useful in model evaluation.

The methods above assume that a single ground-truth label exists and that rater disagreement is a form of noise to overcome. A separate line of research instead treats rater disagreement as signal, preserving rater-level information in the training objective. \citet{davani-etal-2022-dealing} train multi-task models with per-annotator prediction heads and a shared task representation, finding that this yields aggregate classification performance equal to or better than training on majority-vote labels, and also produces uncertainty estimates that correlate with annotator disagreement. Gordon et al.'s ``jury learning'' \citep{gordon2022jury} extends per-annotator modeling to inference: at prediction time, a jury is sampled from the learned per-annotator models, and its collective decision is the classifier's output. This enables the analyst to correct for systematic differences between the rater pool and the population of interest, or to produce predictions that reflect the view of a subgroup of interest.


\section{Modeling Rater Disagreement and Classifier Utility}

Conceptually, we can think of a rater $j$'s label for an item $i$, $\genericLabel$, as being determined by properties of the rater, the item, and the rating context. However, we do not directly model all of these properties. Instead, we consider a reduced form model, where each item has a \emph{rater response state} $\itemResponseState$, which is a probability distribution. Each rater label is an i.i.d.\ draw from $\itemResponseState$. Although individual raters may differ systematically, this model ignores those differences and treats each label as if it had been collected anonymously. 

\begin{figure}[t]
\centering
\includegraphics[width=0.7\textwidth]{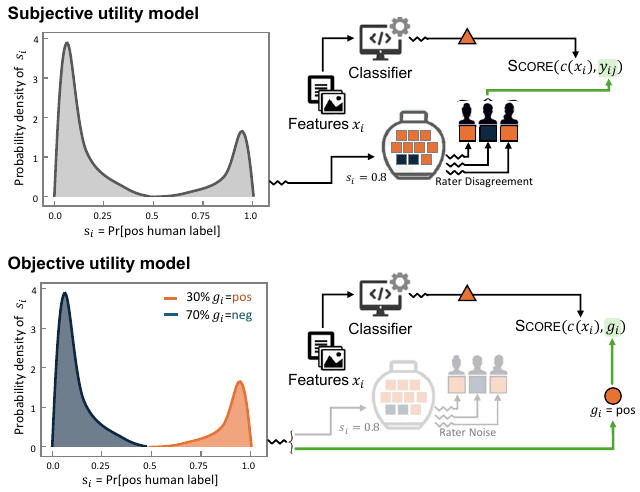}
\caption{Two utility models, with different data generating processes and interpretations of the item-specific response state $s_i$. In both models, raters' labels are i.i.d. draws from an item-specific response state. In the subjective utility model (top), there is no presumed underlying ground truth; $s_i$ directly represents the distribution of valid human judgments for item $i$, and classifier utility is determined by agreement with individual rater labels. In the objective utility model (bottom), each item has a hidden true label $g_i$, and $s_i$ captures noisy human responses conditional on that truth. Accordingly, classifier utility is determined by agreement with $g_i$ only, regardless of the distribution of observed labels.} 
\label{fig:noisy_signals}
\end{figure}

Figure~\ref{fig:noisy_signals} illustrates our two models of classifier utility. In either case, a classifier processes the item's content, represented as a set of features $\itemFeatures$, to produce an output $\itemClassifierOutput$. However, the utility of classifier outputs is different in the two models, leading to different ways of assessing the classifier.

In the \textit{subjective utility model}, each item $i$'s rater response state is an independent draw from a distribution over possible response states $\itemResponseStateRV$. The top of Fig.~\ref{fig:noisy_signals}  depicts a distribution, with the selected item having $\itemResponseState=0.8$, depicted as an urn with 8 orange and two blue balls. The utility of a classifier output depends on matching individual rater judgments, which are independent draws from the urn. This model treats different rater labels for the same items as indicators of valid differences in perspectives and preferences rather than errors. There is no objective ground truth but rather a subjective one, the distribution $\itemResponseState$ of valid human judgments for each item. The entire distribution will not be observable to the analyst for any item; one of our main insights is that using a single label provides an unbiased estimate of the expected utility.

In the \textit{objective utility model}, there is a different model of the data generating process. Each item $i$ has an underlying, hidden ground truth, $\itemGT$. It and the rater response state are a draw from a joint distribution. For example, in the bottom of Fig.~\ref{fig:noisy_signals}, there are 70\% negative items, and $\itemResponseStateRV | \itemGT = \textit{neg}$ is shown as the blue portion of the graph; there are 30\% positive items, and $\itemResponseStateRV | \itemGT = \textit{pos}$ is shown as the orange portion. The selected item has a positive ground truth label and a rater response state of $0.8$, again depicted as an urn with 8 orange and two blue balls. The utility of a classifier output depends only on matching the ground truth, not on how accurately or consistently raters perceive that ground truth. In this model, different rater labels for the same item are indicators of noise in the process of perceiving the objective ground truth. There is an objective ground truth, but it may not be directly observable at the time that the classifier makes decisions. If the ground truth is not even observable for evaluation items, then the analyst may try to use human labels as a proxy, but we provide negative results about the reliability of that approach.

For example, consider the task of determining whether posts in a subreddit violate that community's standards. Suppose there are two posts, with rater response states $0.9$ and $0.6$. 
In the subjective utility model, we think of violating community standards as a continuous property. The different response states represent that the first post \emph{is more violative} than the second, because more of the relevant population would agree that it is violative.
In the objective utility model, we think of violating community standards as a binary property; a post either objectively violates or does not violate the standards. Assuming both posts are truly violative, the different response rates reflect noise levels, with the first's violation being \emph{easier to detect}.
Simply put, in the subjective model, posts are more or less violative while in the objective model, violative posts are easier or harder to detect.

More formally\footnote{A table of notations appears in Appendix~\ref{app:notation}}, let $\La$ be the set of possible labels that a rater can pick for an item. Given a sequence of items, the $\ith$ item will have a realized set of features $\itemFeatures$, a ground truth state $\itemGT$, and a rater response state $\itemResponseState$. 
$\itemGT$ is a one-hot vector over labels, the single element with value 1 indicating the correct ground truth label.
$\itemResponseState$ is a vector of discrete probabilities for each of the labels.
$\itemFeatures$, $\itemGT$ and $\itemResponseState$ are realizations of the random variables $\itemFeaturesRV$, $\itemGTRV$, and $\itemResponseStateRV$ respectively. 
Each item's properties are an i.i.d.\ draw from the joint distribution $D_{\X, \GT, \St}$.\footnote{To avoid degenerate conditions, we assume both that $\itemGTRV$ and $\itemResponseStateRV$ have full support over the set of labels, and that there is more than one possible response state.}

A classifier $c$ operates on the features of an item, producing output $\itemClassifierOutput$, an $|\La|$ dimensional vector with one value for each possible label in $\La$, and values summing to 1. We will refer to $c$ as a hard classifier if $\itemClassifierOutput$ is always a one-hot vector.
Otherwise, we will call $c$ a soft classifier and interpret $\itemClassifierOutput @l$ as the probability assigned to label $l$.

One or more human raters may provide labels for the item. We divide the raters into two groups. The evaluation group's labels for an item are used to score classifiers. The benchmark group's labels are used to simulate human benchmark classifiers. Rather than relying on the generic notation $\genericLabel$ for a human label, we will use $\evaluationRating$ to refer to the $\jth$ evaluation rater's label for the $\ith$ item and $\benchmarkRating$ to refer to the $\jth$ benchmark rater's label for the $\ith$ item. We model both evaluation and benchmark labels as i.i.d. draws from the item's realized response state.
That is: $\forall i \forall j, \evaluationRatingRV \sim \itemResponseState$ and $\benchmarkRatingRV \sim \itemResponseState$.\footnote{If the evaluation raters and benchmark raters are drawn from different populations, an item could have different rater response states for the two populations. For simplicity of exposition, we focus on the setting where both benchmark and evaluation raters are drawn from the same population.}

Throughout the paper, we will illustrate with examples where labels and ground truth are both binary, either positive or negative. 
In that setting, $\itemResponseState$ is a random variable with a Bernoulli distribution, defined by the single parameter $\itemResponseState @ pos$, and a soft classifier's output is similarly defined by the single parameter $\itemClassifierOutput @ pos$. For simplicity, we will abuse notation slightly by using $\itemResponseState$ and $\itemClassifierOutput$ to refer to the probability of a positive label rather than the more cumbersome $\itemResponseState @ pos$ and $\itemClassifierOutput @ pos$. We will also abuse notation in the binary setting by treating $\itemGT$ as an indicator variable that is 1 for a positive label and 0 otherwise, rather than as a more cumbersome one-hot vector.

\subsection{Scoring Functions: Modeling Utility}
\label{sec:utility-models}

Given a realized vector of item feature sets $\itemFeaturesVec$ for $n$ items, a scoring function takes as input a classifier's outputs $\itemClassifierOutputVec$, and a vector of evaluation labels $\scoringLabelVec$. A common practice is to use a single human rater's label for each item. Alternatively, a panel of several raters may be used. The raters' labels are combined into a single \emph{panel label} for each item: if the labels are discrete, the majority vote is typically used as the panel label; if the labels are continuous, the average may be used.
Let $\scoringLabel = \evaluationPanelLabel$ be the panel label constructed from the majority vote of $k_e$ randomly selected raters from the evaluation group. When $k_e=1$, this reduces to using a single evaluation rater's label, $\scoringLabel = \evaluationRatingSingle$.

\setlength{\abovedisplayskip}{0pt}
\setlength{\belowdisplayskip}{0pt}

The scoring function outputs a numeric evaluation score. Popular scoring functions for hard classifiers include percent agreement (accuracy) and F1 score. Popular scoring functions for soft classifiers include area under the receiver operating characteristic curve (AUC) and cross-entropy (CE), which is defined as 
\begin{equation*}
\textsc{CE}(\itemFeaturesVec,\scoringLabelVec)=-\frac{1}{n}\sum_{i =1}^{n} \log\mleft(\itemClassifierOutput @\scoringLabel\mright).
\end{equation*}

To determine what scoring function should be used and what evaluation labels should be used, it is helpful to think of the scoring process as estimating the utility that would result from the action induced by the classifier's outputs. For example, in the radiology scenario, the classifier output might determine whether to take a further round of diagnostic tests, leading to a saved life for a true positive or reduced quality of life for a false positive. In the content moderation scenario, the classifier output might lead to removing a post, pleasing some customers and angering others.

\paragraph{Subjective Utility Model} With subjective ground truth, each person may get a different utility, because they have different subjective judgments about the corresponding items. 
Following Bentham's utilitarian principle, we define the overall utility as the population's average utility. Because our model is based on anonymous ratings, the average utility of the population is the expected utility for any single individual, where the expectation is taken over a randomly drawn evaluation label for each of the items. We use the notation $\evaluationsVecRV \sim \itemResponseStateVec$ to represent that for all items and for all ratings, the random variables $\evaluationRatingRV$ are independent draws from the corresponding rater response state $\itemResponseState$.

\begin{equation*}
\utilityFunction(\itemClassifierOutputVec, \itemResponseStateVec) = \E_{\evaluationsVecRV \sim \itemResponseStateVec} [\scoreFunction(\itemClassifierOutputVec, \evaluationRatingSingleVecRV)]
\end{equation*}

A scoring function is item-separable if it can be computed independently for each item. 
For example, cross-entropy and agreement can be interpreted as item-separable utility functions. Precision, recall, and correlation are not item-separable; indeed, they are defined only for collections of items, not for single items.

When the scoring function is item-separable and labels are discrete, the subjective utility decomposes nicely: 
\begin{align*}
\utilityFunction&\left(\itemClassifierOutputVec, \itemResponseStateVec\right) \\
&= \E_{\evaluationsVecRV \sim \itemResponseStateVec} \left[\scoreFunction\left(\itemClassifierOutputVec, \evaluationRatingSingleVecRV\right)\right] \\
& =\sum_{i}\E_{\evaluationRatingSingleRV \sim \itemResponseState} \left[\scoreFunction(\itemClassifierOutput, \evaluationRatingSingleRV)\right] \tag{item-separable} \\
& =\sum_{i}{\sum_{\ell\in\La}\scoreFunction\left(\itemClassifierOutput,\ell\right) \cdot \itemResponseState @ \ell } \tag{discrete labels}
\end{align*}

\paragraph{Objective Utility Model} In the objective utility model, the utility  depends only on the objectively correct label for each item, $\itemGT \in \La$, not the raters' perceived item labels nor rater response states from which they are drawn. The utility is expressed in terms of a scoring function that takes a vector of classifier outputs and a vector of ground truth labels as inputs.
More formally:
\begin{equation*}
\utilityFunction^{GT}(\itemClassifierOutputVec, \itemGTVec) = \scoreFunction(\itemClassifierOutputVec, \itemGTVec)
\end{equation*}

Note the contrast with the objective utility model. There, the utility is the \emph{expected} score against a random rater's realized labels. Here, the utility is the actual score against the realized ground truth labels. 

With a scoring function that is item-separable: 
\begin{align*}
\utilityFunction^{GT}&\left(\itemClassifierOutputVec, \itemGTVec\right) \\
&= \scoreFunction(\itemClassifierOutputVec, \itemGTVec) \\
&=\sum_{i} \scoreFunction(\itemClassifierOutput,\itemGT) \tag{item-separable}
\end{align*}




\section{Ideal Evaluation Panel Size}
\label{sec:danger-of-majority-vote}


We now consider the ideal evaluation panel size for estimating the utility of a classifier, given data for a sample of items. It is common practice to use the majority vote of a panel of three or more raters as the evaluation label for an item when multiple labels are available. Presumably, this stems from the intuition that one person's label may not reflect what most people think, but the majority vote of several people's labels will be more likely to reflect the population's preferences. However, we show that this does not imply that the majority vote of a panel of three or more raters will provide a better estimate of utility than a single rater's label. In fact, it can provide a worse estimate.

As a shorthand, we will refer to a classifier's score against a panel of $k_e$ raters' labels (or a single rater when $k_e=1$) as its \textit{evaluation panel score}. Formally, for odd $k_e$, let $\evaluationPanelLabel$ represent the majority vote from $k_e$ evaluation labels for the $\ith$ item. When $k_e=1$, this is just an individual label $e_{i1}$. Analogously, we let the random variable $\evaluationPanelLabelRV$ represent a process of collecting $k_e$ evaluation labels, each an i.i.d. draw from $\itemResponseState$, and taking the majority vote as the panel label. 

This section reports results for the subjective utility model. We prove that it is best to use a single evaluation rater. Scoring against the majority vote of larger evaluation panels yields a biased estimate of subjective utility. Instead, the best way to make use of multiple evaluation raters is to compute the classifier's score against each rater individually, and then take the average of the scores.

Appendix~\ref{appdx:objective-utility} presents results for the objective utility model. There, the results are mixed. Larger evaluation panels will sometimes be better. However, there is no theoretical justification for picking any particular evaluation panel size, and there are scenarios where larger panels can yield an incorrect ordering of classifiers where evaluating against a single rater would give the correct ordering.



\subsection{Estimating Absolute Utility}
\label{sec:subjective-gt}


\begin{restatable} [Individual Label Works for Estimating Absolute Utility]{claim}{claimsubjectivejudgmenet} \label{claim:subjectivejudgmenet}

With subjective ground truth, evaluating against one randomly selected rater's labels is an unbiased estimator of the overall utility.
\begin{align*}
\E_{\itemFeaturesVecRV,\itemResponseStateVecRV}&\mleft[\utilityFunction(\itemClassifierOutputVecRV,\itemResponseStateVecRV)\mright]=\\
&\E_{\itemFeaturesVecRV, \itemResponseStateVecRV}\E_{\evaluationsVecRV \sim \itemResponseStateVecRV} \mleft[\scoreFunction(\itemClassifierOutputVecRV, \evaluationRatingSingleVecRV)\mright]
\end{align*}
\end{restatable}
\begin{proof}
Recall that with subjective judgments, the overall utility reflects the average benefit or satisfaction that the population gets from all items:
\begin{equation*}
\utilityFunction(\itemClassifierOutputVec, \itemResponseStateVec) = \E_{\evaluationsVecRV \sim \itemResponseStateVec} [\scoreFunction(\itemClassifierOutputVec, \evaluationRatingSingleVecRV)]
\end{equation*}
Since the equality holds for every realized draw of items with feature sets $\itemFeaturesVec$ and response states $\itemResponseStateVec$, it also holds in expectation for a randomly selected set of items. 
\end{proof}


In contrast, as stated formally below in Claim~\ref{claim:subjnonunbias}, using a panel label $\evaluationPanelLabel$ derived from the majority vote of three or more panelists' ratings will provide a distorted estimate of the utility.






\begin{restatable} [Panel Labels Fail]{claim}{claimsubjnonunbias} \label{claim:subjnonunbias}

With subjective ground truth, $\forall k > 1$, there exists a scenario such that: 
\begin{align*}
\E_{\itemFeaturesVecRV,\itemResponseStateVecRV}&\mleft[\utilityFunction(\itemClassifierOutputVecRV,\itemResponseStateVecRV)\mright] \neq \\ 
&\E_{\itemFeaturesVecRV, \itemResponseStateVecRV}\E_{\evaluationsVecRV \sim \itemResponseStateVecRV}\mleft[\scoreFunction(\itemClassifierOutputVecRV, \evaluationPanelLabelRV)\mright]
\end{align*}
\end{restatable}
\begin{proof}
The key insight is that whatever label is most frequent will be over-represented in the majority votes.
The proof is by counter-example. Set the individual utility as the percentage of agreements between the classifier's outputs and labels. There is a single rater response state,  $0.6$. The classifier always outputs `pos'. The overall utility for any item, then, will be the expected agreement with a random label, which is the proportion of positive individual labels, $0.6$. However, for $k_e>1$, the expected fraction of positive panel labels will be greater than $0.6$. Indeed, as $k_e \rightarrow \infty$, the measured utility score against panel labels will approach $1$, diverging farther and farther from the correct expected utility of $0.6$. 
\end{proof}

\subsection{Estimating Relative Utility}

The same results extend to the relative ordering of two classifiers under subjective utility. Of course, an unusual sample of items could always lead to an ordering of observed evaluation panel scores that is reversed from the ordering of expected utilities. We are concerned, however, about systematic reversals where the expected scores against human evaluation panels are reversed from the ordering of expected utilities.

\setlength{\abovedisplayskip}{0pt}
\setlength{\belowdisplayskip}{0pt}

\begin{definition}[Reliable Ordering] 
\label{def:reliableordering}
An evaluation panel size $k_e$ gives a \textit{reliable ordering} if the ordering of expected evaluation panel scores matches the ordering of expected utility for classifiers: 
\begin{align*}
\E_{\itemFeaturesVecRV,\itemResponseStateVecRV}&\utilityFunction(\itemGivenClassifierOutputVecRV{c_1},\itemResponseStateVecRV) \ge \\ &\E_{\itemFeaturesVecRV,\itemResponseStateVecRV}\utilityFunction(\itemGivenClassifierOutputVecRV{c_2},\itemResponseStateVecRV) \\
& \qquad\iff \\
\E_{\itemFeaturesVecRV, \itemResponseStateVecRV}&\E_{\evaluationsVecRV \sim \itemResponseStateVecRV}\mleft[\scoreFunction\mleft(\itemGivenClassifierOutputVecRV{c_1},\evaluationPanelLabelVecRV\mright)\mright] \ge \\
& \E_{\itemFeaturesVecRV, \itemResponseStateVecRV}\E_{\evaluationsVecRV \sim \itemResponseStateVecRV}\mleft[\scoreFunction\mleft(\itemGivenClassifierOutputVecRV{c_2},\evaluationPanelLabelVecRV\mright)\mright]
\end{align*}
\end{definition}

\begin{restatable} [Individual Labels Work for Ordering Classifiers]{claim}{claimsubjectivesuccessrelative} \label{claim:subjective-success-relative}

In the subjective utility model, $k_e=1$ yields a reliable ordering of classifiers. 
\end{restatable}
\begin{proof}
From Claim~\ref{claim:subjectivejudgmenet}, we have, for each classifier: 
\begin{align*}
\E_{\itemFeaturesVecRV,\itemResponseStateVecRV}&\utilityFunction(\itemClassifierOutputVecRV,\itemResponseStateVecRV) =\\
&\E_{\itemFeaturesVecRV, \itemResponseStateVecRV}\E_{\evaluationsVecRV \sim \itemResponseStateVecRV} \scoreFunction(\itemClassifierOutputVecRV, \evaluationRatingSingleVecRV)
\end{align*}
Thus, the ordering of expected scores for the two classifiers also gives the correct ordering of the expected subjective utility.
\end{proof}

\begin{restatable} [Larger Panels Fail for Ordering Classifiers]{claim}{claimsubjectivefailure} \label{claim:subjective-failure}

There exists a scenario where $\forall k_e>1$, an unreliable ordering happens. 
\end{restatable}
\begin{proofsketch}
We construct a case with two classifiers—$c_1$ that always predicts positive, and $c_2$ that always predicts negative—evaluated under agreement-based utility. The true utility of $c_1$ exceeds that of $c_2$, since it better aligns with the average rater response across item types.
However, for any $k_e > 1$, majority vote introduces bias: it over-represents the most probable label within each response state. Full details and calculations appear in Appendix~\ref{app:evalbenchmarkclass}.

\end{proofsketch}

In summary, with the subjective utility model, scoring should be based on individual rater evaluations rather than combining panel ratings through a majority vote. Intuitively, scoring against a single rater rewards a classifier in proportion to the fraction of the population that agrees with its output on each item. By contrast, scoring against the majority vote of a large evaluation panel rewards a classifier only for matching the most common label on each item, regardless of the margin. When these diverge — for example, when a classifier does well on items with strong consensus but poorly on items where opinion is nearly split — scoring against larger panels can provide a misleading result.

There is still a benefit from having multiple evaluation raters available. But a separate score should be computed using each rater's evaluation labels and the scores should be averaged. The mean of scores for individual raters provides an unbiased estimate of utility, with a lower variance than using only one rater.

\section{Human Panels as Benchmarks: Theory}
\label{sec:human-benchmarks}

In a setting where an automated classifier is already used as part of decision processes, and the question is whether to replace it with another one, the currently used classifier should be the benchmark for comparison.
If, however, the new classifier is not replacing any existing classifier, it may be useful to compare the classifier's performance to that of a human panel, especially if the practical alternative to using the classifier is to rely on judgments of human panels.

There are many ways to organize a panel with multiple people. They may discuss each item and come to a consensus judgment. Or they may make independent judgments. Or there might be a more complex workflow, such as two people making independent judgments, and bringing in a third person when they disagree. In this paper, we consider only panels with the same number of raters for each item and with each rater making an independent judgment.

We first consider alternative ways of combining those independent judgments into a panel label. We will refer to this as the benchmark combiner. We then define a theoretical power curve based on the expected score for panels of different sizes and define the rater equivalence of the classifier as the point where the power curve first reaches the classifier's expected score. Finally, we remark on the importance of calibration for the benchmark classifier, as an uncalibrated benchmark classifier may artificially raise the power curve and inflate the rater equivalence score. Section~\ref{sec:human_benchmark_implementation} explores how to empirically estimate power curves from a dataset of rater labels.


\subsection{Benchmark Combiner Functions}

For a particular item $i$, let $\benchmarkRatingPanel = \benchmarksVec_{i1}, \benchmarksVec_{i2}, \ldots, \benchmarksVec_{ik_b}$, represent a vector of $k_b$ human benchmark labels. 
A benchmark combiner $\comb$ is a classifier that ignores all properties of the item and instead takes as input the benchmark panel labels for an item.

A hard combiner outputs a single label.
For example, with binary labels, the \emph{majority vote combiner} outputs the label assigned by the majority of the raters, picking at random to break ties. The \emph{plurality combiner} generalizes to a setting with more than two labels; it outputs the most common label even if less than half of the raters assign it. The \emph{average vote combiner} works in settings with continuous numeric labels by taking the mean of the labels.

A soft combiner produces a probability distribution over discrete labels instead of a single label. 
For example, the \emph{frequency combiner} outputs, for every possible label, the frequency of that label in the $k_b$ observed labels. 

A myopic combiner makes use of only the observed labels for the particular item, ignoring other items' labels.
All of the example combiners mentioned above are myopic. 
In Section~\ref{sec:abc} we will introduce a non-myopic combiner, the Anonymous Bayesian Combiner,  that is learned from a dataset; the output for one item will depend on the pattern of labels observed for the rest of the items.


\subsection{Power Curve}
\label{sec:power-curve}

The power score for a particular combiner and panel size $k_b$ is the expected score for a benchmark panel of that size, intuitively, the expected utility of using such panels. The expectation is taken over a randomly selected set of items and a randomly selected set of human benchmark labels for those items. As $k_b$ increases, so does the expected score (for an optimal combiner) because the benchmark panel labels will provide more information about the rater response state, and thus about the evaluation labels. 

More formally, let $\benchmarkRatingPanelInfRV$ be an infinite height matrix of random variables corresponding to a process of randomly selecting an infinite number of items and then selecting $k_b$ benchmark labels for each item as i.i.d. draws from the rater response state $\itemResponseStateRV$. $\benchmarkPanelLabelInfRV$ is an infinite vector of induced random variables, one for each item, corresponding to running the combiner on the $k_b$ benchmark labels for each item. 
$\comb(\benchmarkRatingPanelRV)$ represents the combiner's output for the $\ith$ item.



The expected score for a benchmark classifier's outputs may depend, for some benchmark combiners and scoring functions, on the number of items. Thus, it is defined as a limit as the number of items grows.\footnote{We restrict ourselves to well behaved utility functions where the limit exists.}

\begin{definition} [Power Score]
\label{def:powerscore} 

\begin{align*}
& \powerscore_{\comb}(k_b) = \\
& \quad\lim_{n\rightarrow \infty}\E_{\itemResponseStateVecRV}\E_{\benchmarksVecRV,\evaluationsVecRV \sim \itemResponseStateVecRV}\mleft[\scoreFunction(\benchmarkPanelLabelVecRV, \evaluationRatingSingleVecRV)\mright]
\end{align*}
\end{definition}

Note that for the subjective utility model, the benchmark labels are drawn from the rater response state $\itemResponseStateRV$ and the vector of rater response states is also a direct input to the utility function. Wherever human labels are generated, they are all i.i.d. draws from the same rater response state. Thus, the rater response states affect both the benchmark labels and, in the subjective utility model, the evaluation labels. 


\begin{definition}[Power Curve]
\label{def:powercurve}
We define the power curve $\powercurve_{\comb}(x)$ as a piecewise linear function that matches the power score at every natural number, and for other values is defined by linear interpolation between those points.\footnote{For example, for $k_b \leq x \leq k_b+1$, 
$\powercurve_{\comb}(x) = (x - k_b) \powerscore_{\comb}(k_b{+}1) + (1{-}(x - k_b)) \powerscore_{\comb}(k_b)$.} The power at a non-integer value between $k_b$ and $k_b + 1$ can be interpreted as the expected score for a benchmark panel that sometimes has $k_b +1$ raters and sometimes one less. For example, a panel of 1.96 raters would have two raters for 96\% of the items and one rater for the rest of the items.
\end{definition}

\subsection{Rater Equivalence}

The expected score (i.e., utility) for the classifier may also depend, for some scoring functions, on the number of items. It is again defined as the limit, as the number of items grows.

\begin{definition} [Classifier Expected Score]
\label{def:classifierscore}

\begin{align*}
\classifierescore 
&= \lim_{n\rightarrow \infty}\E_{\itemFeaturesVecRV, \itemResponseStateVecRV}[\scoreFunction(c(\itemFeaturesVecRV), \itemResponseStateVecRV) ]
\end{align*}

\end{definition}

Given a classifier score \( \classifierescore \), we aim to find a benchmark panel size that produces the same power score. 




\begin{definition}[Rater Equivalence]
\label{def:ratereq}
\label{def:idealraterequivalence}

\begin{equation*}
\fratereqS(c) = \min \Big\{x|\powercurve_{\comb}(x) \geq \classifierescore\Big\}.
\end{equation*}
\end{definition}

If there is no such $x$, we define the rater equivalence as $\infty$.


Graphically, when the power curve is increasing, the rater equivalence is the x-value for the point where the horizontal line representing the score of the classifier intersects the power curve (1.96 in Figure~\ref{fig:ce_power_curve}, corresponding to a benchmark panel that has two raters for 96\% of items and one rater for the rest). 
Intuitively, this is the smallest panel size that would have the same expected utility as the classifier.
In the perverse case where the power curve is non-monotonic and there are multiple points of intersection, the rater equivalence is either the leftmost intersection point or zero.

\subsection{Calibrated Human Benchmark Classifier $\Rightarrow$ Minimal Rater Equivalence}
\label{sec:calibration_minimizes_rater_equiv}


Rater equivalence is affected by the choice of combiner function. If the combiner function is not effective, the human benchmark classifier is weak. This leads to an artificially high rater equivalence, making the classifier appear better than it is.  In this section, we discuss using a calibrated combiner, which minimizes rater equivalence with respect to a family of utility functions, including the commonly used cross-entropy. 

Intuitively, a calibrated classifier is one whose outputs can be correctly interpreted as event frequencies~\citep{b2d25122-884a-3c72-a8a5-06152481379b}. For example, if a weather forecaster is calibrated, it will rain on 40\% of the days on which it reports a 40\% chance of rain. Below is a more formal definition that applies more generally to settings where the evaluation labels may not directly correspond to an observable ground truth. 

\begin{definition}[Calibration]~A classifier $c$ is calibrated with respect to a joint distribution of input $Z$ and evaluation label $Y$ if, for all possible realized outputs $c(z)$:
\begin{equation*}
c(z) = \E_{Y, Z  \mid c(Z) = c(z)}\mleft[Y\mright].
\end{equation*}
\end{definition}


A realized evaluation label is always a discrete label. When it is represented as a one-hot vector, $\E(Y)$ is a vector of probabilities, one for each of the possible labels. Thus, $\E_{Y, Z \mid c(Z) = c(z) }[Y]@pos$ can be interpreted as the Bayesian posterior probability that the evaluation label is positive, given a particular observed realized classifier output $c(z)$.


In the context of a human benchmark classifier, the input distribution $Z$ is the distribution of benchmark labels $\benchmarkRatingPanelRV$. 
In the subjective utility model, the combiner that is \emph{calibrated with respect to item response states}, $\combCaliSub$, has the property that for any item $i$ and any realization of $k_b$ labels, $\benchmarkRatingPanel$:
\begin{equation*}
\benchmarkPanelLabelCalibratedS = \E_{\itemResponseStateRV, \benchmarksVecRV_i \sim \itemResponseStateRV \mid \benchmarkRatingPanelRV = \benchmarkRatingPanel}[\itemResponseStateRV ]
\end{equation*}




In other words, the calibrated benchmark classifier is the one that produces the Bayesian posterior probability of a label drawn from the rater response state. Here, the posterior is conditional on having observed the realized benchmark labels for an item.

Note that simple combiners, such as majority vote or frequency, are not inherently calibrated. They do not even consider the prior, the base probability of the rater response state, and thus do not produce updated posterior probabilities. In Section~\ref{sec:abc}, we introduce a benchmark combiner, learned from benchmark data, that approaches the calibrated combiner $\benchmarkPanelLabelCalibratedS$ if the benchmark dataset that it learns from is large enough.

\begin{restatable}[Proper Scoring Rules Lead to Maximum Power Score]{claim}{propmaxpower}\label{prop:maxpower}

If the utility function corresponds to the cross-entropy scoring function, or any other proper scoring rule, the calibrated human benchmark classifier ($\comb$) leads to the maximum power score. That is, for any classifier $c$ and any potential alternative classifier $\comb'$:
\begin{align*}
\powerscore_{\comb}(k_b) &\geq \powerscore_{\comb'}(k_b).
\end{align*}
\end{restatable}
\begin{proofsketch}
Proper scoring rules, such as cross-entropy, are designed so that truthful, calibrated predictions maximize expected score~\citep{723c01ad-4d66-3cb4-987e-5f56192b514e, Gneiting01032007}. In our setting, the calibrated human benchmark classifier corresponds to the Bayesian posterior over human labels. By the definition of proper scoring rules, this benchmark achieves the highest expected score compared to any alternative prediction strategy.

As a result, when power scores are defined using a proper scoring rule, the calibrated benchmark always attains the maximum power score. No alternative combiner can outperform it in expectation. A full formal justification appears in Appendix~\ref{app:humanpanel}.
\end{proofsketch}

\begin{restatable}[Calibrated Classifier Plus Proper Scoring Rules Lead to Minimal Rater Equivalence]{claim}{propminre}\label{prop:minre}

If the utility function corresponds to the cross-entropy scoring function, or any other proper scoring rule, the calibrated human benchmark classifier ($\comb$) leads to the minimal rater equivalence. That is, for any classifier $c$ and any potential alternative classifier $\comb'$:
\begin{align*}
\fratereqcaliS(c) &\leq \fratereqS(c).
\end{align*}
\end{restatable}
\begin{proof}
Given Claim~\ref{prop:maxpower}, the calibrated combiners yield maximal power scores for all $k_b\in\mathbb{N}$. Since power curves are generated through linear interpolation of power scores, the calibrated combiners result in power curves that lie pointwise (weakly) above the power curve induced by other combiners. 

The rater equivalence is defined as the smallest panel size with power score greater than the expected score of the classifier. Substituting the calibrated human benchmark classifier for any other human benchmark classifier will not decrease the power score. Thus, it will not increase the rater equivalence. 
\end{proof}


\section{Human Panels as Benchmarks: Practice}
\label{sec:human_benchmark_implementation}

Empirically, we work with a rating matrix $\allVec$ with $n$ rows (items) and $\numraters$ labels for each item, as well as classifier outputs $c(\itemFeaturesVec)$ for those items. 
We use the empirical data to compute an empirical power curve and empirical classifier scores, from which we generate an empirical rater equivalence.

Section~\ref{subsec:estimating-directly-from-data} defines the process in more detail. Section~\ref{sec:empirical-approximation} explores the relationship between empirical and theoretical values for the power curve and rater equivalence. Section~\ref{sec:abc} describes the Anonymous Bayesian Combiner, which outputs the Bayesian posterior over labels, thus assuring that empirical rater equivalence is minimized when cross-entropy is the scoring function. Finally, section~\ref{sec:running_time} analyzes the running time of the empirical estimation process, showing that it is efficient enough to use in practice.



Appendix~\ref{app:simulatingpanels} explores a second method that simulates panels by generating synthetic data. From the empirical rating matrix and classifier outputs, it estimates the joint distribution of item response states and classifier outputs. From that inferred distribution, it is possible to compute power scores for any finite number of raters. The challenge, however, is the difficulty of accurately estimating the joint distribution. In practice one would need to make a heroic assumption about a parametric model for the distribution and then risk misestimating the parameters of the model due to insufficient data.

\setlength{\abovedisplayskip}{0pt}
\setlength{\belowdisplayskip}{0pt}

\subsection{Simulating Panels by Sampling}
\label{subsec:estimating-directly-from-data}

For any division of $\allVec$ into two panels, a human benchmark panel $\benchmarksVec$ and an evaluation panel $\evaluationsVec$, run the combiner on the benchmark labels for each item, and compute the majority vote of the evaluation labels for each item. Then run the scoring function. 
\footnote{Recall that in Section~\ref{sec:danger-of-majority-vote} we showed that $k_e=1$ is optimal under the subjective utility model but that for the objective utility model there is no theoretical justification for any particular $k_e$. For generality, we treat $k_e$ as a parameter here.}
That yields a score $\scoreFunction(\comb(\benchmarksVec),\textrm{maj}(\evaluationsVec))$.  
We define the empirical power score for a benchmark panel size $k_b$ as the average computed score for many divisions of $\allVec$ into panels where $\benchmarksVec$ has $k_b$ ratings per item.

\begin{definition} [Dataset Partitions] \label{def:allpartitions}
If $\benchmarksVec$ and $\evaluationsVec$ are disjoint sets of columns from $\allVec$, we refer to the pair $(\benchmarksVec, \evaluationsVec)$ as a partition (even though some columns may not be used). 
For fixed $k_b$ and $k_e$ with $k_b + k_e \leq \numraters$, define $\allpartitions$ as the set of all such pairs where $\benchmarksVec$ has $k_b$ columns and $\evaluationsVec$ has $k_e$ columns. Define $\sampledpartitions$ as a random subset of 200 pairs from $\allpartitions$, or all of them if there are fewer than 200.\footnote{Using only a subset of the potential partitions is just a way to reduce the running time of computing the empirical power curve and classifier expected score.}
\end{definition}


\begin{definition} [Empirical Power Score] \label{def:emppowerscore}
Let $\comb$ be a combiner function and let \scoreFunction ~be a scoring function.
The empirical power score is the average score over  $(\benchmarksVec,\evaluationsVec) \in \sampledpartitions$:

\begin{align*}
\empowerscore_{\comb}(k_b, k_e) = 
\frac{
  \sum_{\benchmarksVec, \evaluationsVec}
  \scoreFunction\big(
    \comb(\benchmarksVec),\,
    \textrm{maj}(\evaluationsVec)
  \big)
}{
  |\sampledpartitions|
}
\end{align*}

\end{definition}

\begin{definition} [Empirical Power Curve] 
\label{def:emppowercurve}
The empirical power curve $\empowercurve$ is defined as a linear interpolation of the empirical power scores for integer numbers of raters. It can be calculated only for values up to $\numraters - k_e$, due to the limited number of ratings available for each item.

\end{definition}



Working with empirical data, we also do not know the classifier's score against the item response states. Thus, we compute an empirical classifier score analogously to the empirical power scores. Here, however, we use only the evaluation panel from each partition, ignoring the benchmark panel.
\begin{definition} [Empirical Mean Classifier Score]
\label{def:empirical_mean_classifier_score}
The empirical mean classifier score is the average score, over ($\benchmarksVec,\evaluationsVec) \in \sampledpartitions$:
\[
\emscore = \frac{
\sum_{\benchmarksVec, \evaluationsVec} \scoreFunction\big(
c(\itemFeaturesVec),\,
\operatorname{maj}(\evaluationsVec)
\big)
}{
|\sampledpartitions|
}
\]

\end{definition}

\begin{definition} [Empirical Rater Equivalence] \label{def:empraterequivalence}
The empirical rater equivalence is the smallest benchmark panel size for which the empirical power score is higher than the classifier's score:
\begin{align*}
\emfratereq&(c, k_e) = \\
& \min \Big\{ x \,\Big|\,
\empowercurve_{\comb}(x, k_e) \geq \emscore
\Big\}
\end{align*}


\end{definition}

\begin{remark}[No $k_b$ in range] \hspace{.4em}
    If no $k_b$ value is found where the human benchmark panels' average score is higher than the classifier's average score, we define the empirical rater equivalence as ``at least $\numraters-k_e$''. It could be infinite, or it could be some finite value larger than it was possible to measure given the available ratings.
\end{remark}

\begin{remark}[Choice of $k_e$] \hspace{.4em}
Because the benchmark panel $\benchmarksVec$ and the evaluation panel $\evaluationsVec$ come from the same shared rating panel $\allVec$, we have a binding constraint \( k_b + k_e \leq\numraters \). If we choose a smaller \( k_e \), we increase the maximum $k_b$ for which power scores can be calculated.
\end{remark}

\subsection{Do Empirical Values Approximate Theoretical?}
\label{sec:empirical-approximation}

We now consider whether the empirical power scores, classifier score, and rater equivalence are good proxies for the corresponding theoretical values of interest. We treat the matrix $\allVec$ as a realization of the random variable $\allVecRV$, based on items being independent draws from the unknown underlying distribution $D_{\X, \GT, \St}$ and individual ratings for each item in the cells of $\allVec$ being independent draws from the rater response state $\itemResponseState$ for that row.

We restrict our attention to utilities that are \textit{well behaved}, excluding, for example, pathological utilities that are determined by behavior on one randomly selected item. All common scoring functions are well-behaved according to the following definition.

\begin{definition}[Well Behaved]\label{def:well-behaved} A utility function defined by a scoring function is well-behaved if, for any joint distribution of classifier outputs and item states, as the number of items increases, the limit of the expected score exists and the empirical utility converges in probability (denoted $X_n \xrightarrow{P} X$)\footnote{A sequence of random variables \( X_n \) converges in probability to a random variable \( X \) if for every \( \epsilon > 0 \), $
\lim_{n \to \infty} P\left( |X_n - X| \geq \epsilon \right) = 0$.} to that limit. Formally, for vectors of $n$ items with properties $\itemFeaturesVecRV, \itemResponseStateVecRV$ drawn IID from any $D_{\X, \St}$,
\begin{align*}
\utilityFunction&\mleft(c(\itemFeaturesVecRV), \itemResponseStateVecRV\mright) \xrightarrow{P} \\
&\lim_{n\rightarrow \infty}\E_{\itemFeaturesVecRV,\itemResponseStateVecRV}[\utilityFunction(\itemClassifierOutputVecRV,\itemResponseStateVecRV)].
\end{align*}
\end{definition}


In the subjective utility model, it is reasonable to interpret the empirical rater equivalence with single rater evaluation panels ($k_e=1$) as a proxy for the theoretical rater equivalence. With few items, it gives a biased estimate on the high side, as we will see, but as the number of items grows it approaches the theoretical rater equivalence. We state results and provide some intuitions here. Full details appear in Appendix~\ref{app:empirical}.

\begin{restatable} [empirical power score approximates theoretical]{claim}{claimempiricalunbiased} 
\label{claim:empiricalunbiased}

In the subjective utility model, if we set $k_e=1$, for any $k_b \leq \numraters-1$, as the number of items $n$ in $\allVec$ increases, the empirical power score converges in probability to the theoretical power score:
\begin{equation*}
\empowerscore_{\comb}(k_b, 1) \xrightarrow{P} \powerscore_{\comb}(k_b)
\end{equation*}
\end{restatable}
\begin{proofsketch}
The key observation is that the theoretical and empirical processes both generate benchmark and evaluation ratings as random draws from the same underlying item response states.
Although the empirical score is computed from a finite matrix of rater labels, each sampled partition of that matrix is probabilistically equivalent to a set of independent draws under the theoretical model. As a result, the expected score for any partition matches the expected theoretical score. Given this equivalence, and assuming the scoring function is well-behaved, the law of large numbers ensures that the empirical average converges in probability to the theoretical expectation. Full details are provided in Appendix~\ref{app:empirical}.
\end{proofsketch}

\begin{restatable} [empirical classifier score approximates theoretical]{claim}{claimempclassifierconvergence}
\label{claim:empirical_classifier_score_limit}

In the subjective utility model, if we set $k_e=1$, as the number of items in $\allVec$ increases, the empirical classifier score $\emscore$ converges in probability to the expected classifier score:
\begin{equation*}
\emscore \xrightarrow{P} \classifierescore
\end{equation*}
\end{restatable}
\begin{proofsketch}
This result mirrors the power score convergence argument. With $k_e = 1$, each classifier score is computed against a single randomly sampled rater per item. Since both the empirical and theoretical processes draw rater responses from the same distribution, their expectations match. As the number of items grows, the empirical average over samples from $\allVec$ converges in probability to the expected utility, assuming the scoring function is well-behaved. Averaging over multiple sampled partitions further reduces variance without biasing the result. Full details appear in Appendix~\ref{app:empirical}.
\end{proofsketch}

The empirical rater equivalence will be a biased estimator of the theoretical rater equivalence, due to the convexity of the power curve. Averaging over many rater equivalence values is not the same as finding an equivalence between average scores.
For example, suppose that the theoretical power curve rises sharply between benchmark panel sizes 3 to 5, but more slowly from 5 to 7, and that the theoretical rater equivalence is 5. If we randomly generate many empirical matrices $\allVec$, we will sometimes have an empirical rater equivalence of 7 or even more, since the theoretical power score is rising slowly beyond 5 raters. But observing a rater equivalence of 3 or below would be much less likely, because of the large gap between the theoretical power curve and the theoretical classifier score in that region. Thus, the expected empirical rater equivalence will be above 5. 

However, we show that the empirical rater equivalence does converge in probability to the theoretical rater equivalence as the number of items grows. Intuitively, with enough items the empirical power score and classifier score will be arbitrarily close to their theoretical values, so the slope of the power curve will have little impact on the rater equivalence. To ensure convergence, we additionally assume that the power curve is monotonically increasing. This is a mild assumption, because any reasonable combiner function should, in expectation, provide better predictions of the next label when it combines more benchmark labels.

\begin{restatable} [empirical rater equivalence approximates theoretical]{claim}{thmereqconverge}  \label{claim:ereqconverge}

In the subjective utility model with $k_e=1$, if the theoretical power curve is increasing and the theoretical rater equivalence $\fratereqS(c)$ is in $(0,\numraters-1)$, as the number of items in $\allVec$ increases, the empirical rater equivalence converges in probability to the theoretical rater equivalence.
\begin{equation*}
\emfratereq(c, 1) \xrightarrow{P}  \fratereqS(c)
\end{equation*}
\end{restatable}

\begin{proofsketch}
The theoretical rater equivalence is defined as the inverse of the power curve evaluated at the classifier's expected score. For large $n$, the empirical power curve converges to the theoretical one (Claim~\ref{claim:empiricalunbiased}) and becomes strictly increasing with high probability, making it invertible.

In parallel, the empirical classifier score converges to its expectation (Claim~\ref{claim:empirical_classifier_score_limit}). Together, these imply that the empirical rater equivalence---defined as the inverse of the empirical power curve applied to the empirical classifier score---converges in probability to the theoretical rater equivalence. 
\end{proofsketch}



\subsection{ABC: A Calibrated Human Benchmark Combiner}
\label{sec:abc}

We now introduce the Anonymous Bayesian combiner, $\abc$, which empirically approximates the theoretical calibrated combiner $\comb$ for the subjective utility model. We can think of the $\abc$ in two parts, a learner and an executor. The learner calculates the probability of a random draw of $b_k$ ratings for a randomly drawn item in the dataset producing any realized label sequence $\benchmarkRatingPanel$. The executor then uses the learned frequencies to predict a next rater's label for an item, conditional on some observed labels.
Intuitively, if two positive labels and a negative have been followed by a positive label on $90\%$ of other items, and the current item has received two positive and one negative label, $\abc$ will predict $0.9$.  Appendix~\ref{app:abc} provides details of the algorithm.

The $\abc$ is anonymous in the sense that it does not try to make any inferences based on the identities of the raters. It does not customize its predictions to the particular rater who will provide the next label, and it does not use any information about idiosyncratic rating patterns of the previous labelers; it treats all labels as if they were realized independent draws from an item's unknown rater response state, $\itemResponseState$. No assumptions are made about the distribution of rater response states for items, but it is assumed that all items' rater response states were realized independent draws from that distribution.

\begin{restatable}[ABC is calibrated]{claim}{abciscali}
\label{lem:abc_is_cali}

As the number of items $n$ approaches infinity, the anonymous Bayesian combiner converges to the Bayesian posterior, which is the combiner that is calibrated with respect to item response states, $\combCaliSub$. That is: for all $k_b<k_w$,
\begin{align*}
\lim_{n \rightarrow \infty} &\E_{\allVecRV}\mleft[\abc\mleft(\benchmarkRatingPanel, \allVecRV_{[:n][:]}\mright)\mright] \\
=&\E_{\itemResponseStateRV, \benchmarksVecRV_i \sim \itemResponseStateRV \mid \benchmarkRatingPanelRV = \benchmarkRatingPanel}\mleft[\itemResponseStateRV \mright]
\end{align*}
\end{restatable}

\begin{restatable}[ABC Leads to Maximal Power Scores and Minimal Rater Equivalence]{claim}{thmabsminratereq}\label{thm:abs_min_ratereq}

In the subjective utility model, if the utility function corresponds to the cross-entropy scoring function or any other proper scoring rule, in the limit as the number of items becomes infinity, the anonymous Bayesian combiner leads to maximal power scores $\{\powerscore_{\comb}(k_b)\}_{k_b \in \mathbb{N}}$ for all $k_b\leq \numraters-1$ and the minimal rater equivalence.
\end{restatable}

See Appendix~\ref{app:abc} for proofs.

\subsection{Running Time Analysis}
\label{sec:running_time}

To assess the practicality of the empirical processes, we analyze the running time. With simple combiners, such as majority vote and frequency, the total running time is proportional to the number of items times the square of the number of ratings per item. The Anonymous Bayesian Combiner's implementation requires more computation, especially with a naive implementation. However, memoization can make it computationally tractable.

\begin{restatable}[Running time of empirical power curve computation]{claim}{thmrunningtime}\label{thm:runningtime}

To compute the empirical power curve, when we use the frequency combiner, there exists an implementation such that the total running time is $O(n k_w^2)$; when we use Anonymous Bayesian Combiner, there exists an implementation such that the total running time is 
$O(n k_w^{|\La|+1})$.
\end{restatable}

The proof can be found in appendix~\ref{app:running_time}.

\section{Case Studies}
\label{sec:case-studies}

We illustrate the rater equivalence calculation through three case studies of previously published research where the ground truth was subjective and the labels were provided by human raters. In the first, the items were comments made on Wikipedia and the classifier was the initial version of the Jigsaw personal attacks classifier described in \citet{Wulczyn2017machina}. In the second, the items were news articles and the classifier was the one described in \citet{mitra2015credbank}. In the last, the items were pairs of images, and the classifier selected whichever image had accumulated a higher net upvote score on Reddit. In all three cases, multiple human ratings were available for each item, making it possible to compute a power curve and use held out raters to score both human benchmark classifiers as well as the classifiers. 

\begin{figure}[t]
     \centering
     \begin{tikzpicture}
\sffamily
\begin{axis}[
title = {ABC + Cross Entropy},
title style={align=center,yshift=-.1in},
legend style={font=\small,
	nodes={scale=0.7, transform shape},
	at={(0.0,1)},
	anchor=north west,
	draw=none,
	fill=none},
legend cell align={left},
width = 2.9in, height = 2.0in,
ylabel near ticks,
ylabel = {\small Info Gain ($c_k - c_0$)},
xlabel near ticks,
every tick label/.append style={font=\scriptsize},
xmin=-1,xmax=20,ymin=-0.2,ymax=0.5,
xtick={0, 4.24481358024944, 6, 7.630604332572648, 12, 15, 18},
xlabel={\small Number of raters},
xlabel style = {yshift=0.05in},
yticklabel style={
		/pgf/number format/fixed,
		/pgf/number format/precision=5
},
scaled y ticks=false
]

\addplot[solid, mark=o, mark options={scale=.8}, black]
plot [error bars/.cd, y dir = both, y explicit]
table[y error minus index=2, y error plus index=3]{
0	0.0	0.07525221458878162	0.05442542825419183
1	0.04936886784476313	0.08244307408780027	0.057326979832063696
2	0.08967483698039314	0.08308442776472269	0.05815275851045898
3	0.11399828715692373	0.07993917653486227	0.05770353453101118
4	0.1325425310709319	0.0834731800685713	0.05909259764601049
5	0.14690264061362823	0.08545050984834313	0.06026661638081077
6	0.1544160030071844	0.0830120258915803	0.058494214509282105
7	0.16972734463160244	0.07996888670249125	0.05667789460183281
8	0.17562078237907613	0.08585136555029205	0.0607082348017709
9	0.1810572497234622	0.08140301148472667	0.058213047232555326
10	0.1847156549820811	0.08036826770369804	0.05755891325532708
11	0.18924499518116428	0.07833138547306939	0.05684024740091331
12	0.19632511546171888	0.07824008393738408	0.05714416130063399
13	0.2001985098097599	0.08061770683693442	0.058455541435198544
14	0.20140807379596226	0.08249804258703586	0.059381011929158234
15	0.2055948961157512	0.07972533297491885	0.05724936082568094
16	0.20681277093253914	0.07946237373035142	0.057600105297727855
17	0.21195816158953873	0.0816174590864786	0.059076532684115224
18	0.20967809252439257	0.0794549626434039	0.05749578629389873
19	0.21375287621620143	0.079790656916293	0.05846538174218974
};

\addlegendentry{$k$ raters}

\addplot[mark=none, black, dashed, thick, samples=2, domain=-1:31] {0.13605808090085353};

\addlegendentry{Jigsaw Personal Attack Classifier}

\draw [black, fill=black, opacity=0.1] (axis cs:-1,0.04096808753102943) rectangle (axis cs:31,0.19121684373771108);

\addplot[mark=none, red, dashed, thick, samples=2, domain=-1:31] {0.17344377200890654};

\addlegendentry{Calibrated Personal Attack Classifier}

\draw [red, fill=red, opacity=0.1] (axis cs:-1,0.1071683900710867) rectangle (axis cs:31,0.2114280759331456);

\addplot[mark=*, mark options={scale=.8}, black, thick]
    table[]{
        4.24481358024944	0.13605808090085353
    };

\addplot[mark=none, black, dashed, thick, samples=2, domain=-1:31] coordinates {(4.24481358024944,0.13605808090085353) (4.24481358024944,-0.28233627462045101)};

\draw [black, fill=black, opacity=0.1] (axis cs:2.7990185911719845,1) rectangle (axis cs:6.14718915337283,-0.2);
\addplot[mark=*, mark options={scale=.8}, red, thick]
    table[]{
        7.630604332572648	0.17344377200890654
    };

\addplot[mark=none, red, dashed, thick, samples=2, domain=-1:31] coordinates {(7.630604332572648,0.17344377200890654) (7.630604332572648,-0.28233627462045101)};

\draw [red, fill=red, opacity=0.1] (axis cs:5.218138186868687,1) rectangle (axis cs:14.53696886605178,-0.2);

\end{axis}
\end{tikzpicture}
     \begin{tikzpicture}
\sffamily
\begin{axis}[
title = {Frequency Combiner + Cross Entropy},
title style={align=center,yshift=-.1in},
legend style={font=\small,
	nodes={scale=0.7, transform shape},
	at={(0.0,1)},
	anchor=north west,
	draw=none, fill=none},
legend cell align={left},
width = 2.9in, height = 2.0in,
ylabel near ticks,
ylabel = {\small Info Gain ($c_k - c_0$)},
xlabel near ticks,
every tick label/.append style={font=\scriptsize},
xmin=-1,xmax=20,ymin=-0.2,ymax=0.5,
xtick={0, 3, 6, 9.488535342452668, 12, 16.429822968722306, 19},
xlabel={\small Number of raters},
xlabel style = {yshift=0.05in},
yticklabel style={
		/pgf/number format/fixed,
		/pgf/number format/precision=5
},
scaled y ticks=false
]

\addplot[solid, mark=o, mark options={scale=.8}, black]
plot [error bars/.cd, y dir = both, y explicit]
table[y error minus index=2, y error plus index=3]{
0	0.001079029053995595	0.12235642133600799	0.08495701502200248
1	-0.05012933564167388	0.06725447803303009	0.0473831294493644
2	-0.030445423282303263	0.07911900016520201	0.05370657773919485
3	-0.00011208734824896815	0.09150527188039292	0.06158074175964545
4	0.023283612487382088	0.09775850951832188	0.0635301594574893
5	0.05135721792250991	0.10452087505167718	0.06837955871289025
6	0.05730269142865929	0.10362987366207554	0.06759671287102809
7	0.08257455275125025	0.11286727014481024	0.07141911581397398
8	0.0975970010073135	0.11489008093072006	0.07332267553651989
9	0.1073892359345412	0.11596059622581412	0.07314177475660943
10	0.12120953940483814	0.11766167712183118	0.07704827690115801
11	0.13073674633966081	0.11465710586472855	0.07410100897420002
12	0.1444538268526806	0.11649173794276424	0.07445236427401325
13	0.14459934884890924	0.11872252587124166	0.07583048874293796
14	0.15799465368231835	0.11864700973427078	0.07650699280320983
15	0.1641987666353859	0.11848844923224433	0.07559939366243196
16	0.16786904914961975	0.12098912246641735	0.07785072212240007
17	0.17560799135677435	0.12172935697092913	0.07606244270930929
18	0.17790998082754522	0.11703931379474491	0.0761289531595058
19	0.18584106449550059	0.11744874815541895	0.07629895138933435
};
\addlegendentry{$k$ raters}

\addplot[mark=none, black, dashed, thick, samples=2, domain=-1:31] {0.11414094262320251};
\addlegendentry{Jigsaw Personal Attack Classifier}

\draw [black, fill=black, opacity=0.1] (axis cs:-1,-0.11800307885557604) rectangle (axis cs:31,0.2281284405487617);

\addplot[mark=none, red, dashed, thick, samples=2, domain=-1:31] {0.1711954242638693};
\addlegendentry{Calibrated Personal Attack Classifier}

\draw [red, fill=red, opacity=0.1] (axis cs:-1,0.015349812034862076) rectangle (axis cs:31,0.24927654364317753);

\addplot[mark=*, mark options={scale=.8}, black, thick]
    table[]{
        9.488535342452668	0.11414094262320251
    };

\addplot[mark=none, black, dashed, thick, samples=2, domain=-1:31] coordinates {(9.488535342452668,0.11414094262320251) (9.488535342452668,-0.2)};

\draw [black, fill=black, opacity=0.1] (axis cs:0.9855370160596577,1) rectangle (axis cs:16.16206681625036,-1);
\addplot[mark=*, mark options={scale=.8}, red, thick]
    table[]{
        16.429822968722306	0.1711954242638693
    };

\addplot[mark=none, red, dashed, thick, samples=2, domain=-1:31] coordinates {(16.429822968722306,0.1711954242638693) (16.429822968722306,-0.2)};

\draw [red, fill=red, opacity=0.1] (axis cs:12.64232261151947,1) rectangle (axis cs:19.0,-1);

\end{axis}
\end{tikzpicture}
     \begin{tikzpicture}
\sffamily
\begin{axis}[
title = {ABC + AUC},
title style={align=center,yshift=-.1in},
legend style={font=\small,
	nodes={scale=0.7, transform shape},
	at={(0.0,1)},
	anchor=north west,
	draw=none,
	fill=none},
legend cell align={left},
width = 3.0in, height = 2.0in,
ylabel near ticks,
ylabel = {\small AUC score},
xlabel near ticks,
every tick label/.append style={font=\scriptsize},
xmin=-1,xmax=20,ymin=0.4,ymax=1.0,
xtick={0, 3.765792358444934, 
6, 9, 12, 15, 18},
xlabel={\small Number of raters},
xlabel style = {yshift=0.05in},
yticklabel style={
		/pgf/number format/fixed,
		/pgf/number format/precision=5
},
scaled y ticks=false
]

\addplot[solid, mark=o, mark options={scale=.8}, black]
plot [error bars/.cd, y dir = both, y explicit]
table[y error minus index=2, y error plus index=3]{
0	0.5	0.0	0.0
1	0.608538290188263	0.013045014803354293	0.012779028248469992
2	0.6717120617764561	0.017829238503799583	0.017219901971411766
3	0.7051856527517107	0.01747519515474283	0.017989069472778674
4	0.7279262421182479	0.0163832247002641	0.017222304102014308
5	0.7320799952186493	0.01551480180424003	0.015455494966162164
6	0.7442021825310284	0.016693954556551094	0.0177671780370543
7	0.7444969895051412	0.01649464605964579	0.017453707767692617
8	0.7470975612070602	0.017566563229967636	0.019457144128046888
9	0.7556134450889516	0.018187807612442097	0.019811294156642734
10	0.7617765119741529	0.020206084298294402	0.022010512052541342
11	0.7659934266371806	0.020108283031531227	0.022582859933816368
12	0.767461695776738	0.020525714276187323	0.022301151650847006
13	0.772363186216702	0.022961851572468994	0.02253283645041304
14	0.7734227118281832	0.02326434910321562	0.02400812148610032
15	0.7754563375349438	0.023979962070387772	0.022981962332748296
16	0.7737622781795761	0.02376817616364435	0.023885428939405373
17	0.7749373961141801	0.024679020836417598	0.024235405856575465
18	0.7724008518552778	0.025602858580582644	0.024473966825947824
19	0.7749734742801917	0.029031072897587884	0.026321214655929515
};
\addlegendentry{$k$ raters}

\addplot[mark=none, black, dashed, thick, samples=2, domain=-1:31] {0.722600222315139};
\addlegendentry{Jigsaw Personal Attack Classifier}

\draw [black, fill=black, opacity=0.1] (axis cs:-1,0.669997143634911) rectangle (axis cs:31,0.7775795168153095);

\addplot[mark=none, red, dashed, thick, samples=2, domain=-1:31] {0.7267259283186279};
\addlegendentry{Calibrated Personal Attack Classifier}

\draw [red, fill=red, opacity=0.1] (axis cs:-1,0.6791907030216355) rectangle (axis cs:31,0.7816359379280209);

\addplot[mark=*, mark options={scale=.8}, black, thick]
    table[]{
        3.765792358444934	0.722600222315139
    };

\addplot[mark=none, black, dashed, thick, samples=2, domain=-1:31] coordinates {(3.765792358444934,0.722600222315139) (3.765792358444934,0.0)};

\draw [black, fill=black, opacity=0.1] (axis cs:2.008259474960942,1) rectangle (axis cs:19.0,0);

\addplot[mark=*, mark options={scale=.8}, red, thick]
    table[]{
        3.947217120001903	0.7267259283186279
    };

\addplot[mark=none, red, dashed, thick, samples=2, domain=-1:31] coordinates {(3.947217120001903,0.7267259283186279) (3.947217120001903,0.0)};

\draw [red, fill=red, opacity=0.1]  (axis cs:2.238112645176561,1) rectangle (axis cs:19.0,0);

\end{axis}
\end{tikzpicture}
     \begin{tikzpicture}
\sffamily
\begin{axis}[
title = {Frequency Combiner + AUC},
title style={align=center,yshift=-.1in},
legend style={font=\small,
	nodes={scale=0.7, transform shape},
	at={(0.0,1)},
	anchor=north west,
	draw=none},
legend cell align={left},
width = 3.0in, height = 2.0in,
ylabel near ticks,
ylabel = {\small AUC score},
xlabel near ticks,
every tick label/.append style={font=\scriptsize},
xmin=-1,xmax=20,ymin=0.4,ymax=1.0,
xtick={0, 3, 6, 9, 12, 15, 19.0},
xlabel={\small Number of raters},
xlabel style = {yshift=0.05in},
yticklabel style={
		/pgf/number format/fixed,
		/pgf/number format/precision=5
},
scaled y ticks=false
]

\addplot[solid, mark=o, mark options={scale=.8}, black]
plot [error bars/.cd, y dir = both, y explicit]
table[y error minus index=2, y error plus index=3]{
0	0.5	0.0	0.0
1	0.5147631588848356	0.004783127782616137	0.004917835638169832
2	0.5285107900804795	0.008654306101756215	0.009531520627389711
3	0.5461023382117635	0.017349700910805232	0.017720470657226928
4	0.5558145938581375	0.022168501747761638	0.023367320622532772
5	0.5738147322873032	0.027114867074992	0.028346545929523903
6	0.5777588241668529	0.0311867137336459	0.03213401353984291
7	0.5961150218575055	0.034184939809576265	0.0371667849121734
8	0.6075756935785265	0.03575553740947757	0.03724612348017742
9	0.620687767973551	0.03866049922229986	0.04128351428297883
10	0.6303022677865165	0.04045480035557303	0.04181093730511409
11	0.6362483202541023	0.04035256648891383	0.04154602738857227
12	0.642319708326077	0.042706615928502	0.04311002385158058
13	0.6505291054570168	0.043398597635662894	0.04368017308433192
14	0.654710196208361	0.04389495454327286	0.043748515090273554
15	0.6576129563769623	0.04413694348244346	0.0435675192359557
16	0.6680779860640713	0.04699399725143805	0.04511032542161997
17	0.6689358148444348	0.045439323853909475	0.043954037286727976
18	0.6695187997761471	0.04464962499499303	0.043327221424624196
19	0.6746435724825622	0.04614750994176653	0.04394419797876514
};
\addlegendentry{$k$ raters}

\addplot[mark=none, black, dashed, thick, samples=2, domain=-1:31] {0.7420979483689865};
\addlegendentry{Jigsaw Peresonal Attack Classifier}

\draw [black, fill=black, opacity=0.1] (axis cs:-1,0.6925680582772198) rectangle (axis cs:31,0.7785248072813234);

\addplot[mark=none, red, dashed, thick, samples=2, domain=-1:31] {0.7438469566587296};
\addlegendentry{Calibrated Personal Attack Classifier}

\draw [red, fill=red, opacity=0.1] (axis cs:-1,0.6945474523616997) rectangle (axis cs:31,0.7780547730537732);

\addplot[mark=*, mark options={scale=.8}, black, thick]
    table[]{
        19.0	0.7420979483689865
    };

\addplot[mark=none, black, dashed, thick, samples=2, domain=-1:31] coordinates {(19,0.7356822453082545) (19.0,0)};

\draw [black, fill=black, opacity=0.1] (axis cs:18.9,-1) rectangle (axis cs:19.1,1);
\addplot[mark=*, mark options={scale=.8}, red, thick]
    table[]{
        19	0.7438469566587296
    };

\addplot[mark=none, red, dashed, thick, samples=2, domain=-1:31] coordinates {(19,0.7420979483689865) (19,0)};

\draw [red, fill=red, opacity=0.1] (axis cs:18.9,-1) rectangle (axis cs:19.1,1);

\end{axis}
\end{tikzpicture}
     \caption{Rater equivalence between human labels and Jigsaw's Wikipedia comment personal attack classifier under different combiner and scoring function pairings. Error bars cover 95\% of 500 bootstrap item samples. }
     \label{fig:Wiki}
\end{figure}
\subsection{Personal Attacks}

Jigsaw and the Wikimedia Foundation collected annotations for 23,179 Wikipedia comments, each labeled by 10 to 20 raters for the presence of ``personal attack or harassment''~\citep{Wulczyn2017machina}. A classifier also produced a predicted probability of an attack label. We evaluate predictions using single-rater binary labels, applying scoring functions suitable for soft classifiers and discrete outcomes.

Figure~\ref{fig:Wiki} reports power curves and rater equivalence scores across two scoring functions (cross-entropy and AUC) and two combiners (ABC and frequency). Note that since ABC approximates the calibrated combiner, the cross-entropy scores for benchmark panels are higher with ABC (top left figure) than with the frequency combiner (top right figure).

For the classifier scores, we include both raw and isotonic-calibrated outputs.\footnote{Calibration used \texttt{CalibratedClassifierCV} from \texttt{sklearn}.} Calibration improves the classifier's cross-entropy score but has little effect on AUC.

Our preferred estimate, ABC with cross-entropy, yields a rater equivalence of 7.63 for the calibrated Jigsaw classifier and 4.24 for the uncalibrated Jigsaw classifier. Confidence intervals are wide because the classifier score lies on a flat portion of the power curve, though both classifier and power curve scores individually have narrow error bars.

The original paper presented a ``Human Baseline Comparison,'' a form of rater equivalence. They used a 10-rater majority vote as ground truth and either AUC or Spearman correlation as scoring metrics, along with the frequency and majority vote combiners. They reported that their model outperformed groups of three raters but underperformed groups of five---below our estimate of 7.63.

Differences stem from two sources. First, we score against a single randomly selected rater, which we argue better reflects subjective utility. Second, while their evaluation used a stratified sample enriched for blocked users, we analyzed a uniform random sample of 2,000 comments from the released dataset.

\begin{figure}[t]
     \centering
     \begin{tikzpicture}
\sffamily
\begin{axis}[
title = {ABC + Cross Entropy},
title style={align=center,yshift=-.1in},
legend style={font=\small,
	nodes={scale=0.7, transform shape},
	at={(1,1)},
	anchor=north east,
	draw=none, fill=none},
legend cell align={left},
width = 3.0in, height = 2.0in,
ylabel near ticks,
ylabel = {\small score},
xlabel near ticks,
every tick label/.append style={font=\scriptsize},
xmin=-1,xmax=20,ymin=-0.024726247920585966,ymax=0.0811545676300434,
xtick={0.4792188849549335, 3, 6, 9, 12, 15, 18},
xlabel={\small Number of raters},
xlabel style = {yshift=0.05in},
yticklabel style={
		/pgf/number format/fixed,
		/pgf/number format/precision=5
},
scaled y ticks=false
]

\addplot[solid, mark=o, mark options={scale=.8}, black]
plot [error bars/.cd, y dir = both, y explicit]
table[y error minus index=2, y error plus index=3]{
0	0.0	0.014414745851065391	0.013083300760035721
1	0.00478195181583585	0.013766446258585585	0.012869640402801918
2	0.008655639936390291	0.013292789077305245	0.012757180975514482
3	0.012405271731333656	0.013146222984928513	0.012698756683812928
4	0.01542167057124988	0.013140689750906653	0.012524584455967935
5	0.017901608428898208	0.012915212752321659	0.012683824855584391
6	0.02019363469459534	0.012793627807315189	0.012384147842770865
7	0.02234431884397059	0.012761143815558929	0.012498242857697517
8	0.023902777562520905	0.012586731061017487	0.012301999932207486
9	0.02559298029388657	0.012678796899583955	0.012306327649181714
10	0.026799335363795307	0.012592894337210225	0.012233259014667097
11	0.0283723913055397	0.012584669031107665	0.012298904564477797
12	0.029543739928732382	0.01249269189813973	0.012379253690605885
13	0.030743654991246094	0.012561911352821364	0.012318922126324683
14	0.03197818618175097	0.01240043052643891	0.012314333117889209
15	0.03244410985938473	0.012706457218075506	0.012277780470807631
16	0.033699687781902576	0.012542169594997321	0.012467902091836458
17	0.03429475754742517	0.0125467743092178	0.01195351561483815
18	0.0350506608137412	0.012510095309786529	0.012177460565237164
19	0.036311974083577225	0.012453325843877372	0.012348056820078113
};
\addlegendentry{$k$ raters}

\addplot[mark=none, black, dashed, thick, samples=2, domain=-1:31] {0.0022916016170930753};
\addlegendentry{CredWeb Classifier}

\draw [black, fill=black, opacity=0.1] (axis cs:-1,-0.012873933487272282) rectangle (axis cs:31,0.014766970458949924);

\addplot[mark=*, mark options={scale=.8}, black, thick]
    table[]{
        0.4792188849549335	0.0022916016170930753
    };

\addplot[mark=none, black, dashed, thick, samples=2, domain=-1:31] coordinates {(0.4792188849549335,0.0022916016170930753) (0.4792188849549335,-0.2)};

\draw [black, fill=black, opacity=0.1] (axis cs:0.18764968586431346,1) rectangle (axis cs:0.7815511971008702,-1);

\end{axis}
\end{tikzpicture}
     \caption{Rater equivalence between human labels of news credibility and CredBank's heuristic classifier for ABC and cross-entropy scorer. Error bars cover 95\% of 500 bootstrap samples.}
     \label{fig:credweb}
\end{figure}

\subsection{News Credibility}

The CredBank dataset contains 1,377 news events, each annotated by 30 crowd workers who viewed tweets related to the event~\citep{mitra2015credbank}. Labels were collected on a five-point credibility scale and binarized for analysis: ``certainly accurate'' versus everything else. A linguistic classifier was trained to predict the proportion of raters who would label an event as ``certainly accurate,'' using four discrete prediction buckets: $>90\%$, $80$–$90\%$, $60$–$80\%$, and $<60\%$~\citep{mitra2017parsimonious}. While the original paper reported precision and recall of 50–75\%, the practical informativeness of the classifier remains unclear.

We compute its rater equivalence using the Anonymous Bayesian Combiner and cross-entropy. To calibrate the classifier, we mapped each of its four output buckets to empirical probabilities: $83.1\%$, $80.0\%$, $77.6\%$, and $77.7\%$, respectively. The overall base rate of ``certainly accurate'' labels was $79.7\%$, yielding an information gain of just $0.0023$ bits.

Figure~\ref{fig:credweb} shows the resulting power curve. The classifier’s rater equivalence is 0.48, meaning its performance is much worse than a single human rater. The equivalent human process would acquire a human rater's label on fewer than half the items (48\%). For those, the human process would apply Bayesian inference (\ie, predicting $81.4\%$ if the label is ``certainly accurate,'' $73.1\%$ if not). For the remaining items where no labels are collected, the human process would output the default base rate prediction of $79.7\%$. 

\begin{figure}[t]
     \centering
     \begin{tikzpicture}
\sffamily
\begin{axis}[
title = {ABC + Cross Entropy},
title style={align=center,yshift=-.1in},
legend style={font=\small,
	nodes={scale=0.7, transform shape},
	at={(1,1)},
	anchor=north east,
	draw=none, fill=none},
legend cell align={left},
width = 3.0in, height = 2.0in,
ylabel near ticks,
ylabel style = {yshift=-0.15in, align=center},
ylabel = {\small Info Gain ($c_k - c_0$)},
xlabel near ticks,
every tick label/.append style={font=\scriptsize},
xmin=-1,xmax=20,ymin=-0.06,ymax=0.1,
xtick={0, 2.3873774335172957, 6, 9, 12, 15, 18},
xlabel={\small Number of raters},
xlabel style = {yshift=0.05in},
yticklabel style={
		/pgf/number format/fixed,
		/pgf/number format/precision=5
},
scaled y ticks=false
]

\addplot[solid, mark=o, mark options={scale=.8}, black]
plot [error bars/.cd, y dir = both, y explicit]
table[y error minus index=2, y error plus index=3]{
0	0.0	0.0011378749810377897	0.0010781294186446022
1	0.0032834945791385683	0.0024546375796217212	0.00265600031761537
2	0.006714043745178144	0.003848878912517395	0.004474376906080324
3	0.009930098581262148	0.004742657997470867	0.005929434727149796
4	0.013225108606867098	0.005684406637073369	0.007038127383650727
5	0.015619809356476888	0.00625985346139224	0.007855904992378426
6	0.018119915739087222	0.006822264990619553	0.008342797845187988
7	0.019562116285638642	0.00707639406255911	0.008708792654585684
8	0.022026236693761803	0.007364015922582623	0.00944364437792633
9	0.02422099960942181	0.00780703070798483	0.0095890113738889
10	0.026258166293613994	0.008232057320598396	0.010177671879952777
11	0.02882180300772219	0.008470327138768607	0.010671320174686083
12	0.03038616280424067	0.008583456321584748	0.010563815572551816
13	0.03112734881280199	0.00909367835505892	0.010819749149304059
14	0.03337906471295615	0.009122907173771289	0.010994910929458368
15	0.03416882835563362	0.009363844562618073	0.011210460791887589
16	0.03514805408415389	0.00925100133948531	0.011225479317669262
17	0.0367746381230416	0.00934450326512326	0.011258661334347497
18	0.038011319812587985	0.009672552275190371	0.011516072041572567
19	0.03913558550478302	0.009427247551637152	0.011651258295654165
};
\addlegendentry{$k$ raters}

\addplot[mark=none, black, dashed, thick, samples=2, domain=-1:31] {0.007959870813631253};
\addlegendentry{Reddit Scores}
\draw [black, fill=black, opacity=0.1] (axis cs:-1,0.0003618121865849311) rectangle (axis cs:31,0.01584246728988936);

\addplot[mark=*, mark options={scale=.8}, black, thick]
    table[]{
        2.3873774335172957	0.007959870813631253
    };

\addplot[mark=none, black, dashed, thick, samples=2, domain=-1:31] coordinates {(2.3873774335172957,0.007959870813631253) (2.3873774335172957,-0.12012248614933219)};

\draw [black, fill=black, opacity=0.1] (axis cs:0.2665861473579115,1) rectangle (axis cs:5.26157480883952,-1);

\end{axis}
\end{tikzpicture}
     \begin{tikzpicture}
\sffamily
\begin{axis}[
title = {Frequency + Cross Entropy},
title style={align=center,yshift=-.1in},
legend style={font=\small,
	nodes={scale=0.7, transform shape},
	at={(1,1)},
	anchor=north east,
	draw=none, fill=none},
legend cell align={left},
width = 3.0in, height = 2.0in,
ylabel near ticks,
ylabel style = {yshift=-0.15in, align=center},
ylabel = {\small Info Gain ($c_k - c_0$)},
xlabel near ticks,
every tick label/.append style={font=\scriptsize},
xmin=-1,xmax=20,ymin=-0.06,ymax=0.1,
xtick={0, 3, 6, 9, 11.365866533738242, 15, 18},
xlabel={\small Number of raters},
xlabel style = {yshift=0.05in},
yticklabel style={
		/pgf/number format/fixed,
		/pgf/number format/precision=5
},
scaled y ticks=false
]

\addplot[solid, mark=o, mark options={scale=.8}, black]
plot [error bars/.cd, y dir = both, y explicit]
table[y error minus index=2, y error plus index=3]{
0	-9.819527018950502e-06	0.000339841951431441	0.00032060664389799154
1	-0.03175281702882449	0.00642685658508646	0.006818659740057242
2	-0.037977780834982044	0.0101164912367524	0.010241696418615165
3	-0.0347940240272685	0.011382459289266178	0.01199106881260481
4	-0.029357894619143843	0.012238668254737517	0.013040929200980989
5	-0.022683517688153576	0.01298727658453025	0.01287393946442239
6	-0.017374904566763827	0.013220040803760602	0.013553434324547142
7	-0.011569537016097198	0.012723545072424125	0.01342052625794421
8	-0.005877602798506243	0.012772041833066572	0.014007594827009084
9	2.782868716599829e-05	0.013617656138379175	0.013988114733284096
10	0.0033862260960156876	0.013515191875336008	0.014171629546730014
11	0.007732863714201232	0.01335413077958647	0.013946501052180094
12	0.010478111693573378	0.013434225186780169	0.01442736926851429
13	0.01464359427097306	0.013339678295775204	0.013820745696196401
14	0.016984797357905257	0.013265743149304732	0.01390051377496715
15	0.018976670340946522	0.013443043615486538	0.01405313521231999
16	0.022006925140082734	0.013531151093639449	0.013849176206227387
17	0.025120900542875746	0.013594898478743067	0.014334629449352931
18	0.026088521624558436	0.013275209476231065	0.01401085873214547
19	0.029351585240185663	0.013812850773703866	0.01427821903605797
};
\addlegendentry{$k$ raters}

\addplot[mark=none, black, dashed, thick, samples=2, domain=-1:31] {0.008737258076666032};
\addlegendentry{Reddit Scores}

\draw [black, fill=black, opacity=0.1] (axis cs:-1,0.0011762541709926522) rectangle (axis cs:31,0.017584647304972223);

\addplot[mark=*, mark options={scale=.8}, black, thick]
    table[]{
        11.365866533738242	0.008737258076666032
    };

\addplot[mark=none, black, dashed, thick, samples=2, domain=-1:31] coordinates {(11.365866533738242,0.008737258076666032) (11.365866533738242,-0.1)};

\draw [black, fill=black, opacity=0.1] (axis cs:7.781807801645638,1) rectangle (axis cs:15.938075094696948,-2);

\end{axis}
\end{tikzpicture}
     \begin{tikzpicture}
\sffamily
\begin{axis}[
title = {Majority Vote Combiner + F1},
title style={align=center,yshift=-.1in},
legend style={font=\small,
	nodes={scale=0.7, transform shape},
	at={(1,1)},
	anchor=north east,
	draw=none, fill=none},
legend cell align={left},
width = 3.0in, height = 2.0in,
ylabel near ticks,
ylabel = {\small F1 Score},
xlabel near ticks,
every tick label/.append style={font=\scriptsize},
xmin=-1,xmax=20,ymin=0.2,ymax=0.8,
xtick={0, 3.8741621049495345, 6, 9, 12, 15, 18},
xlabel={\small Number of raters},
xlabel style = {yshift=0.05in},
yticklabel style={
		/pgf/number format/fixed,
		/pgf/number format/precision=5
},
scaled y ticks=false
]

\addplot[solid, mark=o, mark options={scale=.8}, black]
plot [error bars/.cd, y dir = both, y explicit]
table[y error minus index=2, y error plus index=3]{
0	0.4996115070407326	0.02006285129814811	0.019472842948632163
1	0.5338277622492633	0.005811071856594463	0.006209527044766294
2	0.541421446443775	0.00694190168500175	0.007535255232121618
3	0.5521111956610298	0.008313362241741729	0.009271700277606154
4	0.5574558251163481	0.008582717637161008	0.009962374386849882
5	0.5640521338019033	0.009619121925572793	0.01066226160865369
6	0.5685315290847643	0.010101499454346241	0.011590772104749014
7	0.5729004678875935	0.010841256646735298	0.01201005172616565
8	0.5753836750390733	0.010686635888919604	0.011826681074558754
9	0.578362779886669	0.011030455726408817	0.012475490944165002
10	0.5800925910192336	0.010974134323958995	0.012685246036935549
11	0.5824137297392906	0.011623788647032285	0.013167491551731025
12	0.5851951134107203	0.011777761353774285	0.01379285518344553
13	0.5864302154836948	0.011841053869194096	0.013572109041403269
14	0.5888864537324455	0.011994290660183093	0.013811093748082204
15	0.5894226116918378	0.0121815059848851	0.014261599011926474
16	0.5908008363179222	0.012314529329012003	0.013972706050789707
17	0.5920824769745053	0.01231567882204243	0.01416496088567687
18	0.593685746859177	0.012440042459925027	0.014343310379264729
19	0.5943024422861276	0.012788437446256351	0.01449049893984855

};
\addlegendentry{$k$ raters}

\addplot[mark=none, black, dashed, thick, samples=2, domain=-1:31] {0.5567832681958661};
\addlegendentry{Reddit Scores}

\draw [black, fill=black, opacity=0.1] (axis cs:-1,0.5377119729408435) rectangle (axis cs:31,0.5743417674354667);

\addplot[mark=*, mark options={scale=.8}, black, thick]
    table[]{
        3.8741621049495345	0.5567832681958661
    };

\addplot[mark=none, black, dashed, thick, samples=2, domain=-1:31] coordinates {(3.8741621049495345,0.5567832681958661) (3.8741621049495345,0.0)};

\draw [black, fill=black, opacity=0.1] (axis cs:1.9362086311352926,1) rectangle (axis cs:7.47919409300499,0);

\end{axis}
\end{tikzpicture}
     \begin{tikzpicture}
\sffamily
\begin{axis}[
title = {Majority Vote Combiner + Agreement},
title style={align=center,yshift=-.1in},
legend style={font=\small,
	nodes={scale=0.7, transform shape},
	at={(1,1)},
	anchor=north east,
	draw=none, fill=none},
legend cell align={left},
width = 3.0in, height = 2.0in,
ylabel near ticks,
ylabel = {\small Agreement Score},
xlabel near ticks,
every tick label/.append style={font=\scriptsize},
xmin=-1,xmax=20,ymin=0.2,ymax=0.8,
xtick={0, 3, 4.573540675784381, 9, 12, 15, 18},
xlabel={\small Number of raters},
xlabel style = {yshift=0.05in},
yticklabel style={
		/pgf/number format/fixed,
		/pgf/number format/precision=5
},
scaled y ticks=false
]

\addplot[solid, mark=o, mark options={scale=.8}, black]
plot [error bars/.cd, y dir = both, y explicit]
table[y error minus index=2, y error plus index=3]{
0	0.30617647058823527	0.01670779030174907	0.015380444992368347
1	0.33471988795518204	0.012683976235204342	0.011400932728381163
2	0.34042500000000003	0.013394487981867709	0.012068087319336951
3	0.3486661585365854	0.013820046948785247	0.012704343295116993
4	0.3527922839506175	0.0138495832409255	0.012896867376358478
5	0.3558723437500001	0.013865286146457101	0.01306916697854299
6	0.3592371835443036	0.013979245780590566	0.013198681434599402
7	0.36146330128205123	0.014413057858641554	0.013502246628537828
8	0.36365211038961026	0.014683888068019735	0.013845495048863399
9	0.3660832236842104	0.014489028850194152	0.013772319834016444
10	0.3676614999999998	0.014739586327345788	0.013769997005987744
11	0.36868023648648646	0.014800123570426982	0.013945991294438187
12	0.37041712328767146	0.014762066278401764	0.01375942344762554
13	0.37145173611111104	0.014869544591373018	0.013961705408627179
14	0.37271443661971837	0.015054370132973727	0.014306281275476729
15	0.3733317857142859	0.015105823353293357	0.014262123075277966
16	0.3749838768115941	0.014946076506407058	0.014354376392143597
17	0.3760266544117645	0.015311626196137185	0.014252528215627613
18	0.3754246268656714	0.015202563343462427	0.014225981432657242
19	0.377605303030303	0.015173313222040519	0.014269111020383807

};
\addlegendentry{$k$ raters}

\addplot[mark=none, black, dashed, thick, samples=2, domain=-1:31] {0.3545588235294117};
\addlegendentry{Reddit Scores}

\draw [black, fill=black, opacity=0.1] (axis cs:-1,0.33875) rectangle (axis cs:31,0.36930882352941174);

\addplot[mark=*, mark options={scale=.1}, black, thick]
    table[]{
        4.573540675784381	0.3545588235294117
    };

\addplot[mark=none, black, dashed, thick, samples=2, domain=-1:31] coordinates {(4.573540675784381,0.3545588235294117) (4.573540675784381,0.0)};

\draw [black, fill=black, opacity=0.1] (axis cs:2.577510123964533,1) rectangle (axis cs:8.097004672264381,0);

\end{axis}
\end{tikzpicture}
     \caption{Rater equivalence between GuessTheKarma survey responses and Reddit scores under different combiner and scoring function pairings. Error bars cover 95\% of 500 bootstrap item samples.}
     \label{fig:gtk}
\end{figure}
\subsection{Social Rating Systems}
\label{sec:gtk}

How informative are Reddit’s net upvotes? \citet{glenski2018guessthekarma} collected images from a subreddit, then presented pairs of them in a game-like setting, asking users which they preferred. They concluded that net upvotes reliably indicated user preference only when score differences were large.

We reanalyze their dataset using rater equivalence. We treat Reddit’s ranking mechanism as a hard classifier that predicts the higher-voted image in each pair. A calibrated version outputs $57.8\%$ if the left image had more net votes, and $42.8\%$ otherwise. Figure~\ref{fig:gtk} shows power curves and rater equivalence values.

Using the Anonymous Bayesian Combiner and cross-entropy, the classifier’s rater equivalence is 2.39. That is, the Reddit net upvotes for the two images convey less information about human preference between the two images than a poll of three people. This is especially striking given that most images received hundreds or thousands of votes. 

Interestingly, the frequency combiner performs worse than random guessing when only a few raters are sampled. Its overconfident predictions, \eg, assigning near-certainty after one vote—are heavily penalized under cross-entropy. It reaches Reddit equivalence at 11.41 raters.

Figure~\ref{fig:gtk} also shows that the majority combiner yields consistent curve shapes for both F1 and agreement metrics, with similar equivalence values.

\section{Discussion}
\label{sec:discussion}
This paper introduced a unified framework for comparing machine classifiers to human judgment by measuring how many independent human raters would be needed to match a model’s performance: its \textbf{rater equivalence}. Central to this framework is the distinction between two conceptual roles played by human raters: \textbf{evaluation panels} and \textbf{benchmark panels}. We use this distinction to guide a broader reflection on practical deployment, model evaluation, and future directions.

\subsection{Evaluation Panels}
\textit{Evaluation panels} define the standard against which performance is measured. In the \textit{subjective utility model}, the evaluation panel reflects human preferences. Here, a single randomly selected rater yields an unbiased estimate of expected utility (Claim~\ref{claim:subjectivejudgmenet}), and increasing the panel size may actually introduce bias (Claim~\ref{claim:subjnonunbias}). Because the single-rater score is an unbiased estimate of subjective utility, absolute performance thresholds are directly interpretable under this model.

If the managerial question is whether to engage in the classification task at all, then the absolute score can be compared to a threshold of acceptable performance. For example, in the content moderation example, a company might decide not to engage in content moderation at all if the best available process does not achieve at least 80\% agreement with human labels. In this scenario, if human raters can not agree sufficiently on what content should be moderated, then the manager would conclude that it is impossible to have sufficient public legitimacy for content moderation. They might then structure the product to not use content moderation, as many newspapers did when they turned off the public comments feature on online news stories.

Our subjective utility model assumes that the reported label reflects the rater's true internal judgment. In reality, reporting noise may be present; for example, a rater may occasionally not examine the image carefully or click on the wrong button by accident. 
If the noise in reporting yields a distribution of reports that diverges from the distribution of underlying judgments, scoring against an individual report will no longer be an unbiased estimator of the expected score, which defines the subjective utility.
Future work could try to characterize the conditions under which biased reporting by the evaluation does or doesn't affect the relative ordering of classifier scores.
Only changes in relative ordering would affect the rater equivalence.

Our formal model assumes a discrete label space. Our results on panel sizes consider only the majority vote as the combiner function. It is possible, however, to treat labels on a Likert scale (1-5) or letter grades (A-F) as an interval scale and use the mean as the combiner. The software library we provide includes mean as a combiner function and correlation as a scoring function, and it has been used, for example, in computing rater equivalence for Likert scale answers \citep{doi:10.1177/26339137231173407}.
With scoring functions that are linear in the evaluation label, such as percent agreement or mean squared error, scoring against the mean of a panel yields the same score as the mean of scores against each individual rater, so the choice of panel size becomes irrelevant. With non-linear scoring functions, such as correlation or cross-entropy, however, the choice of panel size can affect the score and thus the rater equivalence. 

The core result — that scoring against a single rater provides an unbiased estimate of subjective utility — extends directly, since the proof relies only on linearity of expectation. Indeed, it extends even to truly continuous label space such as the real numbers from 0-100. However, the Anonymous Bayesian Combiner, which relies on counting discrete label sequences, would need adaptation for interval or continuous settings, perhaps through discretization or a parametric model of the response state distribution.

In contrast, in the \textit{objective utility model}, there exists a correct label for each item and panels of human evaluators serve as the best available proxy. But the reliability of such proxies depends on inter-rater reliability, panel size and the scoring function. As we show in Appendix~\ref{appdx:objective-utility}, scoring against a finite evaluation panel does not yield an unbiased estimate of the score against ground truth (Claim~\ref{claim:objnonunbiased}). That means there is no principled basis for a manager to set an absolute deployment threshold: a score of 80\% against a human panel may correspond to a substantially different score against the unknown ground truth.
With estimates of the correlation between human labels and ground truth, it would be possible to bound the error in the absolute score of a classifier against the majority vote of certain panel sizes. Without any ground truth labels, however, such estimates would have to be driven by assumptions rather than empirical evidence.

Moreover, bigger evaluation panels are not guaranteed to be more reliable than smaller ones (Claims~\ref{prop:acc} and~\ref{prop:acc2}). That means that there is no principled basis for a manager to set a relative deployment threshold either: a classifier that outperforms another on a given evaluation panel may not outperform it on the ground truth. From a theory standpoint, this is a somewhat dissatisfying finding; a valuable direction for future research would be to characterize conditions under which bigger evaluation panels are guaranteed to improve reliability. 

These negative results have a clear practical implication: when the objective utility model is appropriate, managers should make every effort to obtain ground truth evaluation labels, even if only for a sample of items. As we argued in the introduction, the need for ground truth under the objective utility model is rarely an insurmountable obstacle, since objective ground truth is typically the kind of thing that can eventually be observed.

Crucially, classifiers and benchmark panels must be evaluated against the same evaluation panel for a fair, apples-to-apples comparison. It may be tempting, for example, to compare the correlation of a classifier with the majority vote of all available raters to the correlation of one human with another human \citep{doi:10.1126/sciadv.abf4393}, but that comparison would be misleading.

\subsection{Human Benchmark Panels}

Our framework generalizes the classic comparison target of machine learning systems: \textit{human-level performance}. We include \textit{team-level performance}, where more than one human participates in the benchmark panel. We simulate benchmark panels by assuming independent raters whose labels are combined using a \textit{combiner function}. While simple combiners like frequency voting perform poorly at small sizes (see Figures~\ref{fig:Wiki} and~\ref{fig:gtk}), the \textit{Anonymous Bayesian Combiner} (ABC) was shown to be optimal under proper scoring rules (Theorem~\ref{thm:abs_min_ratereq}), squeezing the most information out of a limited panel.

In some cases, the benchmark panel represents a human process that is currently in use and might be replaced by an automated classifier. For example, a content moderation team might currently have each comment reviewed by a single human moderator, and the question is whether to replace the human moderators with an LLM-based classifier. 

In other cases, the benchmark panel represents a hypothetical comparison point, a thought-experiment baseline (\eg, a team of peer reviewers), or a construct we have intuitions about (\eg, ``how well would I do?''). For example, a company may not currently engage in content moderation because it would be too expensive to hire human moderators for all the items that would need to be reviewed. However, the manager may be comfortable in principle with the decisions that human panels would make. If the automated classifier can be shown to perform at least that well on a set of evaluation items, then the manager may be comfortable deploying content moderation using the classifier. 

The human benchmark panel framework is particularly timely given the growing use of LLMs as substitutes for human raters in annotation and classification tasks \citep{rathje2024gpt}. Current benchmarks for LLMs like MMLU, ARC, and GSM8K focus on objective correctness on multiple-choice tasks.  In tasks that require nuanced or context-dependent human judgment such as essay grading, content moderation, or peer review, disagreement among humans is common, and ``ground truth'' is ill-defined or unknowable. When an LLM is used to label sentiment, toxicity, or other subjective constructs, the natural question is whether its labels are as good as those of a human rater, or a benchmark panel of human raters. 

Rater equivalence answers this directly: it quantifies how many human raters the LLM is equivalent to, under a given scoring function. If an LLM achieves a rater equivalence of three on a content moderation task, a manager knows that deploying it replaces roughly the informational value of a three-person human panel. Moreover, such comparisons can be translated into meaningful numeraires such as the \textit{dollar cost of equivalent human labor} or the number of LLM tokens required for a given rater equivalence level.

This framing also surfaces an important concern: for many subjective tasks, different demographic or cultural subgroups may systematically differ in their judgments. An LLM used as a rater may not faithfully represent the judgments of all such subgroups.  \citet{wang2025largelanguagemodelsreplace} show that LLMs tend to misportray the perspectives of marginalized identity groups, producing responses that resemble out-group stereotypes rather than actual in-group judgments. Similarly, content moderation judgments may differ systematically between politically liberal and conservative raters, and thus rater equivalence computed against one subgroup may not generalize to the other. 

Computing rater equivalence separately for subgroups of interest, based on labels from raters from those subgroups, would reveal whether an LLM rater is systematically worse for some populations than others.
When labels from a target subgroup are scarce, or the available rater pool is unrepresentative of the target population, the per-annotator models trained by jury learning \citep{gordon2022jury} could in principle be repurposed to label evaluation items, with simulated juries of specified composition providing subgroup-specific or population-weighted scoring targets.

\subsection{Robustness and the Fixed Environment Fallacy}
A critical assumption in our analysis is that the distribution of items remains fixed regardless of which classifier or panel is deployed. In some domains, like cancer detection or interstellar object classification, this assumption is plausible. But in others, especially social systems, the environment may respond to the classifier.

For example, if a peer review system adopts an LLM with a rater equivalence of 5, authors may begin tailoring submissions to exploit its known heuristics. This ``fixed environment fallacy'' extends to both the item pool (\eg, strategic adaptation) and the rater pool (\eg, learning effects or demographic shifts), as well as interactions between them: a classifier that alters the prevalence of certain item types can trigger \emph{prevalence-induced concept change}, shifting how raters categorize borderline items \citep{levari2018prevalence}. Our framework, like most evaluation approaches, does not yet account for this adaptive feedback loop, highlighting an important limitation and opportunity for future work.

Our framework also assumes that each item has a fixed rater response state $s_i$ throughout the rating process. In some settings, such as a breaking news story where new information changes people's judgments about a previously rated item, this assumption may not hold. Extending the framework to handle non-stationary items is an interesting direction for future work.

\subsection{Beyond Performance Goals}
Fairness, accountability, and institutional goals may override mean performance metrics. Experts may be better equipped to detect distribution shifts or contextual subtleties. Institutional processes like peer review panels may serve functions beyond prediction, such as stewardship of disciplinary norms, attribution of responsibility, or training new scholars. Thus, even if a classifier has higher rater equivalence, its substitution may be undesirable for reasons beyond informational value.


\subsection{Summary and Takeaways}

We close with several takeaways:

\begin{itemize}

  \item Under the \textbf{subjective utility model}, which typically applies when ground truth labels are unavailable, score classifiers against a \textbf{single rater at a time} and average the scores. This yields an unbiased estimate of expected utility; combining labels into a majority vote does not. 

  \item Under the \textbf{objective utility model}, every effort should be made to obtain ground truth labels. If unavailable, conclusions drawn from absolute or relative rankings should be made with caution.

    \item \textbf{Evaluation panels} must be shared across all methods being compared and selected in alignment with either subjective or objective utility goals.

  \item \textbf{Benchmark panels} define the human alternative. Their size and composition should be chosen carefully based on the application domain.
    
    \item \textbf{Rater equivalence} provides an intuitive and interpretable measure of model utility when human raters are a practical or hypothetical alternative to the classifier being evaluated.
    
    \item The \textbf{Anonymous Bayesian Combiner} (ABC) is a practical and theoretically optimal method for simulating benchmark panels, particularly useful in settings with limited rater data.
    
    \item Be wary of the \textbf{fixed environment fallacy}. Deployment of classifiers may change the distribution of inputs and the context in which they operate.
    
    \item \textbf{Non-informational criteria}, such as fairness, robustness to distributional shifts, and institutional roles, may also guide decisions about model deployment.
\end{itemize}

Ultimately, our goal is to clarify the meaning of performance metrics that involve people as both providers of evaluation labels and as benchmarks for comparison. As machine learning systems increasingly engage with human-centered domains, our framework offers a more grounded and interpretable lens on model performance.

\bibliography{ref} 

\clearpage

%
%
%


\begin{appendix}

\renewcommand{\thefigure}{A\arabic{figure}}
\setcounter{figure}{0}

\renewcommand{\thetable}{A\arabic{table}}
\setcounter{table}{0}

\let\oldsection\section
\renewcommand{\section}[1]{\clearpage\oldsection{#1}}

\section{Notation}
\label{app:notation}

\begin{table}[h]
\centering
\caption{Notation Table}
\small{
\begin{tabular}{l|l}
\toprule
\textbf{Symbol} & \textbf{Meaning} \\
\midrule
$\itemFeaturesRV$, $\itemFeatures$ & Random variable (RV) and realization for the features of $\ith$ item\\ 
$\itemFeaturesVecRV$ = \{$X_1, X_2, \cdots$\} & IID random variables for features of all items \\
$\itemFeaturesVec$ = $\{x_1, x_2, \cdots\}$ & Realized features for all items \\
$\La$ & Set of possible labels \\
$\itemGTRV$, $\itemGT \in \La$  & RV and realization for ground truth label for $\ith$ item \\
$\Delta_\La$ & Set of all distributions over $\La$\\
$\itemResponseStateRV$, $\itemResponseState\in \Delta_\La$ & RV and realization for rater response state of $\ith$ item\\
$D_{\X, \GT, \St}$ & Joint distribution from which $X_i$, $\itemResponseStateRV$ and $\itemGTRV$ are drawn \\

$c$ & Classifier that maps an item's features to a distribution in $\Delta_{\La}$ \\
$c(\itemFeatures)@\ell$ & Classifier's output for label $\ell$ on the $\ith$ item\\
$\itemClassifierOutputVecRV$, $\itemClassifierOutputVec$  & RVs and realizations for the classifier's outputs for all $n$ items \\
$\evaluationRatingRV$, $\evaluationRating$ & RV and realization for evaluation rater $j$'s label for $\ith$ item\\
$\evaluationRatingVecRV$, $\evaluationRatingVec$ & RV and realization for evaluation rater $j$'s label for all $n$ items\\
$\evaluationRatingPanelRV$, $\evaluationRatingPanel$ & RV and realization for $k_e$ random raters for $\ith$ item\\
$\evaluationPanelLabelRV$, $\evaluationPanelLabel$ & RV and realization for panel label for $\ith$ item: majority vote of $k$ raters\\
$\evaluationPanelLabelVecRV$, $\evaluationPanelLabelVec$ & RV and realization for panel labels for all $n$ items \\

$\scoringLabelRV$, $\scoringLabel$ & RV and realization for evaluation label for $\ith$ item, used in scoring\\
$\benchmarkRatingRV$, $\benchmarkRating$ & RV and realization for  benchmark rater $j$'s label for $\ith$ item\\

$\benchmarkRatingVecRV$, $\benchmarkRatingVec$ & RV and realization for benchmark rater $j$'s label for all $n$ items\\
$\benchmarkRatingPanelRV$, $\benchmarkRatingPanel$ & RV and realization for $k_b$ random benchmark raters for $\ith$ item\\

\bottomrule
\end{tabular}
}
\end{table}



\section{Proof of Claim~\ref{claim:subjective-failure}}
\label{app:evalbenchmarkclass}

\claimsubjectivefailure*
\begin{proof}
Consider a classifier $c_5$ which always outputs a positive label and a $c_6$ which always outputs a negative label. 35\% of items have a response state $0.9$, meaning that 90\% of raters give a positive label. The rest have a response state $0.3$.

Define utility based on agreement between the classifier's outputs and human labels. The expected utility for $c_5$ is $.35 * .9 + .65 *.3 = 0.51$. The expected utility for $c_6$ is $.35 * .1 + .65 * .7 = 0.49$. Thus, $c_5$ has higher expected utility. And, as we would expect from Claim~\ref{claim:subjectivejudgmenet}, if we score using the agreement scoring function against a single randomly selected rater for each item, the expected scores match the expected utilities. 

However, for $k_e>1$, the probability of a positive majority vote will be more extreme than the response state probability. For example, with $k_e=3$, when the response state is $0.9$ the probability of a majority of three labels being positive is $0.9^3 + 3*0.9^2*0.1 = .972$, and when the response state is $0.3$ the probability of a majority of three labels being positive is $.216$. Thus, the expected score for $c_5$ is $.35 * .972 + .65 *.216 = 0.481$, lower than $c_6$'s expected score of $.519$. Note that as $k_e$ increases, expected scores continue to go in the wrong direction, with $c_6$'s advantage over $c_5$ increasing, as illustrated in Figure~\ref{fig:more-raters-bad}.

Intuitively, the majority vote introduces bias: it over-represents the most probable label within each response state. This exaggerates confidence in each item’s label and distorts expected scores. As panel size increases, this distortion grows, so that it approaches 100\% confidence in the most probable label, even if it is only slightly more common than the others.


\end{proof}

\begin{figure}[t]
\centering
\begin{tikzpicture}

\definecolor{darkgray176}{RGB}{176,176,176}
\definecolor{darkorange25512714}{RGB}{255,127,14}
\definecolor{lightgray204}{RGB}{204,204,204}
\definecolor{steelblue31119180}{RGB}{31,119,180}

\sffamily
\begin{axis}[
width = 3.0in, height = 2.0in,
legend cell align={left},
legend style={font=\small,
	nodes={scale=0.7, transform shape},
	at={(0.0,1)},
	anchor=north west,
	draw=none,
	fill=none},
every tick label/.append style={font=\scriptsize},
xlabel={\scriptsize \(\displaystyle k_e\) raters},
xmin=-0.4, xmax=30.4,
xtick style={color=black},
ylabel={\scriptsize Expected agreement score},
ylabel style={yshift=-10pt},
ymin=0.343332480132415, ymax=0.656667519867585,
ytick style={color=black}
]
\addplot [semithick, steelblue31119180, mark=*, mark size=1, mark options={solid}]
table {%
1 0.51
3 0.4806
5 0.453006
7 0.4309686
9 0.413913807
11 0.400742616996
13 0.39050913774924
15 0.382496382324166
17 0.376175995882666
19 0.371158306514031
21 0.367152987883094
23 0.363941037353033
25 0.361355274612742
27 0.359266724588956
29 0.357574981938559
};
\addlegendentry{$c_5$: predict pos}
\addplot [semithick, darkorange25512714, mark=square*, mark size=1, mark options={solid}]
table {%
1 0.49
3 0.5194
5 0.546994
7 0.5690314
9 0.586086193
11 0.599257383004
13 0.60949086225076
15 0.617503617675834
17 0.623824004117334
19 0.628841693485969
21 0.632847012116906
23 0.636058962646967
25 0.638644725387258
27 0.640733275411044
29 0.642425018061441
};
\addlegendentry{$c_6$: predict neg}
\end{axis}

\end{tikzpicture}
\caption{Expected agreement for two classifiers $c_5$ and $c_6$ as evaluation panel size $k_e$ increases from 1 to 30. The correct ordering, detected with single-rater scoring, reverses with $k_e > 1$.} \label{fig:more-raters-bad}
\end{figure}

\section{Proofs for Section~\ref{sec:calibration_minimizes_rater_equiv}}
\label{app:humanpanel}

\propmaxpower*

\begin{proof}
Cross-entropy corresponds to the log scoring rule, which is known to be a proper scoring rule. The defining property of a proper scoring rule is that the Bayesian posterior maximizes the expected score. More formally, for random variables $Y,Z$, the Bayesian posterior of $Y$ given a realization of Z is $\E_{Y,Z|Z=z}[Y]$. If \textsc{Score} is a proper scoring rule \citep{723c01ad-4d66-3cb4-987e-5f56192b514e,Gneiting01032007}, then we will have, for any alternative to that posterior:

\begin{align*}
&\E_{Y,Z|Z=z} \mleft[\textsc{Score}\mleft(\E_{Y,Z|Z=z}[Y],Y\mright)\mright] \\
\geq &\E_{Y,Z|Z=z} \mleft[\textsc{Score}\mleft(\cdot ,Y\mright)\mright]
\end{align*}

Substituting in $\benchmarkRatingPanelRV$ for $Z$, this property implies that the calibrated benchmark classifier yields the maximum possible power score. Given a fixed classifier score, that results in the minimal rater equivalence.
\end{proof}

\section{Proofs for Section~\ref{sec:empirical-approximation}}
\label{app:empirical}

\claimempiricalunbiased*

\begin{proof}

The theoretical power score is defined as the limit of the expected score, as the number of items grows, where the expectation is taken over a process that generates $k_b$ benchmark scores and one evaluation score per item as random draws from the item states:
$$
\powerscore_{\comb}(k_b) = \lim_{n\rightarrow \infty} \E_{\itemResponseStateVecRV} \left[\E_{\benchmarksVecRV, \evaluationsVecRV\mid \itemResponseStateVecRV} [\scoreFunction(\benchmarkPanelLabelVecRV, \evaluationRatingSingleVecRV)]\right],
$$

The empirical power score is defined as the average of scores for partitions in $\sampledpartitions$  of the rater matrix $\allVecRV$ into a benchmark panel and evaluation panel (see Definition~\ref{def:allpartitions}). 
$$
\empowerscore_{\comb}(k_b, k_e) = \frac{1}{|\sampledpartitions|} \sum_{(\benchmarksVec,\evaluationsVec) \in \sampledpartitions} \scoreFunction(\comb(\benchmarksVec), \textrm{maj}(\evaluationsVec)).
$$

Those partitions are generated through a slightly different random process, where an entire matrix $\allVecRV$ is drawn according to the item states, and then the benchmark and evaluation scores are randomly selected columns.
Any single partition, however, is probabilistically equivalent to a set of benchmark and evaluation ratings that are generated through the theoretical process, as random draws from the item states. Therefore, the expected score for any single partition matches the corresponding theoretical score.  Formally, for any function $F(\benchmarksVecRV, \evaluationsVecRV)$ of the benchmark and evaluation panels, we have:
$$\text{(empirical)  }
\E_{\itemResponseStateVecRV}\left[\E_{\allVecRV\mid \itemResponseStateVecRV}[\E_{\benchmarksVecRV, \evaluationsVecRV \mid \allVecRV} [F(\benchmarksVecRV, \evaluationsVecRV)]] \right] = \E_{\itemResponseStateVecRV} \left[\E_{\benchmarksVecRV, \evaluationsVecRV\mid \itemResponseStateVecRV}  [F(\benchmarksVecRV, \evaluationsVecRV)]\right] \text{  (theoretical)}
$$

By the linearity of expectations, the expectation of the empirical score $\E_{\allVecRV} [\empowerscore_{\comb}(k_b, k_e)]]$ is the same as the expected score for any one partition. And since, for each $n$, the expected empirical power score is $\E_{\itemResponseStateVecRV} \left[\E_{\benchmarksVecRV, \evaluationsVecRV\mid \itemResponseStateVecRV} [\scoreFunction(\benchmarkPanelLabelVecRV, \evaluationRatingSingleVecRV)]\right]$, with the assumption that the score function is well-behaved (Definition~\ref{def:well-behaved}) in the sense that it converges to its expectation in probability, as $n\rightarrow \infty$, the empirical power score converges to the theoretical power score in probability.


\end{proof}

\claimempclassifierconvergence*

\begin{proof}
The argument is similar to the one for the power score. The theoretical classifier expected score is defined in the limit, as the number of items increases. 

$$
\begin{aligned}
\classifierescore &= \lim_{n\rightarrow \infty}\E_{\itemFeaturesVecRV, \itemResponseStateVecRV}[\scoreFunction(c(\itemFeaturesVecRV), \itemResponseStateVecRV)] 
\end{aligned}
$$

For an empirical matrix $\allVec$ with $n$ rows, for any partition in $\sampledpartitions$ the expected classifier score against the selected evaluation column is the theoretical quantity of interest. An average over all partitions has the same expected value, with lower variance.


\end{proof}

\thmereqconverge*

\begin{proof}
The proof has three parts. First, we show that the theoretical rater equivalence can be expressed as the inverse of the theoretical power curve $\powercurve$ evaluated at the classifier's expected score:  
$$ \fratereqS(c) = {\powercurve_{\comb}}^{-1}(\classifierescore).$$

This follows directly from the definition
$$\fratereqS(c) = \inf \{ x : \powercurve_{\comb}(x) \geq \classifierescore \},$$ and the assumption that $\powercurve_{\comb}$ is strictly increasing and $\fratereqS(c) \in (0, \numraters-1)$. Thus, there is a unique intersection point of the power curve and the classifier score whose coordinates are $(\fratereqS(c), \classifierescore)$.

Second, for the empirical analogue, we seek to show that, for large enough n, the empirical rater equivalence can be expressed as the inverse of the empirical power curve evaluated at the empirical classifier's score: 
\newcommand{\pcn}{\empowercurve_n}
$$\emfratereq(c, 1) = \pcn^{-1}(\emscore).$$
where $\pcn(X)$ is a shorthand for the empirical power curve function $\empowercurve_{\comb}(X,1)$
with $n$ items (rows) in $\allVecRV$. The domain of $\pcn$ is $(0,\numraters-1)$, since we can only simulate benchmark panels using the ratings in the columns of $\allVecRV$.

We first establish the monotonicity of $\pcn(x)$ for sufficiently large $n$. Because the theoretical power curve is strictly increasing, there is a smallest jump size between integer benchmark panel sizes:
$$
\delta^* = \min_{k_b=0,\dots,\numraters-2} \left(\powercurve_{\comb}(k_b+1) - \powercurve_{\comb}(k_b)\right) > 0
$$

Choose $\delta_1 = \delta^*/3$. For each integer $k_b$, by the convergence in probability of empirical power scores (Claim~\ref{claim:empiricalunbiased}), for sufficiently large $n$ the empirical power will be close to the theoretical with high probability:
$$
\forall \epsilon > 0, \exists N_{k_b} \text{ such that } n > N_{k_b} \implies P\left(|\pcn(k_b) - \powercurve_{\comb}(k_b)| \geq \delta_1\right) < \epsilon^{\numraters}
$$

Let $N^* = \max_{k_b} N_{k_b}$. Then for $n > N^*$, with probability at least $1-\epsilon^{\numraters}$:
$$
\begin{aligned}
\pcn(k_b+1) - \pcn(k_b) 
&= [\powercurve_{\comb}(k_b+1) + (\pcn(k_b+1) - \powercurve_{\comb}(k_b+1))] \\
&\quad - [\powercurve_{\comb}(k_b) + (\pcn(k_b) - \powercurve_{\comb}(k_b))] \\
&> \delta^* - 2\delta_1 = \delta^*/3 > 0
\end{aligned}
$$

Since $\pcn$ is piecewise linear between integer points, this implies that with probability at least $1-\epsilon$, $\pcn$ is strictly increasing on $(0, \numraters-1)$ for all $n > N^*$. 



This establishes that $\pcn$ becomes invertible with probability approaching 1 as $n \to \infty$, allowing us to define, with probability approaching 1:
$$
\emfratereq(c,1) = \pcn^{-1}(\emscore)
$$


Third, we apply Lemma~\ref{lem:inverseconverge}, which states sufficient conditions for a sequence of inverse functions to converge in probability. Notice that the above proof also implies that the inverse of the sequence of functions $\pcn$ has bounded Lipschitz constant at most $\frac{3}{\delta^*}$, with probability approaching 1. Therefore, the sequence of functions $\pcn$ satisfies the criteria for the lemma. With the convergence results of Claim~\ref{claim:empiricalunbiased} and Claim~\ref{claim:empirical_classifier_score_limit}, we have $\emfratereq(c,1)= \pcn^{-1}(\emscore)\xrightarrow{P}  {\powercurve_{\comb}}^{-1}(\classifierescore)=\fratereqS(c)$.


\end{proof}

\begin{lemma}[Convergence of Inverse Functions]\label{lem:inverseconverge}
Let $\{f_n\}$ be a sequence of random continuous bijections on $[0,K]$ with continuous inverses $f_n^{-1}$. Assume:
\begin{enumerate}
    \item $f_n(x) \xrightarrow{P} f(x)$ for a continuous bijection $f$ with continuous inverse
    \item $\exists L > 0$ such that $\lim\limits_{n\to\infty} P\left(
       |f_n^{-1}(y) - f_n^{-1}(y')| \leq L|y-y'|
        \right) = 1$
    
\end{enumerate}
Then:
\begin{enumerate}
    \item For fixed $y \in (f(0),f(K))$: $f_n^{-1}(y) \xrightarrow{P} f^{-1}(y)$
    \item If $Y_n \xrightarrow{P} y^*$, then $f_n^{-1}(Y_n) \xrightarrow{P} f^{-1}(y^*)$
\end{enumerate}
\end{lemma}

\begin{proof}
\textbf{Part 1: Fixed y convergence.} Fix $\epsilon > 0$ and $y \in (f(0),f(K))$, we seek to show $P\left(|f_n^{-1}(y)-f^{-1}(y)|\geq \epsilon\right)$ converges to zero. 

Since $f^{-1}$ is continuous, $\exists \delta > 0$ such that for all $x\in[0,K]$:
$$
|f(x) - y| < \delta \implies |x - f^{-1}(y)| < \epsilon 
$$

Thus, $|f_n^{-1}(y)-f^{-1}(y)|\geq \epsilon \implies |f(f_n^{-1}(y)) - y| \geq \delta$

$f_n^{-1}(y)$ is legal as we pick sufficiently large $n$ such that $y\in {f_n(0),f_n(K)}$.

Since $f$ is continuous, there exists $\epsilon_1>0$,
$$
 |f(f_n^{-1}(y)) - y| \geq \delta \implies |f_n^{-1}(y) - f^{-1}(y)| \geq \epsilon_1
$$

Thus, $P\left(|f_n^{-1}(y)-f^{-1}(y)|\geq \epsilon\right)\leq P\left(|f_n^{-1}(y) - f^{-1}(y)| \geq \epsilon_1\right)$ which converges to zero by $f_n(x) \xrightarrow{P} f(x)$. This proves $f_n^{-1}(y) \xrightarrow{P} f^{-1}(y)$.

\textbf{Part 2: Random input convergence.} Now let $Y_n \xrightarrow{P} y^*$. For any $\epsilon > 0$:
$$
P\left(|f_n^{-1}(Y_n) - f^{-1}(y^*)| \geq \epsilon\right) \leq 
\underbrace{P\left(|f_n^{-1}(Y_n) - f_n^{-1}(y^*)| \geq \epsilon/2\right)}_{\text{(A)}} + 
\underbrace{P\left(|f_n^{-1}(y^*) - f^{-1}(y^*)| \geq \epsilon/2\right)}_{\text{(B)}}
$$
Term (A) vanishes by $Y_n \xrightarrow{P} y^*$ and the Lipschitz condition. Term (B) vanishes by Part 1. 
\end{proof}

\section{Algorithms and Proofs for Section~\ref{sec:abc}, Anonymous Bayesian Combiner}
\label{app:abc}

We can think of the $\abc$ in two parts, a learner and an executor. The learner, described in Algorithm~\ref{alg:jointdist}, calculates the probability of a random draw of $b_k$ ratings for a randomly drawn item in the dataset producing any realized label sequence $\benchmarkRatingPanel$. The executor, described in Algorithm~\ref{alg:abc_predictor}, uses the learned frequencies to predict a next rater's label for an item, conditional on some observed labels. We show that this converges to the Bayesian posterior for the next label, given the observed labels for the item. Intuitively, if two positive labels and a negative have been followed by a positive label on $90\%$ of other items, and the current item has received two positive and one negative label, $\abc$ will predict $.9$.

A naive approach for the learner component would simply look at the empirical frequency of an observed label sequence in the dataset. Intuitively, the probability of a specific label sequence, say (D, C, D, C, C), is just the fraction of items for which those are the first five labels. In the limit as the number of items grows, that fraction will approach the probability of that label sequence for a randomly drawn rater response state $\itemResponseStateRV$.

However, we find a way to greatly reduce the variance of the estimate, by extracting a little more information from the observed labels for each item. Rather than observing whether the observed labels exactly match the sequence of interest, we assess the probability of getting the sequence of interest from a random set of draws from the observed labels. This works because our anonymous assumption implies that the labels are independent draws from the same distribution and the order of observed labels is arbitrary. This produces, for each observed row in $\allVec$, the probability of producing the label sequence of interest. The average of those probabilities, across all rows in $\allVec$ (Algorithm~\ref{alg:jointdist}, line 6), is a lower variance estimate of the probability of getting the specified label sequence for a randomly selected item.

To provide intuitions, let's follow the subroutine \textsc{ProbabilityOneItem} for one hypothetical item in $\allVec$ that has two $D$ labels and seven $C$ labels. What is the probability of getting the specific sequence $(D, C, D, C, C)$ from a randomly chosen sequence of five  of those nine labels?

The number of possible sequences of $k=5$ raters is \[{9\choose 5}*5!\] because we can first chose the set of five raters and then choose their order. Of these, how many are $(D, C, D, C, C)$? 
We have to include three of the seven C raters; there are ${7\choose 3}$ ways to do that. Each selection can be placed in any order in positions 2, 4, and 5 in the sequence ($3!$ options). These can be combined with any assignment of both D raters (${2\choose 2}$ options) to positions 1 and 3 of the label sequence ($2!$ options).

\let\oldnl\nl
\newcommand{\nonl}{\renewcommand{\nl}{\let\nl\oldnl}}
\makeatother
\begin{algorithm}[t] 
\SetAlgoLined
\KwResult{An unbiased estimator of $\Pr[\benchmarkRatingPanelRV=\benchmarkRatingPanel]$}
    $\textit{sumprob}, \textit{numitems}$ = 0, 0 \;
    $\mathbf{bcounts} \gets \textsc{counts}(\benchmarkRatingPanel)$ \; 
    \For{$\mathbf{wrow}$$ \in \allVec$}{
        $\textit{sumprob} \gets \textit{sumprob} + \textsc{ProbabilityOneItem}(\mathbf{bcounts}, \textsc{counts}(\mathbf{wrow})) $ \;
        $numitems += 1$ \;
    }
    \Return $\frac{1}{numitems}tot$
    
 \caption{(Learner) $\textsc{LabelSeqProb}(\benchmarkRatingPanel, \allVec)$}\label{alg:jointdist}

    \setcounter{AlgoLine}{0}
    \SetProcNameSty{textsc}
    \SetArgSty{textnormal}
    \SetKwProg{myproc}{Procedure}{}{}
    \SetKwFunction{proc}{counts}
    \nonl\myproc{\proc{\upshape labels}}{
    \KwResult{Vector of counts: frequency of each label in labels}
    }
    
    \setcounter{AlgoLine}{0}
    \SetKwProg{myproc}{Procedure}{}{}
    \SetKwFunction{proc}{ProbabilityOneItem}
    \nonl\myproc{\textsc{\proc}$(\mathbf{bcounts}, \mathbf{wcounts})$}{

    \KwResult{Probability of selecting any specific sequence with $\mathbf{bcounts}$ from $\mathbf{wcounts}$ labels}
    $k_b \gets \sum_\ell{\mathbf{bcounts}@\ell}$ \;
    $k_w \gets \sum_\ell{\mathbf{wcounts}@\ell}$ \;
     \If{$\forall \ell \in \La$, $\mathbf{wcounts}@\ell \geq \mathbf{bcounts}@\ell $ }
     {\Return $ \frac{\prod_{\ell}{\mathbf{wcounts}@\ell \choose \mathbf{bcounts}@\ell} * \mathbf{(bcounts}@\ell)!}
{{k_w\choose k_b } * k_b!}$}
     \Else{\Return 0}
     }

\end{algorithm}

\begin{algorithm}[t]
\SetAlgoLined
\KwResult{Prediction $\vec{p}$, Bayesian posterior conditional on $\benchmarkRatingPanel$}
    $\forall \ell\in \La$, $\vec{p}@{\ell} = \frac{\textsc{LabelSeqProb}([\benchmarkRatingPanel + \ell],\allVec_{-i})}{\textsc{LabelSeqProb}(\benchmarkRatingPanel,\allVec_{-i})}$ \tcc*{\textsc{LabelSeqProb} from Algorithm~\ref{alg:jointdist}}
    \Return $\vec{p}$
 \caption{(Executor) $\textsc{AnonymousBayesianCombiner} 
 \abc(\benchmarkRatingPanel,\allVec_{-i})$}\label{alg:abc_predictor}
\end{algorithm}

Thus, the fraction of sequences that match is:
\begin{equation*}
\frac{({7\choose 3} * 3!)({2\choose 2} * 2!)}{{9\choose 5}*5!}
\end{equation*}

This explains the key formula on line 4 in \textsc{ProbabilityOneItem}:
\begin{equation*}
\frac{\prod_{\ell}{\mathbf{wcounts}@\ell \choose \mathbf{bcounts}@\ell} * \mathbf{(bcounts}@\ell)!}
{{k_w\choose k_b } * k_b!} 
\end{equation*}

Note that when Algorithm~\ref{alg:abc_predictor} asks for the learned $\textsc{LabelSeqProb}$ it redacts the current item from the data, providing $\allVec_{-i}$. This prevents the observed data from the current item from affecting the predicted probability of the label sequence for that item.

In the special case where the $\abc$ is asked to predict the first rating for an item (i.e., $b_k=0$), the denominator in Algorithm~\ref{alg:abc_predictor} will be 1 and the numerator will be the fraction of labels for all items (excluding the current item) that are labeled $\ell$. Thus, it estimates the base probability of each label by that label's overall frequency on the other items in the dataset.

\subsection{Optimality of ABC}

\abciscali*

\begin{proof}
We first show that  $\textsc{LabelSeqProb}(\benchmarkRatingPanel;\allVec)$ from Algorithm~\ref{alg:jointdist} is an unbiased estimator of $\Pr[\benchmarkRatingPanelRV=\benchmarkRatingPanel]$ and therefore converges to it almost surely as the number of items increases, due to the strong law of large numbers. 

To get an unbiased estimator of $\Pr[\benchmarkRatingPanelRV=\benchmarkRatingPanel]$, we could select a random sequence of raters for a random item and see if they happen to produce the specified label sequence. To reduce variance, we could run this process many times and take the mean. However, we can efficiently calculate what those means converge to in closed form.

First, note that the process chooses items uniformly at random. Thus, the overall probability is just the mean of the probability of the label sequence for each item. 
For each item, we enumerate all the possible sequences of raters and compute the fraction that match the specified label sequence. That gives the probability conditional on choosing that item. The probability of the label sequence for item $i$ is:
$$\frac{\prod_{\ell}{\mathbf{wcounts}@\ell \choose \mathbf{bcounts}@\ell} * \mathbf{(bcounts}@\ell)!}
{{k_w\choose k_b } * k_b!} $$

\textsc{ProbabilityOneItem} outputs exactly this quantity for one item. Algorithm~\ref{alg:abc_predictor} takes the mean across items. 

With the above result, almost surely 
$$
\begin{aligned}
\frac{\textsc{LabelSeqProb}([\benchmarkRatingPanel+\ell],\allVec_{-i})}{\textsc{LabelSeqProb}(\benchmarkRatingPanel,\allVec_{-i})} & \rightarrow \frac{\Pr[\benchmarkRatingPanelRV=\benchmarkRatingPanel,B_{i,k_b+1}=\ell]}{\Pr[\benchmarkRatingPanelRV=\benchmarkRatingPanel]} \nonumber \\
&=\Pr[B_{i,k_b+1}=\ell|\benchmarkRatingPanelRV=\benchmarkRatingPanel]
\end{aligned}
$$

Step 1 follows from the fact that the numerator and denominator each converge almost surely and the denominator's limit is non-zero because we assume that the distribution has full support. Step 2 follows from the definition of conditional probability. 

\end{proof}

\thmabsminratereq*

\begin{proof}
Since, in the limit, $\abc$ converges to the Bayesian posterior, which is the calibrated human benchmark classifier with respect to item states, Proposition \ref{prop:maxpower} implies that $\abc$ achieves maximal power scores,  
and Proposition~\ref{prop:minre} implies that it minimizes the rater equivalence. The results for the convergence of empirical to theoretical rater equivalence hold when the theoretical rater equivalence is in the range $(0, \numraters-1)$. If the theoretical rater equivalence for $\abc$ exceeds $\numraters-1$, then the theoretical rater equivalence for other combiners also exceeds $\numraters-1$. Consequently, the empirical rater equivalence computed for ABC and other combiners will all be the same in the limit. If the theoretical rater equivalence is 0 (a completely uninformative classifier), the expected empirical rater equivalence for the calibrated human benchmark classifier will also be 0, for any number of items, since any human raters will yield an expected score higher than the completely uninformative classifier. 
Since by definition, zero is the smallest possible rater equivalence, no other combiner can yield a lower expected empirical rater equivalence.
\end{proof}

\section{Proof of Running Times for Empirical Power Curve Computation}
\label{app:running_time}

To compute an empirical power curve with natural scoring functions--- such as accuracy, F1, Pearson correlation, or cross-entropy--- and simple combiners such as majority vote and frequency, there exists an implementation such that the total running time is $O(n k_w^2)$, where $n$ is the number of items and $k_w$ is the number of labels per item. The Anonymous Bayesian Combiner's implementation requires more time but with appropriate optimizations is still computationally tractable. 

\thmrunningtime*

\begin{proof}

To compute an empirical power score for a benchmark panel size $k_b$, $
\scoreFunction(\comb(\benchmarksVec),\textrm{maj}(\evaluationsVec))$ needs to be computed once for each partition $(\benchmarksVec, \evaluationsVec)$. Recall that when there are enough raters that there are more than 200 partitions, we randomly sample 200 of them. Thus, for running time analysis, there is a constant number of partitions.

All the common scoring functions such as agreement, F1, Pearson correlation, and cross-entropy, require $O(n)$ running time. Computing the majority vote of the evaluation panel takes time $O(k_e)$. The running time for the combiner depends on the combiner. 

For example, the running time of the frequency combiner, which just averages $k_b$ values  on each of $n$ items is $O(n k_b)$. Since we always have $k_b<k_w$ and $k_e<k_w$, the total running time for computing a power score with the frequency combiner is $O(n) + O(k_e) + O(n k_b) = O(n k_w)$. The total running time for computing the full power curve for all $k_b$ values, with the frequency combiner, is $O(n k_w^2)$.

The analysis for the Anonymous Bayesian Combiner, $\abc$, is more complicated. A naive implementation would involve redundant computations, so let us take some opportunities for optimization. First, notice that $\textsc{LabelSeqProb}(\mathbf{bcounts}, \allVec)$ depends only on the counts of labels, not the actual sequence $\benchmarkRatingPanel$. With $k_b$ labels, how many distinct label count vectors $\mathbf{bcounts}$ are possible? The maximum count for any single label is $k_b$ and counts for all but one label determine the count for the last one. Thus, the number of distinct $\mathbf{bcounts}$ is bounded by $k_b^{|\La|-1}$--- in the case of binary labels (positive and negative), this is just $k_b$.

Next, consider what computation needs to be done for any particular label count vector $\mathbf{bcounts}$. For different items, the redacted row in $\allVec_{-i}$ is different. However, if we pass in to \textsc{LabelSeqProb} the full matrix $\allVec$ and an indication of which row $i$ to redact, the algorithm can compute $\textit{sumprob}$ for all rows and then subtract out \textsc{ProbabilityOneItem} for the redacted item. The full sum only needs to be calculated once for each $\mathbf{bcounts}$. Moreover, if we memoize the results from calls to \textsc{ProbabilityOneItem} during the one-time calculation of the sum, the value to subtract for a redacted item can be looked up at constant cost rather than recalculated.

What is the running time for one invocation of \textsc{ProbabilityOneItem}? It is dominated by the product term on line 4. Each term in the product can be expressed as $\mathbf{wcounts}@\ell! /  \mathbf{bcounts}@\ell!$, which takes time bounded by $O(k_w) + O(k_b) = O(k_w)$. Since there are $|\La|$ terms in the product, the total running time for one invocation is $O(k_w |\La|)$. Because $\La$ is a linear factor and we assumed it to be constant, it is taken care of by the big-O notation, and we can write the total running time as $O(k_w)$.

For each label count vector $\mathbf{bcounts}$, then, the running time to compute the full $\textit{sumprob}$ for all rows is $O(n k_w)$. Since the number of distinct $\mathbf{bcounts}$ is bounded by $k_b^{|\La|-1} \leq k_w^{|\La|-1}$, the total running time for this is $O(n  k_w^{|\La|})$.

The computation of $\textsc{LabelSeqProb}$ for each item then requires two memoized lookups, one for the full $\textit{sumprob}$ and one for the \textsc{ProbabilityOneItem} to subtract out. The total for all items is $O(n)$, which is subsumed in the larger quantity above.

For a complete power curve, with all $k_b$ values, the total running time of the Anonymous Bayesian Combiner is $O(n k_w^{|\La|+1})$. With binary labels, this would be $O(n  k_w^3)$. In practice, the number of raters per item, $k_w$, is likely to be relatively small, typically 10 or fewer, so the cubic term is still computationally tractable.

\end{proof}

\section{Simulating Panels by Generating Synthetic Data}
\label{app:simulatingpanels}

The empirical approach of Section~\ref{sec:human_benchmark_implementation} estimated the expected score for a benchmark panel as the average score of simulated benchmark panels, scored against simulated evaluation panels, where both benchmark and evaluation panels are samples drawn from the empirical data. A tempting alternative approach is to use the empirical data to estimate parameters of the distribution $D_{\St}$ of rater response states for items, and then generate synthetic data from that distribution. However, this requires unreasonable assumptions about the shapes of distributions and the ability to infer parameters from the empirical data.

How would the synthetic data generation approach work? 
Given $D_{\St}$, generate a synthetic $\allVecRV$: make an independent draw of each synthetic item's response state $\itemResponseStateRV$ from $D_{\St}$ and then make draws from $\itemResponseStateRV$ for rater labels.
The synthetic $\allVecRV$ can have as many columns (rater labels) as desired. In the subjective utility model, a single column can be used as the synthetic evaluation labels. In the objective utility model, we can make the further mild assumption that the ground truth label $\itemGTRV$ is just the mode of $\itemResponseState$, meaning that rater labels are always biased toward the ground truth. Thus, we can use the mode of $S_i$ as the evaluation label for item $i$. 

This approach appears very attractive, because it makes it possible to compute power scores for all possible $k_b$ values.\footnote{In the unlikely case that the mathematical description of $D_{\St}$ were sufficiently simple, we could even skip generating the synthetic datasets and analytically compute the expected score for each panel size.} By contrast, the empirical approach using partitions of $\allVec$ can only compute power scores for panel sizes up to $\numraters - k_e$.

Unfortunately, the synthetic generation approach works poorly in practice because it relies on estimating $D_{\St}$ from the empirical data $\allVec$. If we assume $D_{\St}$ belongs to a parametric family $D_{\St}(\theta)$, we could estimate the parameter $\theta$ using maximum likelihood estimation (MLE):
\begin{equation*}
\hat{\theta} = \argmax_{\theta} \Pr_{\allVecRV \sim \itemResponseStateVecRV, \itemResponseStateVecRV \sim D_{\St}(\theta)}[\allVecRV = \allVec]
\end{equation*}
For example, we could assume that $D_{\St}$ is a mixture of two Beta distributions and estimate the five parameters, two for each Beta distribution and one for the mixture between them. 

However, this approach requires both a correct assumption about the parametric family, and enough empirical data in $\allVec$ to correctly estimate the parameters. 
When we tried to apply this approach, even using synthetic data with a known parametric family, we were disappointed to find that the estimated parameters could be far from recovering the true parameters $D_{\St}(\theta)$, leading to derived power curves (and consequently, rater equivalences) that deviate substantially from their theoretical counterparts. 

There is also one further challenge. In order to compute rater equivalences we need to compare the power curve to the classifier score. The synthetic data generation described so far gives us benchmark and evaluation raters, which is enough to estimate the power curve. If we wanted to generate classifier scores for the synthetic items as well, we would need to have the joint distribution between item states and classifier scores. That would require even more heroic assumptions about the parametric form, and even more data to estimate the parameters.

In the subjective utility model, there is a workaround for this second challenge. Rather than trying to generate classifier scores for synthetic items, we can use the the empirical mean classifier score for the observed items, with a single evaluation rater for each item. This is an unbiased estimator of the expected classifier score (see Claim \ref{claim:empirical_classifier_score_limit}). The power curve based on the synthetic $\allVecRV$ can then be compared to the empirical mean classifier score from the empirical $\allVec$ to yield rater equivalences.

That workaround is more problematic in the objective utility model. The reason is, as we have seen in Claim~\ref{claim:objnonunbiased}, there is no finite evaluation panel size that is a perfect substitute for evaluating against the ground truth. Even the majority vote of all of the columns of $\allVec$ will be a noisy proxy for the ground truth. The expected empirical mean classifier score against noisy ground truth will typically be lower than the theoretical classifier score. Thus, comparing the power scores based on the synthetic $\allVecRV$ to the empirical classifier scores is likely to yield an inflated rater equivalence value.

\section{Objective Utility Model}
\label{appdx:objective-utility}


The main text defined both the subjective and objective utility models but presented results only for the subjective utility model. Here, we provide results for the objective utility model. We begin with the corresponding formal definitions for the objective utility model. Next, we consider the effect of evaluation panel size on estimating absolute utility of one classifier and relative utility of two classifiers. In neither case is there a clean theoretical result in favor of either larger or smaller panels. Finally, we briefly consider the implications for empirical estimation of rater equivalence.

\subsection{Formal Definitions Under the Objective Utility Model}
\label{appdx:obj-formal-defs}

The main text defines power score, classifier expected score, and rater equivalence for the subjective utility model. Here we provide the corresponding definitions for the objective utility model.

\paragraph{Power Score.}
\begin{align*}
& \powerscore_{\comb}^{GT}(k_b) = \\
& \quad\lim_{n\rightarrow \infty}\E_{\itemResponseStateVecRV, \itemGTVecRV}\E_{\benchmarksVecRV \sim \itemResponseStateVecRV}\mleft[\scoreFunction\mleft(\benchmarkPanelLabelVecRV, \itemGTVecRV\mright)\mright]
\end{align*}

\paragraph{Classifier Expected Score.}
\begin{align*}
\classifierescore^{GT} &= \lim_{n\rightarrow \infty}\E_{\itemFeaturesVecRV, \itemGTVecRV}[\scoreFunction^{GT}(c(\itemFeaturesVecRV), \itemGTVecRV)]
\end{align*}

\paragraph{Rater Equivalence.}
\begin{equation*}
\fratereqG(c) = \min \Big\{x|\powercurve_{\comb}^{GT}(x) \geq \classifierescore^{GT}\Big\}.
\end{equation*}


\subsection{Evaluation Panel Scores as Estimates of Absolute Utility}
\label{appdx:obj-absolute}

We now examine how well evaluation panel scores approximate the true objective utility. In cases where there exists an objective ground truth, evaluation labels are a noisy approximation for the ground truth labels. If the noise is substantial, the computed classifier score may be substantially different from what it would be if ground truth labels were available to use for evaluation.

\setlength{\abovedisplayskip}{0pt}
\setlength{\belowdisplayskip}{0pt}
\begin{restatable} [Objective Utility Model: Individual \& Panel Labels Fail]{claim}{claimobjnonunbiased} \label{claim:objnonunbiased}

With objective ground truth, $\forall k_e\geq 1$, there exists a scenario such that:
\begin{align*}
\E_{\itemFeaturesVecRV,\itemGTVecRV}&[\utilityFunction^{GT}(\itemClassifierOutputVecRV,\itemGTVecRV)] \neq \\
&\E_{\itemFeaturesVecRV, \itemResponseStateVecRV}\E_{\evaluationsVecRV \sim \itemResponseStateVecRV} [\scoreFunction(\itemClassifierOutputVecRV, \evaluationPanelLabelVecRV)]
\end{align*}
\end{restatable}


%

\begin{proof}

In a scenario where the classifier’s outputs perfectly match the ground truth labels, its agreement score is 1. However, when considering a finite number of noisy raters, the agreement score of the classifier will always be strictly less than 1. This is because, with a nonzero probability, the majority vote of the raters may differ from the ground truth label, resulting in a mismatch with the classifier’s output.

While this is sufficient for the formal proof, we also include several more realistic scenarios with imperfect classifiers and show a mismatch in scores for small values of k. 
In practice, objective ground truth labels are not directly observable, so it is not possible to determine how discrepancies between rater labels and ground truth impact computed classifier scores. In our synthetic scenarios, however, we set the ground truth, and thus we can calculate a classifier's scores against ground truth labels. Moreover, in synthetic scenarios, we can also calculate the expected score for a classifier's labels against randomly selected human labels for panels of any size.

For all our scenarios, half of all items are ground truth positive. We simulate a very accurate classifier, to more clearly illustrate the impact of rater noise on scoring the classifier. The hard classifier we simulate outputs the correct label on 99\% of both ground truth positive and negative items. The soft classifier we simulate is a transformation of the hard classifier. When the hard classifier outputs a positive label, the soft classifier outputs the calibrated probability that a human rater will output a positive label, and similarly for hard classifier outputs of a negative label. Here, the calibration is done with respect to human rater labels, so the soft classifier's calibrated outputs are different depending on the distribution of human labels.

We consider multiple scenarios, with different levels of noise for evaluation raters, leading to different levels of inter-rater agreement. In the first, there is 20\% noise on positive items and 30\% noise on negative items. By that, we mean that on positive items, 80\% of human raters give a positive label and on negative items 70\% of human raters give a negative label. In this scenario, the Cohen's kappa measure of inter-rater agreement would be 0.25. This scenario is shown in the first row of Table~\ref{tab:evaluation-panel-size}; later scenarios decrease the rater noise, leading to higher inter-rater agreement.

\begin{table}[t]
\centering
\caption{Synthetic Scenarios with Expected Scores Against Evaluation Panels of Different Sizes}
\label{tab:evaluation-panel-size}
\small{
\begin{tabular}{cc|c|cccc|cccc}
\toprule
 \makecell{noise\\on pos} &  \makecell{noise\\on neg} &  \makecell{IRR\\kappa} &  \makecell{AGR\\k=1} &  \makecell{AGR\\k=3} &  \makecell{AGR\\k=5} & \makecell{AGR\\k=$\infty$} &  \makecell{CE\\k=1} &  \makecell{CE\\k=3} &  \makecell{CE\\k=5} & \makecell{CE\\k=$\infty$} \\
\midrule
                     0.20 &                      0.30 &                   0.25 &                 0.74 &                 0.83 &                 0.88 &              \textbf{0.99} &                0.81 &                0.63 &                0.50 &             \textbf{0.08} \\
                     0.13 &                      0.23 &                   0.41 &                 0.81 &                 0.90 &                 0.94 &              \textbf{0.99} &                0.68 &                0.45 &                0.31 &             \textbf{0.08} \\
                     0.09 &                      0.19 &                   0.52 &                 0.85 &                 0.93 &                 0.96 &              \textbf{0.99} &                0.59 &                0.34 &                0.22 &             \textbf{0.08} \\
                     0.06 &                      0.16 &                   0.61 &                 0.88 &                 0.95 &                 0.97 &              \textbf{0.99} &                0.51 &                0.27 &                0.17 &             \textbf{0.08} \\
                     0.01 &                      0.09 &                   0.82 &                 0.94 &                 0.98 &                 0.99 &              \textbf{0.99} &                0.30 &                0.14 &                0.10 &             \textbf{0.08} \\
\bottomrule
\end{tabular}
}
\end{table}

In each scenario, we can compute expected classifier scores against human labels, which are reported in columns on the right side of Table~\ref{tab:evaluation-panel-size}. There are closed-form algebraic solutions, so we do not need to generate simulated samples. With large samples of items generated according to the joint distribution parameters described, the scores would approach the reported expected values. 
For the hard classifier, we report overall agreement (AGR), where higher scores are better. For the soft classifier, we report cross entropy (CE), where lower scores are better. The $k=1$ columns provide expected scores against a single rater's label (a panel of size 1). 


With the majority vote of evaluation rater panels, we effectively reduce the noise; panel labels will more often agree with the ground truth. Indeed, we can think of an infinite sized panel as always producing the ground truth label (as long as the noise is less than 50\% for single rater labels). The bolded $k=\infty$ columns provided expected scores against ground truth.

For example, in the first scenario, with the most noise, the classifier's expected agreement with a single rater's labels is 74\%. Even in the final scenario, with least noise, the measured agreement would still be only 0.94. Recall that, by construction, the classifier's expected agreement level with ground truth labels is 0.99, which is substantially higher. When using cross-entropy as the scoring function, the discrepancies are even more stark because the logarithm in the scoring function amplifies differences for low probability events.

Using panels of 3 or 5 reduces noise and brings the measured scores closer to the ideal scores against ground truth labels. For example, in the first, noisiest scenario, with a panel of $k=5$ raters, the agreement score would reach 0.88, but is still far below the true agreement score of 0.99 measured against ground truth labels.

For any of the scenarios represented in the table, with a finite panel size $k$ there will still be some noise in the panel labels. As $k \rightarrow \infty$, the expected score converges to the score against ground truth. 


\end{proof}

We can build intuitions for practical settings by examining inter-rater agreement levels.
It is impossible to know the joint distribution of human
labels with the ground truth for real data generating processes. However, we can measure the inter-rater agreement level for an empirical dataset. Matching it to one of the rows in Table~\ref{tab:evaluation-panel-size}, where the third column reports the expected kappa score, may give a rough indication of how misleading it will be to treat the score against a single rater, or a panel of three or five, as if it were the score against ground truth labels.  The indication is only rough, however, since the actual impact will also depend on how accurate the classifier is.

Table~\ref{tab:typicalIRRs} shows the inter-rater agreement metrics for several public datasets where more than one rater label was collected for each item. Kappa scores range from as low as 0.09 to as high as .94.

\begin{table}[h]
\centering
\caption{Inter Rater Reliability (IRR) scores for selected public datasets, measured by Cohen's $\kappa$~\citeyearpar{cohen1960coefficient}, Krippendorff's $\alpha$~\citeyearpar{krippendorff2018content}, and  Intra-class correlation (ICC)~\citet{shrout1979intraclass}. The datasets illustrate the range of agreement levels encountered in practice.}
\label{tab:typicalIRRs}
\small{
\begin{tabular}{rrrrl}
\toprule
    & \makecell{\# Raters \\per item} &\makecell{\# Distinct \\ Labels $|\mathcal{L}|$} & \makecell{Measure} & \makecell{IRR\\}\\ \midrule
    \citet{madigan2021towards} & 2 & 2 & $\kappa$ & =0.94  \\
    **\citet{glenski2018guessthekarma} & $\sim$50 & 2 & $\kappa$ & $\sim$0.60-0.75\\
    \citet{salas2022student} & 2 & 3  & $\kappa$ & $\sim$0.78\\
    \citet{coudray2018classification} & 3 & 2  & $\kappa$ & $\sim$0.67-0.82\\
    \citet{van2021systematic} & 2 & 2 & $\kappa$ & =0.85\\
    \citet{kiela2020hateful} & 5 & 3 & $\kappa$ & =0.68\\
    \citet{salminen2022creating} & 3 & 2 & $\kappa$ & =0.09\\
    **\citet{Wulczyn2017machina} & 10-20 & 5 & $\alpha$ & =0.45\\
    \citet{cheng2015antisocial} & 5 & 5 & $\alpha$ & $\sim$0.35-0.39\\
    **\citet{mitra2015credbank} & 30 & 5 & ICC & =0.77 \\
    \citet{wang2023document} & 2 & 6 & $\kappa$ & =0.86 \\
\bottomrule
\multicolumn{4}{l}{**Used in Case Study}\\
\end{tabular}}
\vspace{-0.5cm}
\end{table}

\subsection{Relative Evaluation: Bigger Panels Not Always Better}
\label{appdx:obj-relative}

Having shown that panel scores are unreliable estimates of absolute utility, we now ask whether they at least preserve the relative ordering of classifiers.
Intuitively, in the objective utility model larger evaluation panel sizes should improve our ordering of classifiers. With more raters, the probability of random errors affecting the majority vote decreases, making the proxy approach the ground truth. However, we show that this intuition does not always hold.
For both the cross-entropy and agreement scoring functions we provide example scenarios where scoring against a single-rater evaluation panel yields a more reliable ordering than scoring against the majority votes of a three-rater evaluation panel.

\subsubsection{Cross-Entropy Utility}

Consider, first, a scenario where utility is determined by the cross-entropy score:
\begin{equation*}
\utilityFunction^{GT}\mleft(\itemClassifierOutputVec, \itemGTVec\mright) = \textsc{CE}\mleft(\itemClassifierOutputVec, \itemGTVec\mright)
\end{equation*}


\begin{table}[t]
\centering
\caption{A scenario with two classifiers with different error profiles. Classifier $c_1$ has better accuracy on positive items (96\%) but worse on negative items (80\%), while $c_2$ matches human rater accuracy on both. Half of all items are ground truth positive.}
\label{tab:two_classifiers}
\small{
\begin{tabular}{rc|c|c}
\toprule
& \multirow{2}{*}{\textbf{Classifier}} & \textbf{State 1} & \textbf{State 2} \\
& & $s_i = 0.85,\ g_i=\text{pos}$ & $s_i = 0.07,\ g_i=\text{neg}$ \\
\midrule
\multirow{4}{*}{\rotatebox{90}{\parbox{2cm}{\centering \small Joint\\distribution}}}
& \multirow{2}{*}{$c_1$} & 0.48 (Pos) & 0.10 (Pos) \\
&                         & 0.02 (Neg) & 0.40 (Neg) \\
\cmidrule(lr){2-4}
& \multirow{2}{*}{$c_2$} & 0.425 (Pos) & 0.035 (Pos) \\
&                         & 0.075 (Neg) & 0.465 (Neg) \\
\bottomrule
\end{tabular}
}
\end{table}

\begin{table}[t]
\centering
\caption{Cross-entropy scores of $c_1$ and $c_2$ against different evaluation panels, using calibrated probability outputs. Scoring against ground truth or a single rater ranks $c_2$ above $c_1$ (lower cross-entropy), but scoring against a three-rater panel reverses the ordering.}
\label{tab:calibrated_predictions}
\small{

\begin{tabular}{c|l|c|c|c}
\toprule
\textbf{Classifier} & \textbf{Evaluation} & \textbf{Calib. Pred.} & \textbf{Calib. Pred.} & \textbf{Cross-}\\
 & \textbf{Labels} &  \textbf{given c pos} & \textbf{given c neg} & \textbf{entropy} \\
\midrule
     $c_1$ &  Ground Truth &                0.8276 &                0.0476 &         0.5007 \\
     $c_2$ &  Ground Truth &                0.9239 &                0.1389 &         \textbf{0.4925} \\
\midrule
     $c_1$ & 3-Rater Panel &                0.7797 &                0.0581 &         \textbf{0.5755} \\
     $c_2$ & 3-Rater Panel &                0.8689 &                0.1425 &         0.5769 \\
\midrule
     $c_1$ &  Single Rater &                0.7155 &                0.1071 &         0.7060 \\
     $c_2$ &  Single Rater &                0.7907 &                0.1783 &         \textbf{0.7058} \\
\bottomrule
\end{tabular}

}
\end{table}

\begin{restatable}[Bigger is not Always Better with Cross-Entropy]{claim}{propacc}\label{prop:acc}

Let classifiers $c_1$ and $c_2$ be evaluated in terms of cross-entropy-based utility. There exists a scenario where  the expected score against individual labels $\evaluationRatingSingleRV$ yields the same ordering as the expected score against ground truth labels, but the expected score against panel labels $\evaluationPanelLabelRV$ with panel size three yields the opposite ordering of the two classifiers.
\end{restatable}

\begin{proof}

As a counter-example, consider the following scenario. 
Half the items have a ground truth positive label.
Human raters provide correct labels on 85\% of positive items and 93\% of negative items.  
Benchmark classifier $c_2$ makes errors with the same probabilities as human raters.
Classifier $c_1$ has better accuracy, 96\%, on positive items but worse, 80\%, on negative items. Table~\ref{tab:two_classifiers} summarizes the joint distribution of the classifiers' outputs and the rater response states more precisely.

To use the cross-entropy scoring function with these binary classifiers, we have them output decimal values that correspond to calibrated probability estimates for the labels they will be scored against. For example, when $c_1$ is scored against ground truth labels, it outputs 0.83 when it detects a positive and 0.05 when it detects a negative. To compute those, note that when $c_1$ is scored against ground truth, $\Pr[g_i = pos | c_1(x_i)=pos] = .48/(.48+.1) = .83$ and $\Pr[g_i = pos | c_1(x_i)=neg] = .02/(.4+.02) = .05$. 
Since the probability of a classifier positive for $c_1$ is $0.48+0.1=0.58$, its expected cross-entropy score, for a large dataset, is $-(0.58 * [.83 * \log .83 + .17*\log .17]  + 0.42 * [.05 * \log .05 + .95 * \log .95]) \approx .501$. These are reflected in the first row of Table~\ref{tab:calibrated_predictions}. 

Similar calculations for $c_2$, shown in the second row, yield an output of 0.924 when it detects a positive and 0.139 when it detects a negative, with an expected cross-entropy score of 0.492. Thus, if the two classifiers were both scored against the ground truth, the expected score is higher for $c_1$, meaning that the first classifier is slightly less informative about ground truth.  

The same is true if both classifiers are scored against a single rater, as shown in the last two rows of the table: the first classifier is also slightly less informative about a single rater. However, as seen in the middle two rows of the table, $c_1$ is slightly more informative about the majority of a panel of three raters. 

\end{proof}

\begin{figure}[t]
\vspace{0.5em}  
\centering
\begin{tikzpicture}

\definecolor{darkgray176}{RGB}{176,176,176}
\definecolor{green}{RGB}{0,128,0}
\definecolor{lightgray204}{RGB}{204,204,204}
\definecolor{orange}{RGB}{255,165,0}

\sffamily

\begin{axis}[
legend style={font=\small,
	nodes={scale=0.7, transform shape},
	at={(0.0,1)},
	anchor=north west,
	draw=none,
	fill=none},
legend cell align={left},
every tick label/.append style={font=\scriptsize},
width = 3.0in, height = 2.0in,
xlabel={\scriptsize $\Pr[c_1(x_i) = \text{pos} | g_i = \text{pos})]$},
xmin=0.957005, xmax=0.967895,
xtick = {0.9592, 0.9603, 0.9639},
xticklabels = {.96, .961, .964},
xtick style={color=black},
ylabel style={yshift=-5pt},
xticklabel style={
    /pgf/number format/fixed,
    /pgf/number format/precision=3
},
ylabel={\scriptsize $\textrm{CE}(c_2) - \textrm{CE}(c_1)$},
ymin=-0.0146274446026598, ymax=0.0157007461000535,
ytick style={color=black},
yticklabel style={
            /pgf/number format/fixed,
            /pgf/number format/precision=3
        },
scaled y ticks=false
]
\addplot [semithick, orange]
table {%
0.9575 -0.00279486493386988
0.9576 -0.0026920253057468
0.9577 -0.00258913684400663
0.9578 -0.00248619949744988
0.9579 -0.00238321321478074
0.958 -0.00228017794460605
0.9581 -0.00217709363543628
0.9582 -0.00207396023568318
0.9583 -0.001970777693662
0.9584 -0.00186754595758926
0.9585 -0.00176426497558341
0.9586 -0.00166093469566403
0.9587 -0.00155755506575228
0.9588 -0.0014541260336704
0.9589 -0.0013506475471407
0.959 -0.00124711955378615
0.9591 -0.00114354200112932
0.9592 -0.00103991483659305
0.9593 -0.000936238007499479
0.9594 -0.000832511461070051
0.9595 -0.000728735144424675
0.9596 -0.000624909004582619
0.9597 -0.000521032988461345
0.9598 -0.00041710704287562
0.9599 -0.000313131114538512
0.96 -0.000209105150061228
0.9601 -0.000105029095950726
0.9602 -9.02898612098468e-07
0.9603 0.000103273495653422
0.9604 0.000207500140648409
0.9605 0.000311777090279353
0.9606 0.000416104398556716
0.9607 0.000520482119596211
0.9608 0.000624910307617632
0.9609 0.00072938901694658
0.961 0.00083391830201357
0.9611 0.000938498217355532
0.9612 0.00104312881761515
0.9613 0.00114781015754151
0.9614 0.00125254229199034
0.9615 0.00135732527592436
0.9616 0.00146215916441378
0.9617 0.00156704401263635
0.9618 0.00167197987587747
0.9619 0.00177696680953165
0.962 0.00188200486910067
0.9621 0.0019870941101961
0.9622 0.00209223458853902
0.9623 0.00219742635995829
0.9624 0.00230266948039465
0.9625 0.00240796400589743
0.9626 0.0025133099926275
0.9627 0.00261870749685622
0.9628 0.00272415657496505
0.9629 0.00282965728344869
0.963 0.00293520967891187
0.9631 0.00304081381807259
0.9632 0.00314646975776062
0.9633 0.00325217755491852
0.9634 0.00335793726660227
0.9635 0.00346374894998097
0.9636 0.0035696126623373
0.9637 0.00367552846106844
0.9638 0.00378149640368558
0.9639 0.00388751654781472
0.964 0.00399358895119711
0.9641 0.00409971367168915
0.9642 0.00420589076726308
0.9643 0.0043121202960078
0.9644 0.00441840231612739
0.9645 0.00452473688594374
0.9646 0.00463112406389582
0.9647 0.00473756390853997
0.9648 0.00484405647854941
0.9649 0.00495060183271706
0.965 0.00505720002995319
0.9651 0.00516385112928747
0.9652 0.00527055518986902
0.9653 0.00537731227096627
0.9654 0.00548412243196733
0.9655 0.00559098573238065
0.9656 0.00569790223183636
0.9657 0.005804871990084
0.9658 0.0059118950669958
0.9659 0.00601897152256531
0.966 0.00612610141690756
0.9661 0.00623328481026131
0.9662 0.00634052176298722
0.9663 0.00644781233556962
0.9664 0.00655515658861616
0.9665 0.00666255458285842
0.9666 0.00677000637915309
0.9667 0.0068775120384803
0.9668 0.00698507162194656
0.9669 0.00709268519078288
0.967 0.0072003528063469
0.9671 0.00730807453012144
0.9672 0.00741585042371712
0.9673 0.0075236805488712
0.9674 0.00763156496744799
};
\addlegendentry{Single rater}
\addplot [semithick, green]
table {%
0.9575 -0.00278519038371283
0.9576 -0.00261852431589937
0.9577 -0.00245174116017888
0.9578 -0.0022848406878625
0.9579 -0.00211782266941246
0.958 -0.001950686874436
0.9581 -0.00178343307168366
0.9582 -0.00161606102904066
0.9583 -0.00144857051352498
0.9584 -0.00128096129128136
0.9585 -0.00111323312757672
0.9586 -0.000945385786794317
0.9587 -0.000777419032430282
0.9588 -0.000609332627087844
0.9589 -0.000441126332471509
0.959 -0.000272799909383947
0.9591 -0.000104353117718503
0.9592 6.42142835441395e-05
0.9593 0.000232902536341784
0.9594 0.000401711883536215
0.9595 0.000570642568917701
0.9596 0.000739694837211702
0.9597 0.00090886893408243
0.9598 0.00107816510613978
0.9599 0.00124758360094368
0.96 0.00141712466701072
0.9601 0.00158678855381855
0.9602 0.00175657551181135
0.9603 0.00192648579240773
0.9604 0.00209651964800228
0.9605 0.002266677331976
0.9606 0.00243695909869734
0.9607 0.00260736520353289
0.9608 0.00277789590284794
0.9609 0.00294855145401679
0.961 0.00311933211542686
0.9611 0.00329023814648388
0.9612 0.00346126980761974
0.9613 0.00363242736029634
0.9614 0.00380371106701421
0.9615 0.00397512119131677
0.9616 0.00414665799779707
0.9617 0.00431832175210384
0.9618 0.00449011272094957
0.9619 0.0046620311721135
0.962 0.00483407737445074
0.9621 0.00500625159789864
0.9622 0.00517855411348145
0.9623 0.00535098519331811
0.9624 0.00552354511062914
0.9625 0.0056962341397428
0.9626 0.00586905255610204
0.9627 0.00604200063627164
0.9628 0.00621507865794324
0.9629 0.00638828689994431
0.963 0.00656162564224516
0.9631 0.00673509516596377
0.9632 0.00690869575337522
0.9633 0.00708242768791717
0.9634 0.00725629125419791
0.9635 0.00743028673800361
0.9636 0.00760441442630516
0.9637 0.00777867460726539
0.9638 0.00795306757024727
0.9639 0.00812759360582099
0.964 0.00830225300577081
0.9641 0.00847704606310451
0.9642 0.00865197307205856
0.9643 0.00882703432810772
0.9644 0.00900223012797191
0.9645 0.00917756076962556
0.9646 0.0093530265523028
0.9647 0.00952862777650831
0.9648 0.00970436474402347
0.9649 0.00988023775791513
0.965 0.0100562471225445
0.9651 0.0102323931435735
0.9652 0.0104086761279757
0.9653 0.0105850963840425
0.9654 0.0107616542213931
0.9655 0.0109383499509815
0.9656 0.0111151838851073
0.9657 0.0112921563374221
0.9658 0.0114692676229405
0.9659 0.0116465180580467
0.966 0.0118239079605049
0.9661 0.0120014376494685
0.9662 0.0121791074454878
0.9663 0.0123569176705212
0.9664 0.0125348686479418
0.9665 0.012712960702549
0.9666 0.0128911941605777
0.9667 0.0130695693497059
0.9668 0.0132480865990671
0.9669 0.0134267462392565
0.967 0.013605548602345
0.9671 0.0137844940218849
0.9672 0.0139635828329228
0.9673 0.0141428153720076
0.9674 0.0143221919772029
};
\addlegendentry{3-rater panel}
\addplot [semithick, blue]
table {%
0.9575 -0.0132488904798092
0.9576 -0.0130470354721757
0.9577 -0.0128450175967938
0.9578 -0.012642836452731
0.9579 -0.0124404916371484
0.958 -0.0122379827452878
0.9581 -0.0120353093704575
0.9582 -0.0118324711040189
0.9583 -0.0116294675353723
0.9584 -0.0114262982519426
0.9585 -0.0112229628391655
0.9586 -0.0110194608804726
0.9587 -0.0108157919572769
0.9588 -0.0106119556489594
0.9589 -0.0104079515328523
0.959 -0.0102037791842252
0.9591 -0.00999943817627019
0.9592 -0.00979492808008575
0.9593 -0.00959024846466261
0.9594 -0.00938539889686618
0.9595 -0.00918037894142376
0.9596 -0.00897518816090587
0.9597 -0.00876982611571142
0.9598 -0.00856429236405232
0.9599 -0.00835858646193582
0.96 -0.00815270796314871
0.9601 -0.00794665641924053
0.9602 -0.00774043137950692
0.9603 -0.00753403239097206
0.9604 -0.00732745899837228
0.9605 -0.00712071074413845
0.9606 -0.00691378716837787
0.9607 -0.00670668780885719
0.9608 -0.00649941220098493
0.9609 -0.00629195987779296
0.961 -0.0060843303699174
0.9611 -0.00587652320558152
0.9612 -0.00566853791057664
0.9613 -0.00546037400824351
0.9614 -0.00525203101945249
0.9615 -0.00504350846258594
0.9616 -0.00483480585351703
0.9617 -0.00462592270559137
0.9618 -0.00441685852960672
0.9619 -0.00420761283379323
0.962 -0.00399818512379269
0.9621 -0.00378857490263895
0.9622 -0.003578781670736
0.9623 -0.00336880492583946
0.9624 -0.00315864416303224
0.9625 -0.00294829887470627
0.9626 -0.00273776855053809
0.9627 -0.00252705267746967
0.9628 -0.00231615073968594
0.9629 -0.00210506221859053
0.963 -0.00189378659278527
0.9631 -0.00168232333804719
0.9632 -0.00147067192730477
0.9633 -0.0012588318306157
0.9634 -0.00104680251514266
0.9635 -0.000834583445129766
0.9636 -0.000622174081878812
0.9637 -0.000409573883725312
0.9638 -0.000196782306012167
0.9639 1.62011989325306e-05
0.964 0.000229377181821588
0.9641 0.000442746196435251
0.9642 0.0006563087996454
0.9643 0.000870065551441535
0.9644 0.00108401701495797
0.9645 0.00129816375649883
0.9646 0.0015125063455671
0.9647 0.00172704535488888
0.9648 0.00194178136044298
0.9649 0.00215671494148784
0.965 0.00237184668058954
0.9651 0.00258717716365053
0.9652 0.00280270697993745
0.9653 0.00301843672211111
0.9654 0.00323436698625557
0.9655 0.00345049837190747
0.9656 0.00366683148208635
0.9657 0.00388336692332458
0.9658 0.00410010530569882
0.9659 0.00431704724286086
0.966 0.00453419335206873
0.9661 0.00475154425421898
0.9662 0.004969100573878
0.9663 0.00518686293931614
0.9664 0.00540483198253927
0.9665 0.00562300833932261
0.9666 0.00584139264924488
0.9667 0.00605998555572151
0.9668 0.00627878770604129
0.9669 0.00649779975139775
0.967 0.00671702234692784
0.9671 0.00693645615174698
0.9672 0.00715610182898441
0.9673 0.00737596004582053
0.9674 0.00759603147352461
};
\addlegendentry{Ground Truth}
\addplot [semithick,  black]
table {%
0.957005 0
0.967895 0
};
\addplot [semithick, orange, dashed, forget plot]
table {%
0.9603 -0.0146274446026598
0.9603 0
};
\addplot [semithick, green, dashed, forget plot]
table {%
0.9592 -0.0146274446026598
0.9592 0
};
\addplot [semithick, blue, dashed, forget plot]
table {%
0.9639 -0.0146274446026598
0.9639 0
};
\end{axis}

\end{tikzpicture}
\caption{Cross-entropy ordering reversal between single-rater and three-rater evaluation as $\Pr[c_1(x_i)=\text{pos} \mid g_i=\text{pos}]$ varies. The reversal occurs only in the narrow range 0.96–0.961.}
\label{fig:reversal-fig}
\end{figure}

Holding the other parameters fixed, Figure~\ref{fig:reversal-fig} shows the difference in scores for two classifiers as $\Pr[(c_1(x_i)=\text{pos} | g_i=\text{pos}]$ increases. We can see that this reversal happens only for a relatively small range of parameter values: between about .959 and .961. For values below .959, scoring against either one rater or a panel of three yields the correct ordering that $c_2$ is better than $c_1$. For values between about 0.961 and .964, scoring against either one rater or a panel of three yields an incorrect ordering that $c_1$ is better. And above about .964, that becomes the correct ordering that $c_1$ really is better.
While the range of parameter values where the reversal occurs is small, there is no value of $\Pr[(c_1(x_i)=\text{pos} | g_i=\text{pos}]$ for which scoring against a panel of three gives the correct ordering while scoring against a single rater gives an incorrect ordering.

We acknowledge that in the counter-examples presented here, the score differences constituting the reversal are small in magnitude. However, the practical significance of any reversal depends on the stakes of the decision: even a small score difference can matter if it determines which classifier is deployed. Moreover, the analyst will not know the true parameters of the scenario and thus cannot determine whether they are in a region where reversals occur. A theoretical characterization of the conditions under which reversals arise---and their likely magnitude--- remains an open question.

\subsubsection{Agreement}

Suppose, instead, that utility is determined by raw agreement. Here, if there are only two item states, larger evaluation panels are better.\footnote{The proof is long and not very relevant, so we have omitted it.} However, if we consider slightly more complicated information structures, we again find a scenario where single rater evaluation panels are more reliable than three-person evaluation panels.

\begin{restatable}[Bigger is not Always Better with Agreement]{claim}{propagg}\label{prop:acc2}

Let two classifiers be evaluated in terms of agreement-based utility. With three item states there exists a scenario where the expected score against individual labels $\evaluationRatingSingleRV$ yields the same ordering as the expected score against ground truth labels, but the expected score against panel labels $\evaluationPanelLabelRV$ with panel size three yields the opposite ordering of the two classifiers.
\end{restatable}

\begin{proof}
Consider a scenario with three rater response states, $0.96,0.1,0.54$. The ground truth is defined as the state's mode. Thus, these three states have ground truth pos, neg, and pos correspondingly. Classifier $c_3$ has better accuracy on states 1 and 3, while $c_4$ is better on state 2. Table~\ref{tab:agreementreversal_JD} describes the joint distributions between states and classifier outputs.

To compute $c_3$'s expected agreement against ground truth, we compute $0.18+0.24+0.32=0.74$. To compute $c_3$'s expected agreement against a single rater, we first compute the joint distribution matrix between $c_3$'s output and a single rater's label:
\begin{align*}
&\left(\begin{bmatrix}
0.18 & 0.16 & 0.32 \\
0.02 & 0.24 & 0.08 \\
\end{bmatrix}\right)
\left(\begin{bmatrix}
0.96 & 0.1 & 0.54\\
0.04 & 0.9 & 0.46\\
\end{bmatrix}\right)^{\top}\\
=& \begin{bmatrix}
0.3616 & 0.2984\\
0.0864 & 0.2536
\end{bmatrix}
\end{align*}

Thus, $c_3$'s expected agreement against a single rater is the trace of the above matrix, which is $0.6152$.

To compute $c_3$'s expected agreement against three raters' majority, we first compute the probability that three raters' majority is pos for these three states correspondingly and obtain $0.9953$, $0.028$, and $0.5599$. We then follow similar steps and obtain that $c_3$'s expected agreement against a three raters' majority $\approx 0.631$.

By analogous steps, we repeat the computation process for $c_4$. Results for both classifiers are summarized in Table~\ref{tab:agreementreversal}. Scoring against a panel of three raters gives expected scores that are closer to the scores against ground truth. However, scoring against a panel of three raters gives the wrong ordering of the two classifiers, while scoring against a single rater gives the correct ordering.

\end{proof}

\begin{table}[t]
\centering
\caption{A scenario with three rater response states and two classifiers. Classifier $c_3$ is more accurate overall but $c_4$ is better on state 2 (the negative items).}
\label{tab:agreementreversal_JD}
\small
\begin{tabular}{rc|c|c|c|c}
\toprule
& \multirow{2}{*}{\textbf{Classifier}} & \multirow{2}{*}{\textbf{Output}} & \textbf{State 1 } & \textbf{State 2} & \textbf{State 3 } \\
 & & & $s_i=0.96,\ g_i=\text{pos}$ & $s_i=0.1,\ g_i=\text{neg}$ & $s_i=0.54,\ g_i=\text{pos}$ \\
\midrule
\multirow{4}{*}{\rotatebox{90}{\parbox{2cm}{\centering \small Joint\\distribution}}} & \multirow{2}{*}{$c_3$}   & Pos & 0.18 & 0.16 & 0.32 \\
&   & Neg & 0.02 & 0.24 & 0.08 \\
\cmidrule(lr){2-6}
& \multirow{2}{*}{$c_4$} & Pos & 0.12 & 0.08 & 0.20 \\
&   & Neg & 0.08 & 0.32 & 0.20 \\
\bottomrule
\end{tabular}
\end{table}

\begin{table}[t]
\centering
\caption{Agreement against a single rater gives the same ordering as scoring against ground truth, but scoring against a panel of three raters yields a reversal.}
\label{tab:agreementreversal}
\small
\begin{tabular}{c|c|c|c|c}
\toprule
\textbf{Classifier} & \textbf{Evaluation Labels} & \multicolumn{3}{c}{\textbf{Expected Agreement}} \\
\midrule
\textbf{$c_3$} & \textbf{Ground Truth} & \multicolumn{3}{c}{\textbf{0.74}} \\
$c_4$ & Ground Truth & \multicolumn{3}{c}{0.64} \\
\midrule
$c_3$ & 3-Rater Panel & \multicolumn{3}{c}{0.6314} \\
\textbf{$c_4$} & \textbf{3-Rater Panel} & \multicolumn{3}{c}{\textbf{0.6331}} \\
\midrule
\textbf{$c_3$} & \textbf{Single Rater} & \multicolumn{3}{c}{\textbf{0.6152}} \\
$c_4$ & Single Rater & \multicolumn{3}{c}{0.6144} \\
\bottomrule
\end{tabular}
\end{table}


\subsection{Interpreting Empirical Rater Equivalence}
\label{appdx:obj-empirical-req}

The previous results concern theoretical expected scores. We now consider what happens in practice when rater equivalence is estimated from empirical data.
In the objective utility model, the empirical power scores and classifier scores are not reliable approximations for the corresponding theoretical values, for the reason explored in Section~\ref{appdx:obj-absolute}: the majority vote of any finite evaluation panel is not a reliable proxy for the ground truth. Moreover, for the reasons explored in Section~\ref{appdx:obj-relative}, a finite evaluation panel cannot provide a reliable rank ordering of how two classifiers would score against the ground truth. That includes comparisons between the score of a classifier and a benchmark panel. Thus, the empirical power score may not reveal what the power score would be if it were possible to evaluate both the classifier outputs and benchmark panel outputs against the unknown ground truth. That said, it still provides some intuition about the quality of a classifier to report the empirical power score even in settings that we think of as having objective utility.

Since no panel size provides a guarantee about how the empirical rater equivalence matches the theoretical rater equivalence, it may be reasonable to set $k_e=1$. That allows for computing the power scores for larger $k_b$ values than would be possible if more of the available rating labels were allocated to the evaluation matrix $\evaluationsVec$.

\end{appendix}

\end{document}